\documentclass[twocolumn,english]{article}
\usepackage[T1]{fontenc}
\usepackage[latin9]{inputenc}
\usepackage{geometry}
\usepackage{color}
\usepackage{booktabs}
\usepackage{url}
\usepackage{amsmath}
\usepackage{amsthm}
\usepackage{amssymb}
\usepackage{graphicx}
\usepackage[authoryear]{natbib}

\makeatletter

\providecommand{\tabularnewline}{\\}
\newcommand{\lyxdot}{.}

\theoremstyle{plain}
\newtheorem{thm}{\protect\theoremname}
  \theoremstyle{definition}
  \newtheorem{defn}[thm]{\protect\definitionname}
  \theoremstyle{plain}
  \newtheorem{prop}[thm]{\protect\propositionname}
  \theoremstyle{plain}
  \newtheorem*{prop*}{\protect\propositionname}
  \theoremstyle{plain}
  \newtheorem{lem}[thm]{\protect\lemmaname}

\usepackage{amsmath}
\usepackage{amsfonts}
\usepackage{amsthm}
\usepackage{bm}
\usepackage{enumitem}

\usepackage[T1]{fontenc}
\usepackage{hyperref}

  \definecolor{mydarkblue}{rgb}{0,0.08,0.65}

  \hypersetup{ %
    pdftitle={},pdfauthor={},pdfsubject={}, 
pdfkeywords={}, pdfborder=0 0 0,pdfpagemode=UseNone,
colorlinks=true, linkcolor=mydarkblue, citecolor=mydarkblue,
filecolor=mydarkblue, urlcolor=mydarkblue,pdfview=FitH}

\usepackage[mathscr]{eucal}

\usepackage[resetlabels,labeled]{multibib}
\newcites{sup}{References}

\usepackage{mathtools}

\usepackage{microtype}
\usepackage{natbib}

\usepackage{soul}
\usepackage{subfig}
\usepackage{url}
\usepackage{tikz}
\usepackage{thm-restate}
\usepackage{thmtools}
\usepackage{wrapfig}
\usepackage{verbatim}

\allowdisplaybreaks

\newtheorem{assumption}{Assumption}

\renewcommand{\@listI}{%
\leftmargin=1.4em
\rightmargin=0pt
\labelsep=5pt
\labelwidth=20pt
\itemindent=2pt
\listparindent=2pt
\topsep=2pt plus 2pt 
\partopsep=2pt plus 2pt 
\parsep=2pt plus 5pt
\itemsep=\parsep}

\makeatother

\usepackage{babel}
  \providecommand{\definitionname}{Definition}
  \providecommand{\lemmaname}{Lemma}
  \providecommand{\propositionname}{Proposition}
\providecommand{\theoremname}{Theorem}

\begin{document}

\newcommand{\ourtitle}{An Adaptive Test of Independence with Analytic Kernel Embeddings}

\title{\ourtitle{}}

\author{Wittawat Jitkrittum,$^{1}$ \quad{}Zolt{\'a}n Szab{\'o},$^2$\thanks{Zolt{\'a}n Szab{\'o}'s ORCID ID: \protect\url{http://orcid.org/0000-0001-6183-7603}.}
\quad{}\,Arthur Gretton$^{1}$\vspace{2mm}\\
$^{1}$Gatsby Computational Neuroscience Unit, University College
London\\
$^{2}$Department of Applied Mathematics, CMAP, \'{E}cole Polytechnique\vspace{2mm}\\
\url{wittawat@gatsby.ucl.ac.uk}\\
\url{zoltan.szabo@polytechnique.edu}\\
\url{arthur.gretton@gmail.com}}
\maketitle
\begin{abstract}
A new computationally efficient dependence measure, and an adaptive
statistical test of independence, are proposed. The dependence measure
is the difference between analytic embeddings of the joint distribution
and the product of the marginals, evaluated at a finite set of locations
(features). These features are chosen so as to maximize a lower bound
on the test power, resulting in a test that is data-efficient, and
that runs in linear time (with respect to the sample size n). The
optimized features can be interpreted as evidence to reject the null
hypothesis, indicating regions in the joint domain where the joint
distribution and the product of the marginals differ most. Consistency
of the independence test is established, for an appropriate choice
of features. In real-world benchmarks, independence tests using the
optimized features perform comparably to the state-of-the-art quadratic-time
HSIC test, and outperform competing $\mathcal{O}(n)$ and $\mathcal{O}(n\log n)$
tests.
\end{abstract}

\section{Introduction}

We consider the design of adaptive, nonparametric statistical tests
of dependence: that is, tests of whether a joint distribution $P_{xy}$
factorizes into the product of marginals $P_{x}P_{y}$. While classical
tests of dependence, such as Pearson's correlation and Kendall's $\tau$,
are able to detect monotonic relations between univariate variables,
more modern tests can address complex interactions, for instance changes
in variance of $X$ with the value of $Y$. Key to many recent tests
is to examine covariance or correlation between data features. These
interactions become significantly harder to detect, and the features
are more difficult to design when the data reside in high dimensions.

A basic nonlinear dependence measure is the Hilbert-Schmidt Independence
Criterion (HSIC), which is the Hilbert-Schmidt norm of the covariance
operator between feature mappings of the random variables \citep{Gretton2005,Gretton2008}.
Each random variable $X$ and $Y$ is mapped to a respective reproducing
kernel Hilbert space $\mathcal{H}_{k}$ and $\mathcal{H}_{l}$. For
sufficiently rich mappings, the covariance operator norm is zero if
and only if the variables are independent. A second basic nonlinear
dependence measure is the smoothed difference between the characteristic
function of the joint distribution, and that of the product of marginals.
When a particular smoothing function is used, the statistic corresponds
to the covariance between distances of X and Y variable pairs \citep{Feuerverger93,Cordis2007,Szekely2009},
yielding a simple test statistic. It has been shown by \citet{Sejdinovic2013}
that the distance covariance (and its generalization to semi-metrics)
is an instance of HSIC for an appropriate choice of kernels. A disadvantage
of these feature covariance statistics, however, is that they require
quadratic time to compute (besides in the special case of the distance
covariance with univariate real-valued variables, where \citet{HuoSze14}
achieve an $\mathcal{O}(n\log n)$ cost). Moreover, the feature covariance
statistics have intractable null distributions, and either a permutation
approach or the solution of an expensive eigenvalue problem \citep[e.g.][]{ZhaPetJanSch11}
is required for consistent estimation of the quantiles. Several approaches
were proposed by \citet{Zhang2016} to obtain faster tests along the
lines of HSIC. These include computing HSIC on finite-dimensional
feature mappings chosen as random Fourier features (RFFs) \citep{Rahimi2008},
a block-averaged statistic, and a Nyström approximation to the statistic.
Key to each of these approaches is a more efficient computation of
the statistic and its threshold under the null distribution: for RFFs,
the null distribution is a finite weighted sum of $\chi^{2}$ variables;
for the block-averaged statistic, the null distribution is asymptotically
normal; for Nyström, either a permutation approach is employed, or
the spectrum of the Nyström approximation to the kernel matrix is
used in approximating the null distribution. Each of these methods
costs significantly less than the $\mathcal{O}(n^{2})$ cost of the
full HSIC (the cost is linear in $n$, but also depends quadratically
on the number of features retained). A potential disadvantage of the
Nyström and Fourier approaches is that the features are not optimized
to maximize test power, but are chosen randomly. The block statistic
performs worse than both, due to the large variance of the statistic
under the null (which can be mitigated by observing more data).

In addition to feature covariances, correlation measures have also
been developed in infinite dimensional feature spaces: in particular,
\citet{BacJor02,FukGreSunSch08} proposed statistics on the correlation
operator in a reproducing kernel Hilbert space. While convergence
has been established for certain of these statistics, their computational
cost is high at $\mathcal{O}(n^{3})$, and test thresholds have relied
on permutation. A number of much faster approaches to testing based
on feature correlations have been proposed, however. For instance,
\citet{Dauxois1998} compute statistics of the correlation between
finite sets of basis functions, chosen for instance to be step functions
or low order B-splines. The cost of this approach is $\mathcal{O}(n)$.
This idea was extended by \citet{Lopez-Paz2013}, who computed the
canonical correlation between finite sets of basis functions chosen
as random Fourier features; in addition, they performed a copula transform
on the inputs, with a total cost of $\mathcal{O}(n\log n)$. Finally,
space partitioning approaches have also been proposed, based on statistics
such as the KL divergence, however these apply only to univariate
variables \citep{JMLR:v17:14-441}, or to multivariate variables of
low dimension \citep{GreGyo10} (that said, these tests have other
advantages of theoretical interest, notably distribution-independent
test thresholds). 

The approach we take is most closely related to HSIC on a finite set
of features. Our simplest test statistic, the Finite Set Independence
Criterion (FSIC), is an average of covariances of analytic functions
(i.e., features) defined on each of $X$ and $Y$. A normalized version
of the statistic (NFSIC) yields a distribution-independent asymptotic
test threshold. We show that our test is consistent, despite a finite
number of analytic features being used, via a generalization of arguments
in \citet{Chwialkowski2015}. As in recent work on two-sample testing
by \citet{Jitkrittum2016}, our test is \emph{adaptive} in the sense
that we choose our features on a held-out validation set to optimize
a lower bound on the test power. The design of features for independence
testing turns out to be quite different to the case of two-sample
testing, however: the task is to find correlated feature \emph{pairs}
on the respective marginal domains, rather than attempting to find
a single, high-dimensional feature representation for the \emph{entire}
$(x,y)$ (as we would need to do if we were comparing distributions
$P_{xy}$ and $Q_{xy}$, rather than testing a specific property of
$P_{xy}$).  We demonstrate the performance of our tests on several
challenging artificial and real-world datasets, including detection
of dependence between music and its year of appearance, and between
videos and captions. In these experiments, we outperform competing
linear and $\mathcal{O}(n\log n)$ time tests.

\section{Independence Criteria and Statistical Tests}

We introduce two test statistics: first, the Finite Set Independence
Criterion (FSIC), which builds on the principle that dependence can
be measured in terms of the covariance between data features. Next,
we propose a normalized version of this statistic (NFSIC), with a
simpler asymptotic distribution when $P_{xy}=P_{x}P_{y}$. We show
how to select features for the latter statistic to maximize a lower
bound on the power of its corresponding statistical test.

\subsection{The Finite Set Independence Criterion}

We begin by introducing the Hilbert-Schmidt Independence Criterion
(HSIC) as proposed in \citet{Gretton2005}, since our unnormalized
statistic is built along similar lines.  Consider two random variables
$X\in\mathcal{X}\subset\mathbb{R}^{d_{x}}$ and $Y\in\mathcal{Y}\subset\mathbb{R}^{d_{y}}$.
Denote by $P_{xy}$ the joint distribution between $X$ and $Y$;
$P_{x}$ and $P_{y}$ are the marginal distributions of $X$ and $Y$.
Let $\otimes$ denote the tensor product, such that $\left(a\otimes b\right)c=a\left\langle b,c\right\rangle $.
Assume that $k:\mathcal{X}\times\mathcal{X}\to\mathbb{R}$ and $l:\mathcal{Y}\times\mathcal{Y}\to\mathbb{R}$
are positive definite kernels associated with reproducing kernel Hilbert
spaces (RKHS) $\mathcal{H}_{k}$ and $\mathcal{H}_{l}$, respectively.
Let $\|\cdot\|_{HS}$ be the norm on the space of $\mathcal{H}_{l}\to\mathcal{H}_{k}$
Hilbert-Schmidt operators. Then, HSIC between $X$ and $Y$ is defined
as 
\begin{align}
 & \mathrm{HSIC}(X,Y)=\big\|\mu_{xy}-\mu_{x}\otimes\mu_{y}\big\|_{\mathrm{HS}}^{2}\nonumber \\
 & =\mathbb{E}_{(\mathbf{x},\mathbf{y}),(\mathbf{x}',\mathbf{y}')}\left[k(\mathbf{x},\mathbf{x}')l(\mathbf{y},\mathbf{y}')\right]\nonumber \\
 & \quad+\mathbb{E}_{\mathbf{x}}\mathbb{E}_{\boldsymbol{\mathbf{x}'}}[k(\mathbf{x},\mathbf{x}')]\mathbb{E}_{\mathbf{y}}\mathbb{E}_{\mathbf{y}'}[l(\mathbf{y},\mathbf{y}')]\nonumber \\
 & \quad-2\mathbb{E}_{(\mathbf{x},\mathbf{y})}\left[\mathbb{E}_{\mathbf{x}'}[k(\mathbf{x},\mathbf{x}')]\mathbb{E}_{\mathbf{y}'}[l(\mathbf{y},\mathbf{y}')]\right],\label{eq:hsic}
\end{align}
where $\mathbb{E}_{\mathbf{x}}:=\mathbb{E}_{\mathbf{x}\sim P_{x}}$,
$\mathbb{E}_{\mathbf{y}}:=\mathbb{E}_{\mathbf{y}\sim P_{y}}$, $\mathbb{E}_{(\mathbf{x},\mathbf{y})}:=\mathbb{E}_{(\mathbf{x},\mathbf{y})\sim P_{xy}}$,
and $\mathbf{x}'$ is an independent copy of $\mathbf{x}$. The mean
embedding of $P_{xy}$ belongs to the space of Hilbert-Schmidt operators
from $\mathcal{H}_{l}$ to $\mathcal{H}_{k}$, $\mu_{xy}:=\int_{\mathcal{X}\times\mathcal{Y}}k(\mathbf{x},\cdot)\otimes l(\mathbf{y},\cdot)\thinspace\mathrm{d}P_{xy}(\mathbf{x},\mathbf{y})\in\mathrm{HS}(\mathcal{H}_{l},\mathcal{H}_{k})$,
and the marginal mean embeddings are $\mu_{x}:=\int_{\mathcal{X}}k(\mathbf{x},\cdot)\thinspace\mathrm{d}P_{x}(\mathbf{x})\in\mathcal{H}_{k}$
and $\mu_{y}:=\int_{\mathcal{Y}}l(\mathbf{y},\cdot)\thinspace\mathrm{d}P_{y}(\mathbf{y})\in\mathcal{H}_{l}$
\citep{Smola2007}. \citet[Theorem 4]{Gretton2005} show that if the
kernels $k$ and $l$ are universal \citep{Steinwart2008} on compact
domains $\mathcal{X}$ and $\mathcal{Y}$, then $\mathrm{HSIC}(X,Y)=0$
if and only if $X$ and $Y$ are independent. Alternatively, \citet{Gretton15}
shows that it is sufficient for each of $k$ and $l$ to be characteristic
to their respective domains (meaning that distribution embeddings
are injective in each marginal domain: see \citet{Sriperumbudur2010}).
Given a joint sample $\mathsf{Z}_{n}=\{(\mathbf{x}_{i},\mathbf{y}_{i})\}_{i=1}^{n}\sim P_{xy}$,
an empirical estimator of HSIC can be computed in $\mathcal{O}(n^{2})$
time by replacing the population expectations in (\ref{eq:hsic})
with their corresponding empirical expectations based on $\mathsf{Z}_{n}$. 

We now propose our new linear-time dependence measure, the Finite
Set Independence Criterion (FSIC). Let $\mathcal{X}\subset\mathbb{R}^{d_{x}}$
and $\mathcal{Y}\subset\mathbb{R}^{d_{y}}$ be open sets. Define the
empirical measure $\nu:=\frac{1}{J}\sum_{i=1}^{J}\delta_{(\mathbf{v}_{i},\mathbf{w}_{i})}$
over $J$ test locations $V_{J}:=\{(\mathbf{v}_{i},\mathbf{w}_{i})\}_{i=1}^{J}\subset\mathcal{X}\times\mathcal{Y}$
where $\delta_{\mathbf{t}}$ denotes the Dirac measure centered on
$\mathbf{t}$, and $(\mathbf{v}_{i},\mathbf{w}_{i})$ are realizations
from an absolutely continuous distribution (wrt the Lebesgue measure).
Write $\mathbb{E}_{\mathbf{x}\mathbf{y}}$ for $\mathbb{E}_{(\mathbf{x},\mathbf{y})\sim P_{xy}}$.
The idea is to see $\mu_{xy}(\mathbf{v},\mathbf{w})=\mathbb{E}_{\mathbf{x}\mathbf{y}}[k(\mathbf{x},\mathbf{v})l(\mathbf{y},\mathbf{w})],\mu_{x}(\mathbf{v})=\mathbb{E}_{\mathbf{x}}[k(\mathbf{x},\mathbf{v})]$
and $\mu_{y}(\mathbf{w})=\mathbb{E}_{\mathbf{y}}[l(\mathbf{y},\mathbf{w})]$
as smooth functions, and consider an $L^{2}(\mathcal{X}\times\mathcal{Y},\nu)$
distance between $\mu_{xy}$ and $\mu_{x}\mu_{y}$ instead of a Hilbert-Schmidt
distance as in HSIC \citep{Gretton2005}. Let $\mu_{x}\mu_{y}(\mathbf{x},\mathbf{y}):=\mu_{x}(\mathbf{x})\mu_{y}(\mathbf{y})$.
FSIC is defined as 
\begin{align*}
 & \mathrm{FSIC}^{2}(X,Y):=\|\mu_{xy}-\mu_{x}\mu_{y}\|_{L^{2}(\mathcal{X}\times\mathcal{Y},\text{\ensuremath{\nu}})}^{2}\\
 & =\int_{\mathcal{X}}\int_{\mathcal{Y}}(\mu_{xy}(\mathbf{x},\mathbf{y})-\mu_{x}(\mathbf{x})\mu_{y}(\mathbf{y}))^{2}\thinspace\mathrm{d}\nu(\mathbf{x},\mathbf{y})\\
 & :=\frac{1}{J}\sum_{i=1}^{J}u(\mathbf{v}_{i},\mathbf{w}_{i})^{2}=\frac{1}{J}\|\mathbf{u}\|_{2}^{2},\text{ where }
\end{align*}
\vspace{-5mm}
\begin{align}
 & u(\mathbf{v},\mathbf{w}):=\mu_{xy}(\mathbf{v},\mathbf{w})-\mu_{x}(\mathbf{v})\mu_{y}(\mathbf{w})\nonumber \\
 & =\mathbb{E}_{\mathbf{x}\mathbf{y}}[k(\mathbf{x},\mathbf{v})l(\mathbf{y},\mathbf{w})]-\mathbb{E}_{\mathbf{x}}[k(\mathbf{x},\mathbf{v})]\mathbb{E}_{\mathbf{y}}[l(\mathbf{y},\mathbf{w})],\label{eq:mean_embed_diff}\\
 & =\mathrm{cov_{\mathbf{x}\mathbf{y}}}[k(\mathbf{x},\mathbf{v}),l(\mathbf{y},\mathbf{w})],\nonumber 
\end{align}
and $\mathbf{u}:=(u(\mathbf{v}_{1},\mathbf{w}_{1}),\ldots,u(\mathbf{v}_{J},\mathbf{w}_{J}))^{\top}$. 

Our first result in Proposition \ref{prop:fsic_dependence_measure}
states that $\mathrm{FSIC}(X,Y)$ almost surely defines a dependence
measure for the random variables $X$ and $Y$, provided that the
product kernel on the joint space $\mathcal{X}\times\mathcal{Y}$
is characteristic and analytic (see Definition \ref{def:analytic_kernel}). 
\begin{defn}[Analytic kernels \citep{Chwialkowski2015}]
\label{def:analytic_kernel} Let $\mathcal{X}$ be an open set in
$\mathbb{R}^{d}$. A positive definite kernel $k:\mathcal{X}\times\mathcal{X}\to\mathbb{R}$
is said to be analytic on its domain $\mathcal{X}\times\mathcal{X}$
if for all $\mathbf{v}\in\mathcal{X}$, $f(\mathbf{x}):=k(\mathbf{x},\mathbf{v})$
is an analytic function on $\mathcal{X}$.

\begin{assumption}
\label{as:prod_ker_bound}
The kernels $k:\mathcal{X}\times\mathcal{X} \to\mathbb{R}$ and $l:\mathcal{Y}\times\mathcal{Y}\to\mathbb{R}$ are bounded by $B_{k}$ and $B_{l}$ respectively, and the product kernel $g((\mathbf{x},\mathbf{y}),(\mathbf{x}',\mathbf{y}')):=k(\mathbf{x},\mathbf{x}')l(\mathbf{y},\mathbf{y}')$ is characteristic \citep[Definition 6]{Sriperumbudur2010}, and analytic (Definition\,\ref{def:analytic_kernel}) on $(\mathcal{X}\times\mathcal{Y})\times(\mathcal{X}\times\mathcal{Y})$.  
\end{assumption}
\end{defn}
\begin{prop}[FSIC is a dependence measure]
\label{prop:fsic_dependence_measure} 

Assume that 
\begin{enumerate}
\item Assumption \ref{as:prod_ker_bound} holds.
\item The test locations $V_{J}=\{(\mathbf{v}_{i},\mathbf{w}_{i})\}_{i=1}^{J}$
are drawn from an absolutely continuous distribution. 
\end{enumerate}
Then, almost surely, $\mathrm{FSIC}(X,Y)=\frac{1}{\sqrt{J}}\|\mathbf{u}\|_{2}=0$
if and only if $X$ and $Y$ are independent.

\end{prop}
\begin{proof}
Since $g$ is characteristic, the mean embedding map $\Pi_{g}:P\mapsto\mathbb{E}_{(\mathbf{x},\mathbf{y})\sim P}[g((\mathbf{x},\mathbf{y}),\cdot)]$
is injective \citep[Section 3]{Sriperumbudur2010}, where $P$ is
a probability distribution on $\mathcal{X}\times\mathcal{Y}$. Since
$g$ is analytic, by Lemma \ref{lem:analytic_rkhs} (Appendix), $\mu_{xy}$
and $\mu_{x}\mu_{y}$ are analytic functions. Thus, Lemma\,\ref{lem:metric_prob}
(Appendix, setting $\Lambda=\Pi_{g}$) guarantees that $\mathrm{FSIC}(X,Y)=0\iff P_{xy}=P_{x}P_{y}\iff X$
and $Y$ are independent almost surely.
\end{proof}

FSIC uses $\mu_{xy}$ as a proxy for $P_{xy}$, and $\mu_{x}\mu_{y}$
as a proxy for $P_{x}P_{y}$. Proposition \ref{prop:fsic_dependence_measure}
suggests that, to detect the dependence between $X$ and $Y$, it
is sufficient to evaluate at a finite number of locations (defined
by $V_{J}$) the difference of the population joint embedding $\mu_{xy}$
and the embedding of the product of the marginal distributions $\mu_{x}\mu_{y}$.
A brief explanation to justify this property is as follows. If $P_{xy}=P_{x}P_{y}$,
then $\rho(\mathbf{v},\mathbf{w}):=\mu_{xy}(\mathbf{v},\mathbf{w})-\mu_{x}\mu_{y}(\mathbf{v},\mathbf{w})$
is zero, and $\mathrm{FSIC}(X,Y)=0$ for any $V_{J}$. If $P_{xy}\neq P_{x}P_{y}$,
then $\rho$ will not be a zero function, since the mean embedding
map is injective (require the product kernel to be characteristic).
Using the same argument as in \citet{Chwialkowski2015}, since $k$
and $l$ are analytic, $\rho$ is also analytic, and the set of roots
$R:=\{(\mathbf{v},\mathbf{w})\mid\rho(\mathbf{v},\mathbf{w})=0\}$
has Lebesgue measure zero. Thus, it is sufficient to draw $(\mathbf{v},\mathbf{w})$
from an absolutely continuous distribution, as we are guaranteed that
$(\mathbf{v},\mathbf{w})\notin R$ giving $\mathrm{FSIC}(X,Y)>0$. 

For FSIC to be a dependence measure, the product kernel is required
to be characteristic and analytic. We next show in Proposition \ref{prop:product_gaussian_kers}
that Gaussian kernels $k$ and $l$ yield such a product kernel.

\newcommand{\prodgkers}{Let $k(\mathbf{x},\mathbf{x}')=\exp\left(-(\mathbf{x}-\mathbf{x}')^{\top}\mathbf{A}(\mathbf{x}-\mathbf{x}')\right)$
and $l(\mathbf{y},\mathbf{y}')=\exp\left(-(\mathbf{y}-\mathbf{y}')^{\top}\mathbf{B}(\mathbf{y}-\mathbf{y}')\right)$
be Gaussian kernels on $\mathbb{R}^{d_{x}}\times\mathbb{R}^{d_{x}}$
and $\mathbb{R}^{d_{y}}\times\mathbb{R}^{d_{y}}$ respectively, for
positive definite matrices $\mathbf{A}$ and $\mathbf{B}$. Then,
$g((\mathbf{x},\mathbf{y}),(\mathbf{x}',\mathbf{y}'))=k(\mathbf{x},\mathbf{x}')l(\mathbf{y},\mathbf{y}')$
is characteristic and analytic on $(\mathbb{R}^{d_{x}}\times\mathbb{R}^{d_{y}})\times(\mathbb{R}^{d_{x}}\times\mathbb{R}^{d_{y}})$.}
\begin{prop}[A product of Gaussian kernels is characteristic and analytic]
\label{prop:product_gaussian_kers} \prodgkers{} \vspace{-2mm}
\end{prop}
\begin{proof}[Proof (sketch)]
 The main idea is to use the fact a Gaussian kernel is analytic,
and a product of Gaussian kernels is a Gaussian kernel on the pair
of variables. See the full proof in Appendix \ref{sec:proof_prod_kgauss_ana}.
\end{proof}
\textbf{Plug-in Estimator} We now give an empirical estimator of FSIC.
Assume that we observe a joint sample $\mathsf{Z}_{n}:=\{(\mathbf{x}_{i},\mathbf{y}_{i})\}_{i=1}^{n}\stackrel{i.i.d.}{\sim}P_{xy}$.
Unbiased estimators of $\mu_{xy}(\mathbf{v},\mathbf{w})$ and $\mu_{x}\mu_{y}(\mathbf{v},\mathbf{w})$
are $\hat{\mu}_{xy}(\mathbf{v},\mathbf{w}):=\frac{1}{n}\sum_{i=1}^{n}k(\mathbf{x}_{i},\mathbf{v})l(\mathbf{y}_{i},\mathbf{w})$
and $\widehat{\mu_{x}\mu_{y}}(\mathbf{v},\mathbf{w}):=\frac{1}{n(n-1)}\sum_{i=1}^{n}\sum_{j\neq i}k(\mathbf{x}_{i},\mathbf{v})l(\mathbf{y}_{j},\mathbf{w})$,
respectively. A straightforward empirical estimator of $\mathrm{FSIC}^{2}$
is then given by{\small{}
\begin{align}
\widehat{\mathrm{FSIC^{2}}}(\mathsf{Z}_{n}) & =\frac{1}{J}\sum_{i=1}^{J}\hat{u}(\mathbf{v}_{i},\mathbf{w}_{i})^{2},\nonumber \\
\hat{u}(\mathbf{v},\mathbf{w}) & :=\hat{\mu}_{xy}(\mathbf{v},\mathbf{w})-\widehat{\mu_{x}\mu_{y}}(\mathbf{v},\mathbf{w})\label{eq:ustat0}\\
 & =\frac{2}{n(n-1)}\sum_{i<j}h_{(\mathbf{v},\mathbf{w})}((\mathbf{x}_{i},\mathbf{y}_{i}),(\mathbf{x}_{j},\mathbf{y}_{j})),\label{eq:ustat}
\end{align}
}where $h_{(\mathbf{v},\mathbf{w})}((\mathbf{x},\mathbf{y}),(\mathbf{x}',\mathbf{y}')):=\frac{1}{2}(k(\mathbf{x},\mathbf{v})-k(\mathbf{x}',\mathbf{v}))(l(\mathbf{y},\mathbf{w})-l(\mathbf{y}',\mathbf{w}))$.
For conciseness, we define $\hat{\mathbf{u}}:=(\hat{u}_{1},\ldots,\hat{u}_{J})^{\top}\in\mathbb{R}^{J}$
where $\hat{u}_{i}:=\hat{u}(\mathbf{v}_{i},\mathbf{w}_{i})$ so that
$\widehat{\mathrm{FSIC^{2}}}(\mathsf{Z}_{n})=\frac{1}{J}\hat{\mathbf{u}}^{\top}\hat{\mathbf{u}}$. 

$\widehat{\mathrm{FSIC^{2}}}$ can be efficiently computed in $\mathcal{O}((d_{x}+d_{y})Jn)$
time {[}see (\ref{eq:ustat0}){]}, assuming that the runtime complexity
of evaluating $k(\mathbf{x},\mathbf{v})$ is $\mathcal{O}(d_{x})$
and that of $l(\mathbf{y},\mathbf{w})$ is $\mathcal{O}(d_{y})$.
The unbiasedness of $\widehat{\mu_{x}\mu_{y}}$ is necessary for $\eqref{eq:ustat}$
to be a U-statistic. This fact and the rewriting of $\widehat{\mathrm{FSIC^{2}}}$
in terms of $h_{(\mathbf{v},\mathbf{w})}((\mathbf{x},\mathbf{y}),(\mathbf{x}',\mathbf{y}'))$
will be exploited when the asymptotic distribution of $\hat{\mathbf{u}}$
is derived (Proposition \ref{prop:asymp_u}). 

Since $\mathrm{FSIC}$ satisfies $\mathrm{FSIC}(X,Y)=0\iff X\perp Y$,
in principle its empirical estimator can be used as a test statistic
for an independence test proposing a null hypothesis $H_{0}:$ ``$X$
and $Y$ are independent'' against an alternative $H_{1}:$ ``$X$
and $Y$ are dependent''. The null distribution (i.e., distribution
of the test statistic assuming that $H_{0}$ is true) is challenging
to obtain, however and depends on the unknown $P_{xy}$. This prompts
us to consider a normalized version of $\mathrm{FSIC}$ whose asymptotic
null distribution of a convenient form. We first derive the asymptotic
distribution of $\hat{\mathbf{u}}$ in Proposition \ref{prop:asymp_u},
which we  use to derive the normalized test statistic in Theorem \ref{thm:nfsic_good_test}.
As a shorthand, we write $\mathbf{z}:=(\mathbf{x},\mathbf{y})$, and
$\mathbf{t}:=(\mathbf{v},\mathbf{w})$.
\begin{prop}[Asymptotic distribution of $\hat{\mathbf{u}}$]
\label{prop:asymp_u} Define $\tilde{k}(\mathbf{x},\mathbf{v}):=k(\mathbf{x},\mathbf{v})-\mathbb{E}_{\mathbf{x}'}k(\mathbf{x}',\mathbf{v})$,
and $\tilde{l}(\mathbf{y},\mathbf{w}):=l(\mathbf{y},\mathbf{w})-\mathbb{E}_{\mathbf{y}'}l(\mathbf{y}',\mathbf{w})$.
Then, under both $H_{0}$ and $H_{1}$, for any fixed locations $\mathbf{t}$
and $\mathbf{t}'$, {\small{}
\begin{align*}
 & \mathrm{cov_{\mathbf{z}}}[\hat{u}(\mathbf{t}),\hat{u}(\mathbf{t}')]\xrightarrow{n\to\infty}\mathrm{cov_{\mathbf{z}}}[\tilde{k}(\mathbf{x},\mathbf{v})\tilde{l}(\mathbf{y},\mathbf{w}),\tilde{k}(\mathbf{x},\mathbf{v}')\tilde{l}(\mathbf{y},\mathbf{w}')]\\
 & =\mathbb{E}_{\mathbf{x}\mathbf{y}}[\big(\tilde{k}(\mathbf{x},\mathbf{v})\tilde{l}(\mathbf{y},\mathbf{w})-u(\mathbf{t})\big)\big(\tilde{k}(\mathbf{x},\mathbf{v}')\tilde{l}(\mathbf{y},\mathbf{w}')-u(\mathbf{t}')\big)],
\end{align*}
}where $u(\mathbf{v},\mathbf{w})$ is given in (\ref{eq:mean_embed_diff}),
and $\hat{u}(\mathbf{v},\mathbf{w})$ is defined in (\ref{eq:ustat}).
Second, if $0<\mathrm{cov}_{\mathbf{z}}[\hat{u}(\mathbf{t}_{i}),\hat{u}(\mathbf{t}_{i})]<\infty$
for $i=1,\ldots,J$, then $\sqrt{n}(\hat{\mathbf{u}}-\mathbf{u})\stackrel{d}{\to}\mathcal{N}(\mathbf{0},\boldsymbol{\Sigma})$
as $n\to\infty$, where $\Sigma_{ij}=\mathrm{cov}[\hat{u}(\mathbf{t}_{i}),\hat{u}(\mathbf{t}_{j})]$
and $\mathbf{u}:=(u(\mathbf{t}_{1}),\ldots,u(\mathbf{t}_{J}))^{\top}$.
\end{prop}
\begin{proof}
We first note that for a fixed $\mathbf{t}=(\mathbf{v},\mathbf{w})$,
$\hat{u}(\mathbf{v},\mathbf{w})$ is a one-sample second-order U-statistic
\citep[Section 5.1.3]{Serfling2009} with a U-statistic kernel $h_{\mathbf{t}}$
where $h_{\mathbf{t}}(\mathbf{a},\mathbf{b})=h_{\mathbf{t}}(\mathbf{b},\mathbf{a})$.
Thus, by \citet[Section 5.1, Theorem 1]{Kowalski2008}, it follows
directly that $\mathrm{cov}[\hat{u}(\mathbf{t}),\hat{u}(\mathbf{t}')]=4\mathrm{cov}_{\mathbf{z}}[\mathbb{E}_{\mathbf{a}}h_{\mathbf{t}}(\mathbf{z},\mathbf{a}),\mathbb{E}_{\mathbf{b}}h_{\mathbf{t}'}(\mathbf{z},\mathbf{b})]$.
Substituting $h_{\mathbf{t}}$ with its definition yields the first
claim, where we note that $\mathbb{E}_{\mathbf{x}\mathbf{y}}[\tilde{k}(\mathbf{x},\mathbf{v})\tilde{l}(\mathbf{y},\mathbf{w})]=u(\mathbf{v},\mathbf{w})$. 

For the second claim, since $\hat{\mathbf{u}}$ is a multivariate
one-sample U-statistic, by \citet[Theorem 6.1.6]{Lehmann1999} and
\citet[Section 5.1, Theorem 1]{Kowalski2008}, it follows that $\sqrt{n}(\hat{\mathbf{u}}-\mathbf{u})\stackrel{d}{\to}\mathcal{N}(\mathbf{0},\boldsymbol{\Sigma})$
as $n\to\infty$, where $\Sigma_{ij}=\mathrm{cov}[\hat{u}(\mathbf{t}_{i}),\hat{u}(\mathbf{t}_{j})]$.
\end{proof}

Recall from Proposition \ref{prop:fsic_dependence_measure} that $\mathbf{u}=0$
holds almost surely under $H_{0}$. The asymptotic normality in the
second claim of Proposition \ref{prop:asymp_u} implies that $n\widehat{\mathrm{FSIC^{2}}}=\frac{n}{J}\hat{\mathbf{u}}^{\top}\hat{\mathbf{u}}$
converges in distribution to a sum of $J$ dependent weighted $\chi^{2}$
random variables. The dependence comes from the fact that the coordinates
$\hat{u}_{1}\ldots,\hat{u}_{J}$ of $\hat{\mathbf{u}}$ all depend
on the sample $\mathsf{Z}_{n}$. This null distribution is not analytically
tractable, and requires a large number of simulations to compute the
rejection threshold $T_{\alpha}$ for a given significance value $\alpha$. 

\subsection{Normalized FSIC and Adaptive Test\label{sec:independence_test}}

For the purpose of an independence test, we will consider a normalized
variant of $\widehat{\mathrm{FSIC^{2}}}$, which we call \textcolor{black}{$\widehat{\mathrm{NFSIC^{2}}}$,}
whose tractable asymptotic null distribution is $\chi^{2}(J)$, the
chi-squared distribution with $J$ degrees of freedom. We then show
that the independence test defined by \textcolor{black}{$\widehat{\mathrm{NFSIC^{2}}}$}
is consistent. These results are given in Theorem \ref{thm:nfsic_good_test}.

\begin{restatable}[Independence test using $\widehat{\mathrm{NFSIC^2}}$ is consistent]{thm}{nfsicgoodtest}
\label{thm:nfsic_good_test}

Let $\hat{\boldsymbol{\Sigma}}$ be a consistent estimate of $\boldsymbol{\Sigma}$
based on the joint sample $\mathsf{Z}_{n}$. \textcolor{black}{The
$\widehat{\mathrm{NFSIC^{2}}}$} statistic is defined as $\hat{\lambda}_{n}:=n\hat{\mathbf{u}}^{\top}\left(\hat{\boldsymbol{\Sigma}}+\gamma_{n}\mathbf{I}\right)^{-1}\hat{\mathbf{u}}$
where $\gamma_{n}\ge0$ is a regularization parameter. Assume that

\begin{enumerate}
\item Assumption \ref{as:prod_ker_bound} holds.
\item $\boldsymbol{\Sigma}$ is invertible almost surely with respect to
$V_{J}=\{(\mathbf{v}_{i},\mathbf{w}_{i})\}_{i=1}^{J}$ drawn from
an absolutely continuous distribution.
\item $\lim_{n\to\infty}\gamma_{n}=0$.
\end{enumerate}
Then, for any $k,l$ and $V_{J}$ satisfying the assumptions,
\begin{enumerate}
\item Under $H_{0}$, $\hat{\lambda}_{n}\stackrel{d}{\to}\chi^{2}(J)$ as
$n\to\infty$.
\item Under $H_{1}$, for any $r\in\mathbb{R}$, $\lim_{n\to\infty}\mathbb{P}\left(\hat{\lambda}_{n}\ge r\right)=1$
almost surely. That is, the independence test based on $\widehat{\mathrm{NFSIC^{2}}}$
is consistent.
\end{enumerate}
\end{restatable}
\begin{proof}[Proof (sketch) ]
 Under $H_{0}$, $n\hat{\mathbf{u}}^{\top}(\hat{\boldsymbol{\Sigma}}+\gamma_{n}\mathbf{I})^{-1}\hat{\mathbf{u}}$
asymptotically follows $\chi^{2}(J)$ because $\sqrt{n}\hat{\mathbf{u}}$
is asymptotically normally distributed (see Proposition \ref{prop:asymp_u}).
Claim 2 builds on the result in Proposition \ref{prop:fsic_dependence_measure}
stating that $\mathbf{u}\neq0$ under $H_{1}$; it follows using the
convergence of $\hat{\mathbf{u}}$ to $\mathbf{u}$. The full proof
can be found in Appendix \ref{sec:proof_nfsic_consistent}.
\end{proof}
Theorem \ref{thm:nfsic_good_test} states that if $H_{1}$ holds,
the statistic can be arbitrarily large as $n$ increases, allowing
$H_{0}$ to be rejected for any fixed threshold. Asymptotically the
test threshold $T_{\alpha}$ is given by the $(1-\alpha)$-quantile
of $\chi^{2}(J)$ and is independent of $n$. The assumption on the
consistency of $\hat{\mathbf{\Sigma}}$ is required to obtain the
asymptotic chi-squared distribution. The regularization parameter
$\gamma_{n}$ is to ensure that $(\hat{\mathbf{\Sigma}}+\gamma_{n}\mathbf{I})^{-1}$
can be stably computed. In practice, $\gamma_{n}$ requires no tuning,
and can be set to be a very small constant.

The next proposition states that the computational complexity of the
\textcolor{black}{$\widehat{\mathrm{NFSIC^{2}}}$} estimator is linear
in both the input dimension and sample size, and that it can be expressed
in terms of the $\mathbf{K=}[K{}_{ij}]=[k(\mathbf{v}_{i},\mathbf{x}_{j})]\in\mathbb{R}^{J\times n},\mathbf{L}=[L_{ij}]=[l(\mathbf{w}_{i},\mathbf{y}_{j})]\in\mathbb{R}^{J\times n}$
matrices.
\begin{prop}[An empirical estimator of \textcolor{black}{$\widehat{\mathrm{NFSIC^{2}}}$}]
\label{prop:estimators} Let $\boldsymbol{1}_{n}:=(1,\ldots,1)^{\top}\in\mathbb{R}^{n}$.
Denote by $\circ$ the element-wise matrix product. Then, 

\begin{enumerate}
\item $\hat{\mathbf{u}}=\frac{\left(\mathbf{K}\circ\mathbf{L}\right)\boldsymbol{1}_{n}}{n-1}-\frac{\left(\mathbf{K}\boldsymbol{1}_{n}\right)\circ\left(\mathbf{L}\boldsymbol{1}_{n}\right)}{n(n-1)}$.
\item A consistent estimator for $\boldsymbol{\Sigma}$ is $\hat{\boldsymbol{\Sigma}}=\frac{\Gamma\Gamma^{\top}}{n}$
where 
\begin{align*}
\Gamma & :=(\mathbf{K}-n^{-1}\mathbf{K}\boldsymbol{1}_{n}\boldsymbol{1}_{n}^{\top})\circ(\mathbf{L}-n^{-1}\mathbf{L}\boldsymbol{1}_{n}\boldsymbol{1}_{n}^{\top})-\hat{\mathbf{u}}^{b}\boldsymbol{1}_{n}^{\top},\\
\hat{\mathbf{u}}^{b} & =n^{-1}\left(\mathbf{K}\circ\mathbf{L}\right)\boldsymbol{1}_{n}-n^{-2}\left(\mathbf{K}\boldsymbol{1}_{n}\right)\circ\left(\mathbf{L}\boldsymbol{1}_{n}\right).
\end{align*}
\end{enumerate}
Assume that the complexity of the kernel evaluation is linear in the
input dimension. Then the test statistic $\hat{\lambda}_{n}=n\hat{\mathbf{u}}^{\top}\left(\hat{\boldsymbol{\Sigma}}+\gamma_{n}\mathbf{I}\right)^{-1}\hat{\mathbf{u}}$
can be computed in $\mathcal{O}(J^{3}+J^{2}n+(d_{x}+d_{y})Jn)$ time. 

\end{prop}
\begin{proof}[Proof (sketch)]
Claim 1 for $\hat{\mathbf{u}}$ is straightforward. The expression
for $\hat{\mathbf{\Sigma}}$ in claim 2 follows directly from the
asymptotic covariance expression in Proposition \ref{prop:asymp_u}.
The consistency of $\hat{\mathbf{\Sigma}}$ can be obtained by noting
that the finite sample bound for $\mathbb{P}(\|\hat{\mathbf{\Sigma}}-\mathbf{\Sigma}\|_{F}>t)$
decreases as $n$ increases. This is implicitly shown in Appendix
\ref{sec:bound_sigma} and its following sections.
\end{proof}
Although the dependency of the estimator on $J$ is cubic, we empirically
observe that only a small value of $J$ is required (see Section \ref{sec:experiments}).
The number of test locations $J$ relates to the number of regions
in $\mathcal{X}\times\mathcal{Y}$ of $p_{xy}$ and $p_{x}p_{y}$
that differ (see Figure \ref{fig:illus_nfsic}). In particular, $J$
need not increase with $n$ for test consistency. 

Our final theoretical result gives a lower bound on the test power
of \textcolor{black}{$\widehat{\mathrm{NFSIC^{2}}}$} i.e., the probability
of correctly rejecting $H_{0}$. We will use this lower bound as the
objective function to determine $V_{J}$ and the kernel parameters.
Let $\|\cdot\|_{F}$ be the Frobenius norm.

\begin{restatable}[A lower bound on the test power]{thm}{lbpow}
\label{thm:lower_bound_pow}Let $\mathrm{NFSIC^{2}}(X,Y):=\lambda_{n}:=n\mathbf{u}^{\top}\boldsymbol{\Sigma}^{-1}\mathbf{u}$.
Let $\mathcal{K}$ be a kernel class for $k$, $\mathcal{L}$ be a
kernel class for $l$, and $\mathcal{V}$ be a collection with each
element being a set of $J$ locations. Assume that 
\begin{enumerate}
\item There exist finite $B_{k}$ and $B_{l}$ such that $\sup_{k\in\mathcal{K}}\sup_{\mathbf{x},\mathbf{x}'\in\mathcal{X}}|k(\mathbf{x},\mathbf{x}')|\le B_{k}$
and $\sup_{l\in\mathcal{L}}\sup_{\mathbf{y},\mathbf{y}'\in\mathcal{Y}}|l(\mathbf{y},\mathbf{y}')|\le B_{l}$. 
\item $\tilde{c}:=\sup_{k\in\mathcal{K}}\sup_{l\in\mathcal{L}}\sup_{V_{J}\in\mathcal{V}}\|\boldsymbol{\Sigma}^{-1}\|_{F}<\infty$. 
\end{enumerate}
Then, for any $k\in\mathcal{K},l\in\mathcal{L},V_{J}\in\mathcal{V}$,
and $\lambda_{n}\ge r$, the test power satisfies $\mathbb{P}\left(\hat{\lambda}_{n}\ge r\right)\ge L(\lambda_{n})$
where
\begin{align*}
L(\lambda_{n}) & =1-62e^{-\xi_{1}\gamma_{n}^{2}(\lambda_{n}-r)^{2}/n}-2e^{-\lfloor0.5n\rfloor(\lambda_{n}-r)^{2}/\left[\xi_{2}n^{2}\right]}\\
 & -2e^{-\left[(\lambda_{n}-r)\gamma_{n}(n-1)/3-\xi_{3}n-c_{3}\gamma_{n}^{2}n(n-1)\right]^{2}/\left[\xi_{4}n^{2}(n-1)\right]},
\end{align*}
 $\lfloor\cdot\rfloor$ is the floor function, $\xi_{1}:=\frac{1}{3^{2}c_{1}^{2}J^{2}B^{*}}$,
$\xi_{2}:=72c_{2}^{2}JB^{2}$, $B:=B_{k}B_{l}$, $\xi_{3}:=8c_{1}B^{2}J$,
$c_{3}:=4B^{2}J\tilde{c}^{2}$, $\xi_{4}:=2^{8}B^{4}J^{2}c_{1}^{2}$,
$c_{1}:=4B^{2}J\sqrt{J}\tilde{c}$, $c_{2}:=4B\sqrt{J}\tilde{c}$,
and $B^{*}$ is a constant depending on only $B_{k}$ and $B_{l}$.
Moreover, for sufficiently large fixed $n$, $L(\lambda_{n})$ is
increasing in $\lambda_{n}$. 

\end{restatable}

We provide the proof in Appendix \ref{sec:proof_lb_pow}. To put Theorem
\ref{thm:lower_bound_pow} into perspective, let $\theta_{x}$ and
$\theta_{y}$ be the parameters of the kernels $k$ and $l$, respectively.
We denote by $\theta=\{\theta_{x},\theta_{y},V_{J}\}$ the collection
of all tuning parameters of the test. Assume that $\mathcal{K}=\left\{ (\mathbf{x},\mathbf{v})\mapsto\exp\left(-\frac{\|\mathbf{x}-\mathbf{v}\|^{2}}{2\sigma_{x}^{2}}\right)\mid\sigma_{x}^{2}\in[\sigma_{x,l}^{2},\sigma_{x,u}^{2}]\right\} =:\mathcal{K}_{g}$
for some $0<\sigma_{x,l}^{2}<\sigma_{x,u}^{2}<\infty$ and $\mathcal{L}=\left\{ (\mathbf{y},\mathbf{w})\mapsto\exp\left(-\frac{\|\mathbf{y}-\mathbf{w}\|^{2}}{2\sigma_{y}^{2}}\right)\mid\sigma_{y}^{2}\in[\sigma_{y,l}^{2},\sigma_{y,u}^{2}]\right\} =:\mathcal{L}_{g}$
for some $0<\sigma_{y,l}^{2}<\sigma_{y,u}^{2}<\infty$ are Gaussian
kernel classes. Then, in Theorem \ref{thm:lower_bound_pow}, $B=B_{k}=B_{l}=1$,
and $B^{*}=2$. The assumption $\tilde{c}<\infty$ is a technical
condition to guarantee that the test power lower bound is finite for
all $\theta$ defined by the feasible sets $\mathcal{K},\mathcal{L},$
and $\mathcal{V}$. Let $\mathcal{V}_{\epsilon,r}:=\big\{ V_{J}\mid\|\mathbf{v}_{i}\|^{2},\|\mathbf{w}_{i}\|^{2}\le r\text{ and }\|\mathbf{v}_{i}-\mathbf{v}_{j}\|_{2}^{2}+\|\mathbf{w}_{i}-\mathbf{w}_{j}\|_{2}^{2}\ge\epsilon,\text{ for all }i\neq j\big\}$.
If we set $\mathcal{K}=\mathcal{K}_{g},\mathcal{L}=\mathcal{L}_{g},$
and $\mathcal{V}=\mathcal{V}_{\epsilon,r}$ for some $\epsilon,r>0$,
then $\tilde{c}<\infty$ as $\mathcal{K}_{g},\mathcal{L}_{g},$ and
$\mathcal{V}_{\epsilon,r}$ are compact. In practice, these conditions
do not necessarily create restrictions as they almost always hold
implicitly. We show in Appendix \ref{sec:pow_vs_J} that the objective
function used to choose $V_{J}$ will discourage any two locations
to be in the same neighborhood.

\textbf{Parameter Tuning} The test power lower bound $L(\lambda_{n})$
in Theorem \ref{thm:lower_bound_pow} is a function of $\lambda_{n}=n\mathbf{u}^{\top}\mathbf{\Sigma}^{-1}\mathbf{u}$
which is the population counterpart of the test statistic $\hat{\lambda}_{n}$.
As in FSIC, it can be shown that $\lambda_{n}=0$ if and only if $X$
are $Y$ are independent (from Proposition \ref{prop:fsic_dependence_measure}).
If $X$ and $Y$ are dependent, then $\lambda_{n}>0$. According to
Theorem \ref{thm:lower_bound_pow}, for a sufficiently large $n$,
the test power lower bound is increasing in $\lambda_{n}$. One can
therefore think of $\lambda_{n}$ (a function of $\theta$) as representing
how easily the test rejects $H_{0}$ given a problem $P_{xy}$. The
higher the $\lambda_{n}$, the greater the lower bound on the test
power, and thus the more likely it is that the test will reject $H_{0}$
when it is false.

In light of this reasoning, we propose setting $\theta$ to $\theta^{*}=\arg\max_{\theta}\lambda_{n}$.
That this procedure is also valid under $H_{0}$ can be seen as follows.
Under $H_{0}$, $\theta^{*}=\arg\max_{\theta}0$ will be arbitrary.
Since Theorem \ref{thm:lower_bound_pow} guarantees that $\hat{\lambda}_{n}\stackrel{d}{\to}\chi^{2}(J)$
as $n\to\infty$ for any $\theta$, the asymptotic null distribution
does not change by using $\theta^{*}$. In practice, $\lambda_{n}$
is a population quantity which is unknown. We propose dividing the
sample $\mathsf{Z}_{n}$ into two disjoint sets: training and test
sets. The training set is used to optimize for $\theta^{*}$, and
the test set is used for the actual independence test with the optimized
$\theta^{*}$. The splitting is to guarantee the independence of $\theta^{*}$
and the test sample, which is an assumption of Theorem \ref{thm:nfsic_good_test}.

To better under $\widehat{\mathrm{NFSIC^{2}}}$, we visualize $\hat{\mu}_{xy}(\mathbf{v},\mathbf{w}),\widehat{\mu_{x}\mu_{y}}(\mathbf{v},\mathbf{w})$
and $\hat{\mathbf{\Sigma}}(\mathbf{v},\mathbf{w})$ as a function
of one test location $(\mathbf{v},\mathbf{w})$ on a simple toy problem.
In this problem, $Y=-X+Z$ where $Z\sim\mathcal{N}(0,0.3^{2})$. As
we consider only one location $(J=1)$, $\hat{\mathbf{\Sigma}}(\mathbf{v},\mathbf{w})$
is a scalar. The statistic can be written as $\hat{\lambda}_{n}=n\frac{\left(\hat{\mu}_{xy}(\mathbf{v},\mathbf{w})-\widehat{\mu_{x}\mu_{y}}(\mathbf{v},\mathbf{w})\right)^{2}}{\hat{\mathbf{\Sigma}}(\mathbf{v},\mathbf{w})}$.
These components are shown in Figure \ref{fig:illus_nfsic}, where
we use Gaussian kernels for both $X$ and $Y$, and the horizontal
and vertical axes correspond to $\mathbf{v}\in\mathbb{R}$ and $\mathbf{w}\in\mathbb{R}$,
respectively. 
\begin{figure}[t]
\centering

\subfloat[$\hat{\mu}_{xy}(\mathbf{v}, \mathbf{w})$]{
\includegraphics[width=0.48\linewidth]{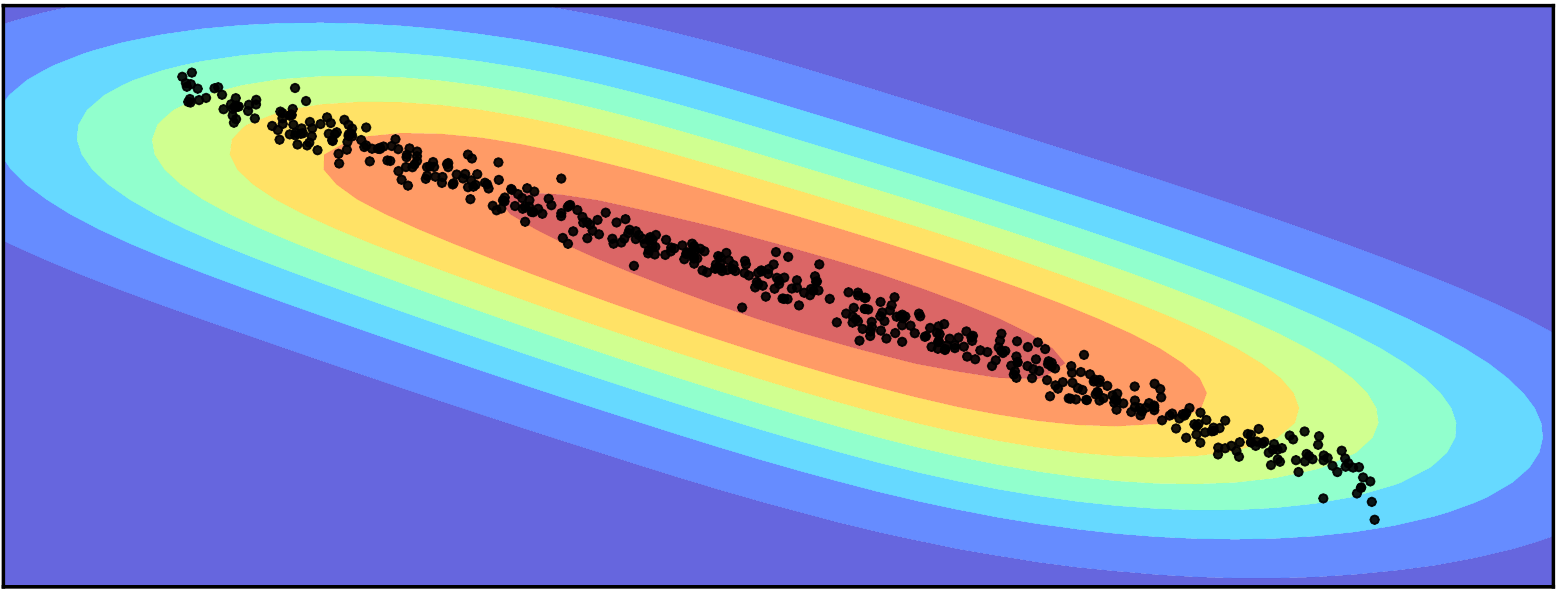}
}
\subfloat[$\widehat{\mu_x \mu_y}(\mathbf{v}, \mathbf{w})$]{ 
\includegraphics[width=0.48\linewidth]{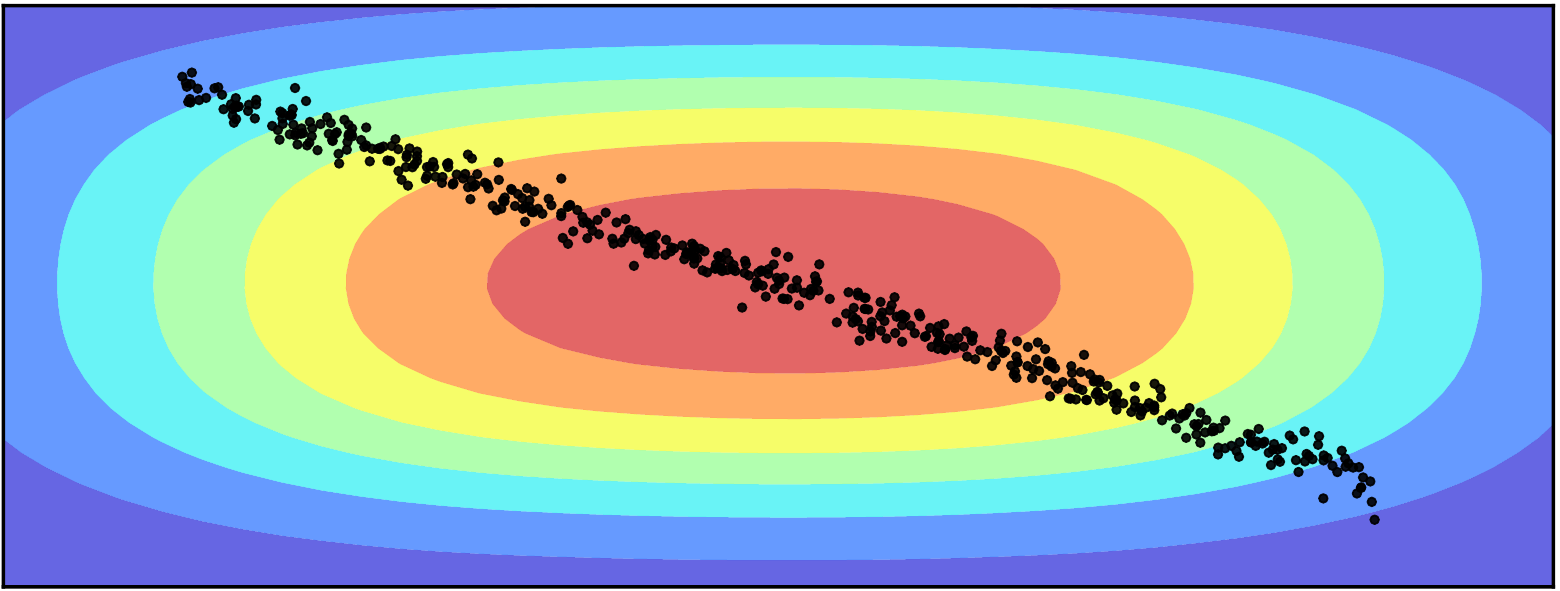} 
}
\\ \vspace{-2mm}
\subfloat[$ \widehat{\mathbf{\Sigma}}(\mathbf{v}, \mathbf{w})$ \label{fig:sigma_hat}]{ 
\includegraphics[width=0.48\linewidth]{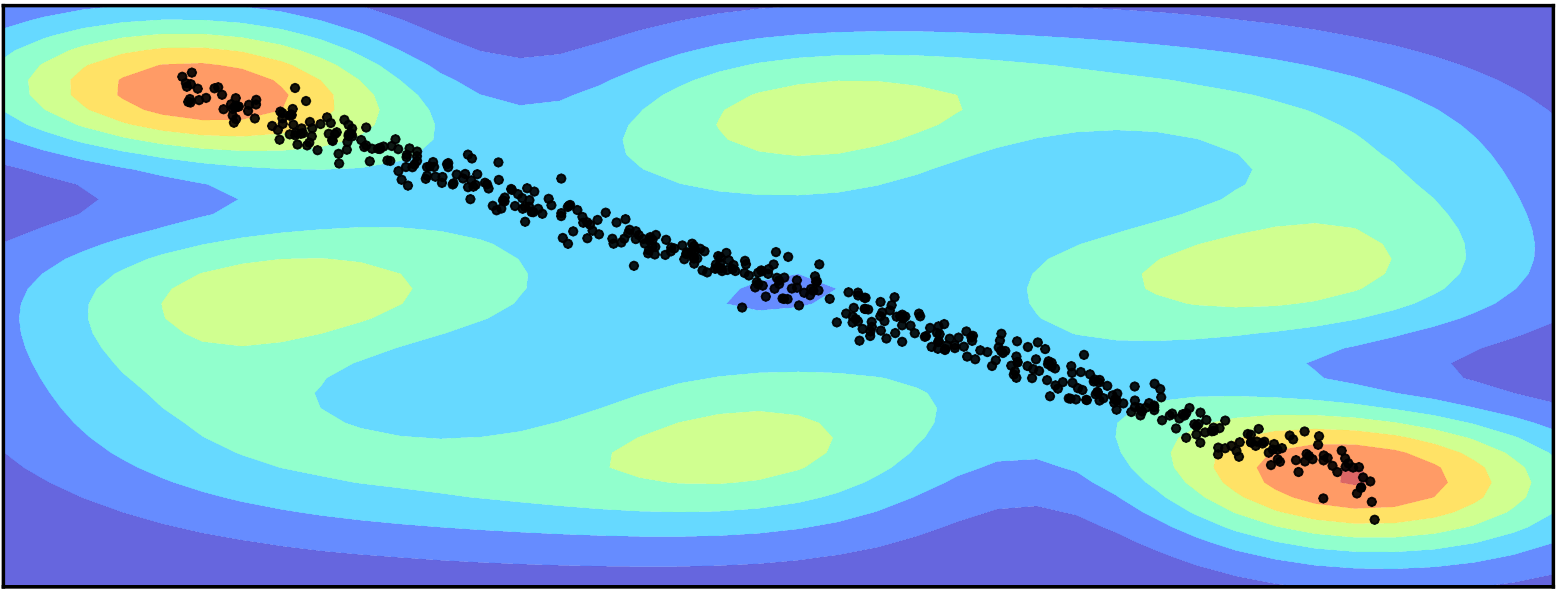} 
}
\subfloat[Statistic $\hat{\lambda}_n(\mathbf{v}, \mathbf{w}) \label{fig:nfsic_surface} $ ]{ 
\includegraphics[width=0.48\linewidth]{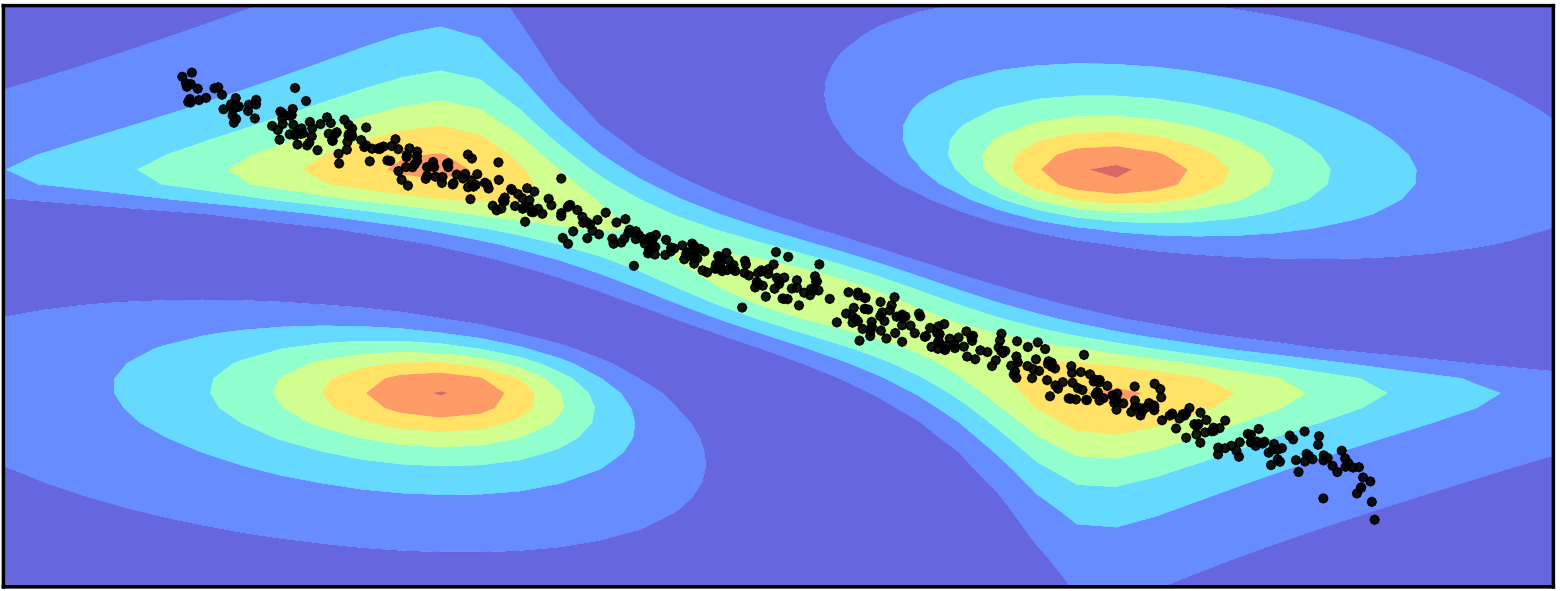} 
}

\caption{Illustration of $\widehat{\mathrm{NFSIC^2}}$. }
\label{fig:illus_nfsic}
\vspace{-4mm}
\end{figure}

Intuitively, $\hat{u}(\mathbf{v},\mathbf{w})=\hat{\mu}_{xy}(\mathbf{v},\mathbf{w})-\widehat{\mu_{x}\mu_{y}}(\mathbf{v},\mathbf{w})$
captures the difference of the joint distribution and the product
of the marginals as a function of $(\mathbf{v},\mathbf{w})$. Squaring
$\hat{u}(\mathbf{v},\mathbf{w})$ and dividing it by the variance
shown in Figure \ref{fig:sigma_hat} gives the statistic (also the
parameter tuning objective) shown in Figure \ref{fig:nfsic_surface}.
The latter figure suggests that the parameter tuning objective function
can be non-convex. However, we note that the non-convexity arises
since there are multiple ways to detect the difference between the
joint distribution and the product of the marginals. In this case,
the lower left and upper right regions equally indicate the largest
difference. 

\section{Experiments\label{sec:experiments}}

\begin{figure*}[htb]
\centering
\subfloat[SG $(\alpha = 0.05)$ \label{fig:sg_param_runtime} ]{
\includegraphics[width=0.21\textwidth]{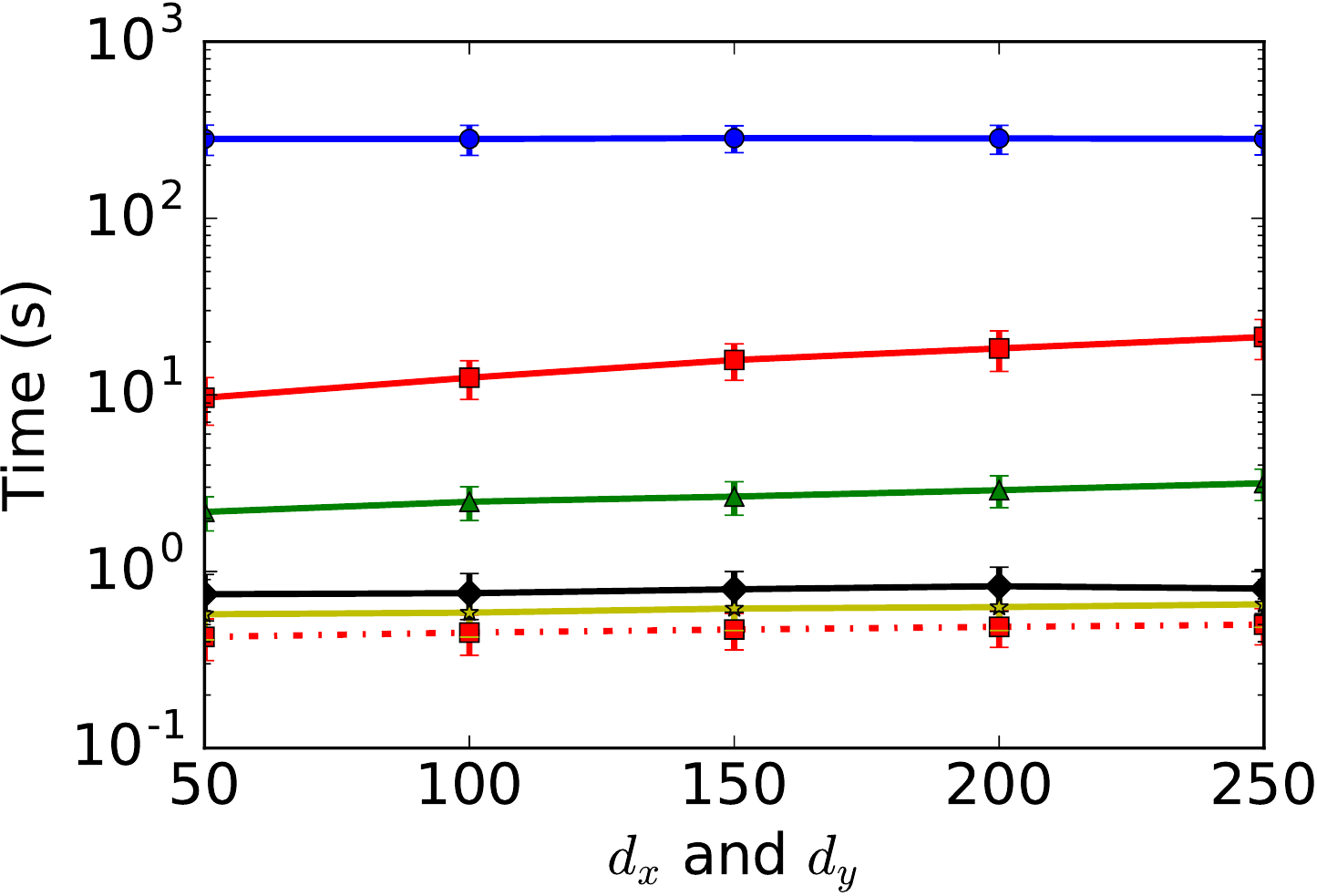}
}
\subfloat[SG $(\alpha = 0.05)$ \label{fig:sg_param_pow} ]{
\includegraphics[width=0.21\textwidth]{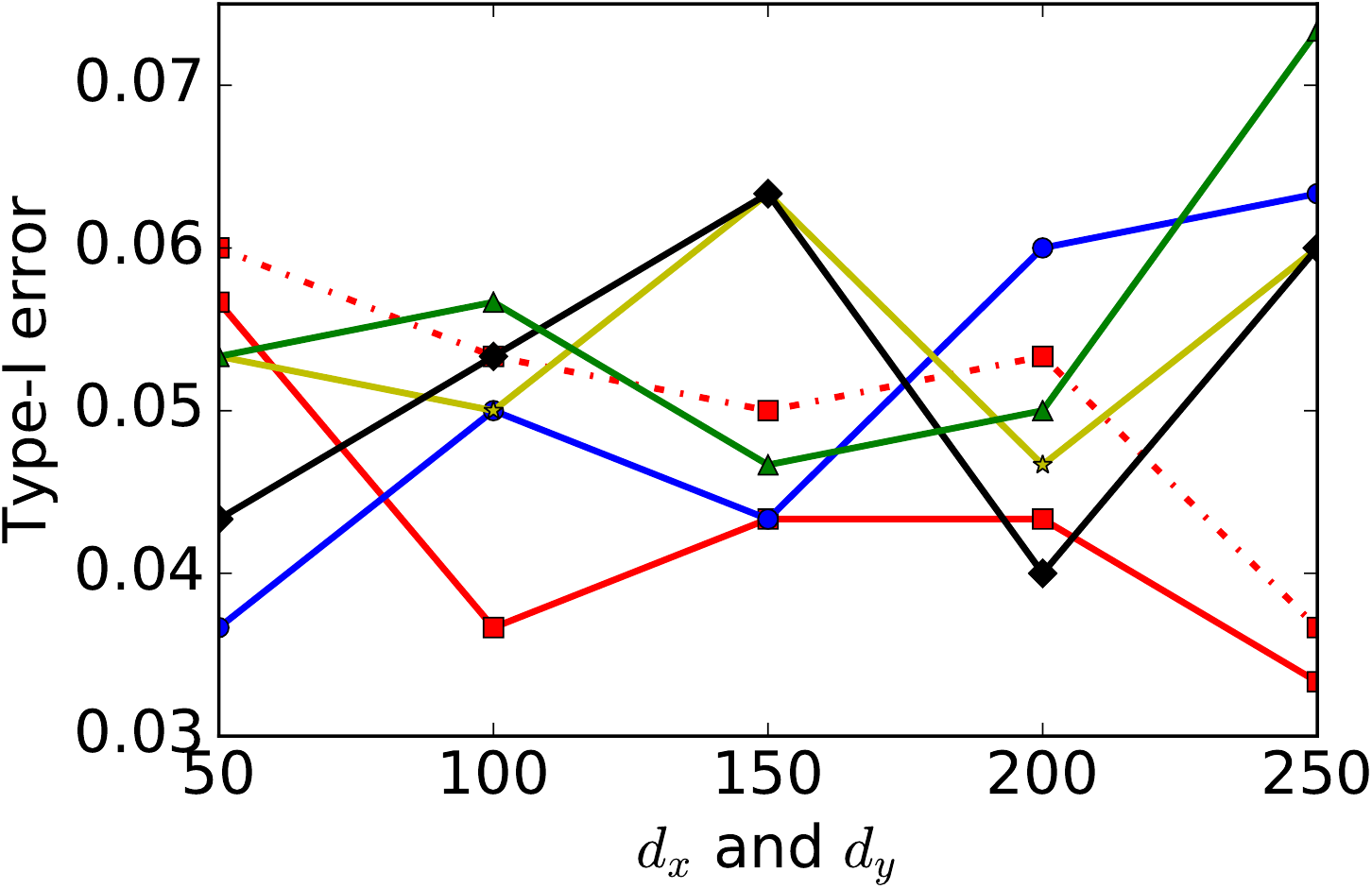}
}
\subfloat[Sin \label{fig:sin_param_pow}]{
\includegraphics[width=0.20\textwidth]{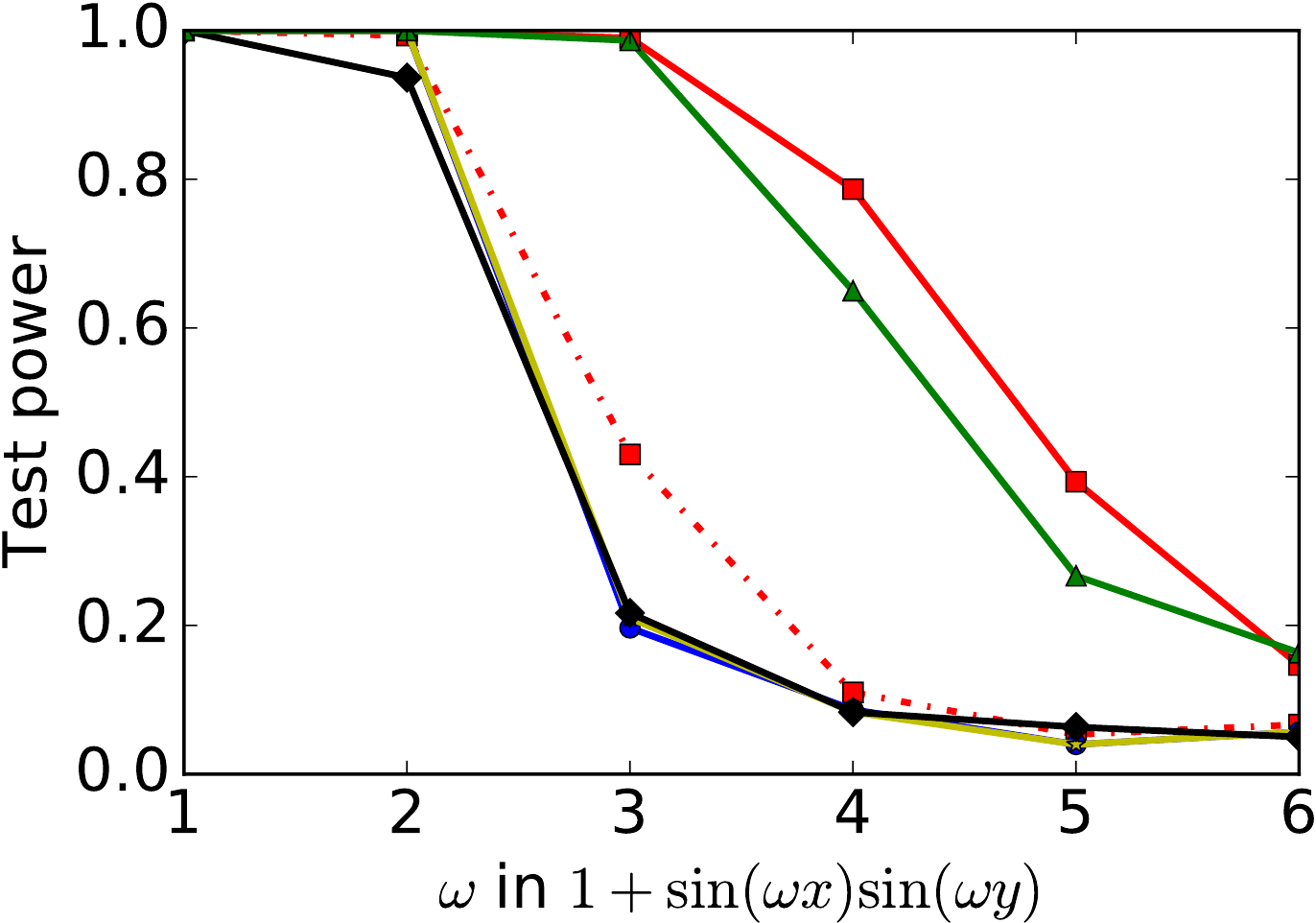}
}
\subfloat[GSign \label{fig:gsign_param_pow}]{
\includegraphics[width=0.31\textwidth]{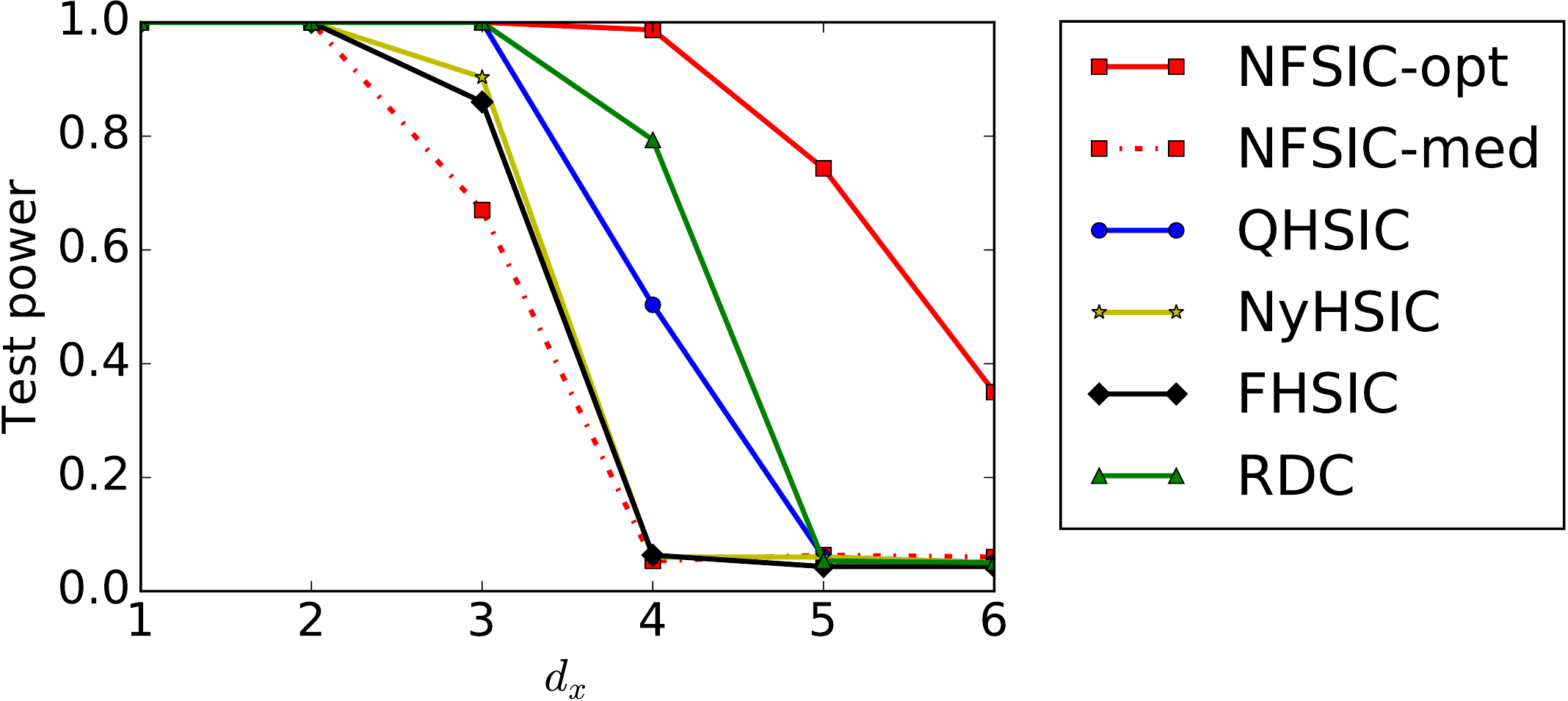}
}
\caption{(a): Runtime. (b): Probability of rejecting $H_0$ as problem parameters vary. Fix $n=4000$.}
\label{fig:toy_pow_vs_params}
\end{figure*}

In this section, we empirically study the performance of the proposed
method on both toy (Section \ref{sec:toy_problems}) and real-life
problems (Section \ref{sec:real_problems}). Our interest is in the
performance of linear-time tests on challenging problems which require
a large sample size to be able to accurately reveal the dependence.
All the code is available at \url{https://github.com/wittawatj/fsic-test}.

We compare the proposed NFSIC with optimization (NFSIC-opt) to five
multivariate nonparametric tests. The $\widehat{\mathrm{NFSIC^{2}}}$
test without optimization (NFSIC-med) acts as a baseline, allowing
the effect of parameter optimization to be clearly seen. For pedagogical
reason, we consider the original HSIC test of \citet{Gretton2005}
denoted by QHSIC, which is a quadratic-time test. Nyström HSIC (NyHSIC)
uses a Nyström approximation to the kernel matrices of $X$ and $Y$
when computing the HSIC statistic. FHSIC is another variant of HSIC
in which a random Fourier feature approximation \citep{Rahimi2008}
to the kernel is used. NyHSIC and FHSIC are studied in \citet{Zhang2016}
and can be computed in $\mathcal{O}(n)$, with quadratic dependency
on the number of inducing points in NyHSIC, and quadratic dependency
in the number of random features in FHSIC. Finally, the Randomized
Dependence Coefficient (RDC) proposed in \citet{Lopez-Paz2013} is
also considered. The RDC can be seen as the primal form (with random
Fourier features) of the kernel canonical correlation analysis of
\citet{BacJor02} on copula-transformed data. We consider RDC as a
linear-time test even though preprocessing by an empirical copula
transform costs $\mathcal{O}((d_{x}+d_{y})n\log n)$.

We use Gaussian kernel classes $\mathcal{K}_{g}$ and $\mathcal{L}_{g}$
for both $X$ and $Y$ in all the methods. Except NFSIC-opt, all other
tests use full sample to conduct the independence test, where the
Gaussian widths $\sigma_{x}$ and $\sigma_{y}$ are set according
to the widely used median heuristic i.e., $\sigma_{x}=\mathrm{median}\left(\left\{ \|\mathbf{x}_{i}-\mathbf{x}_{j}\|_{2}\mid1\le i<j\le n\right\} \right)$,
and $\sigma_{y}$ is set in the same way using $\{\mathbf{y}_{i}\}_{i=1}^{n}$.
The $J$ locations for NFSIC-med are randomly drawn from the standard
multivariate normal distribution in each trial. For a sample of size
$n$, NFSIC-opt uses half the sample for parameter tuning, and the
other disjoint half for the test. We permute the sample 300 times
in RDC\footnote{We use a permutation test for RDC, following the authors' implementation
(\url{https://github.com/lopezpaz/randomized_dependence_coefficient},
referred commit: b0ac6c0).} and HSIC to simulate from the null distribution and compute the test
threshold. The null distributions for FHSIC and NyHSIC are given by
a finite sum of weighted $\chi^{2}(1)$ random variables given in
Eq.\ 8 of \citet{Zhang2016}. Unless stated otherwise, we set the
test threshold of the two NFSIC tests to be the $(1-\alpha)$-quantile
of $\chi^{2}(J)$. To provide a fair comparison, we set $J=10$, use
10 inducing points in NyHSIC, and 10 random Fourier features in FHSIC
and RDC.

\textbf{Optimization of NFSIC-opt} The parameters of NFSIC-opt are
$\sigma_{x},\sigma_{y},$ and $J$ locations of size $(d_{x}+d_{y})J$.
We treat all the parameters as a long vector in $\mathbb{R}^{2+(d_{x}+d_{y})J}$
and use gradient ascent to optimize $\hat{\lambda}_{n/2}$. We observe
that initializing $V_{J}$ by randomly picking $J$ points from the
training sample yields good performance. The regularization parameter
$\gamma_{n}$ in NFSIC is fixed to a small value, and is not optimized.
It is worth emphasizing that the complexity of the optimization procedure
is still linear in $n$.

Since FSIC, NyHFSIC and RDC rely on a finite-dimensional kernel approximation,
these tests are consistent only if both the number of features increases
with $n$. By constrast, the proposed NFSIC requires only $n$ to
go to infinity to achieve consistency i.e., $J$ can be fixed. We
refer the reader to Appendix \ref{sec:pow_vs_J} for a brief investigation
of the test power vs. increasing $J$. The test power does not necessarily
monotonically increase with $J$.

\subsection{Toy Problems\label{sec:toy_problems}}

\begin{figure*}
\centering
\subfloat[SG. $d_x = d_y = 250$. \label{fig:sg_n_runtime}]{
\includegraphics[width=0.21\textwidth]{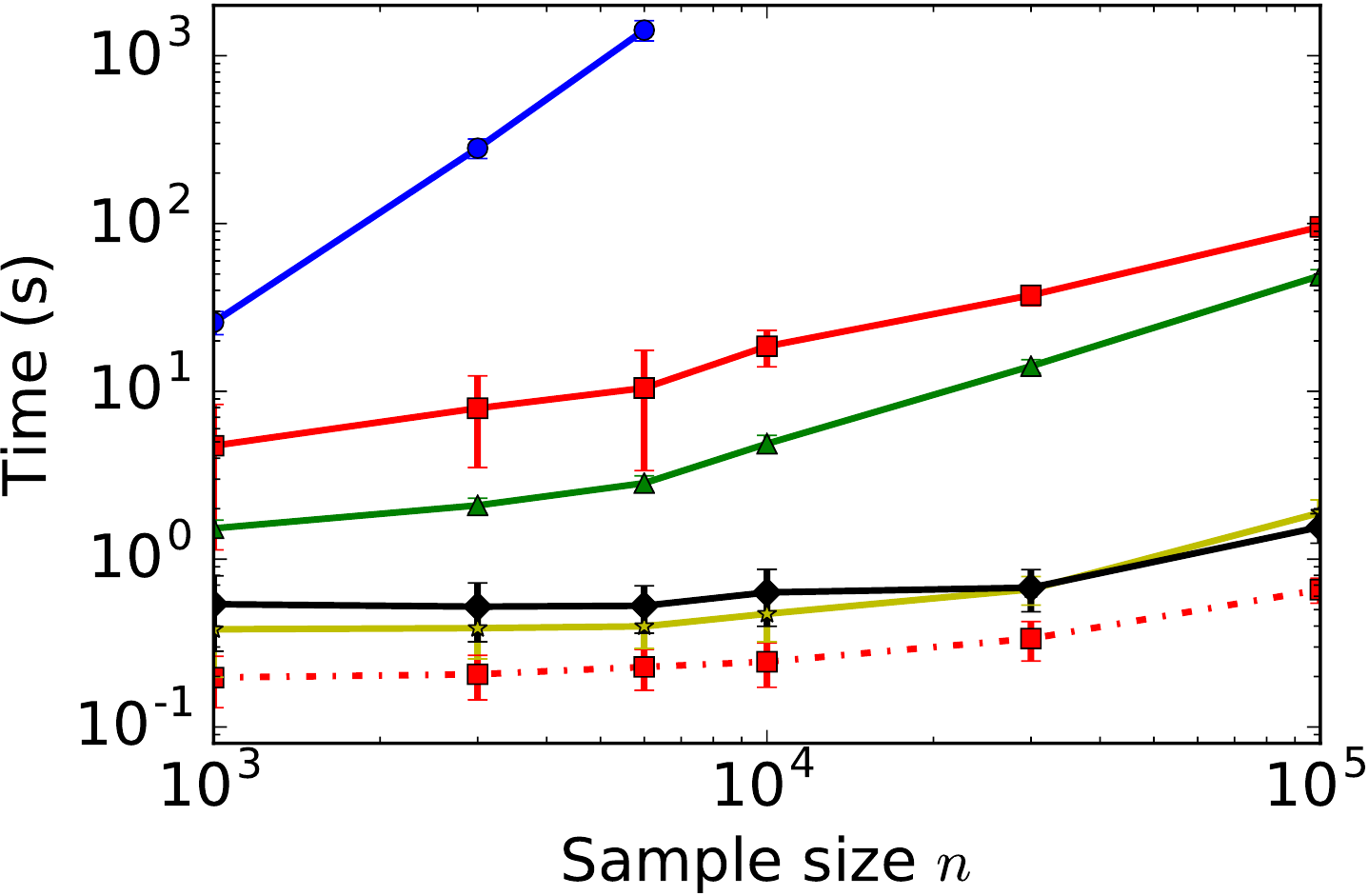}
}
\subfloat[SG. $d_x = d_y = 250$. \label{fig:sg_n_pow}]{
\includegraphics[width=0.21\textwidth]{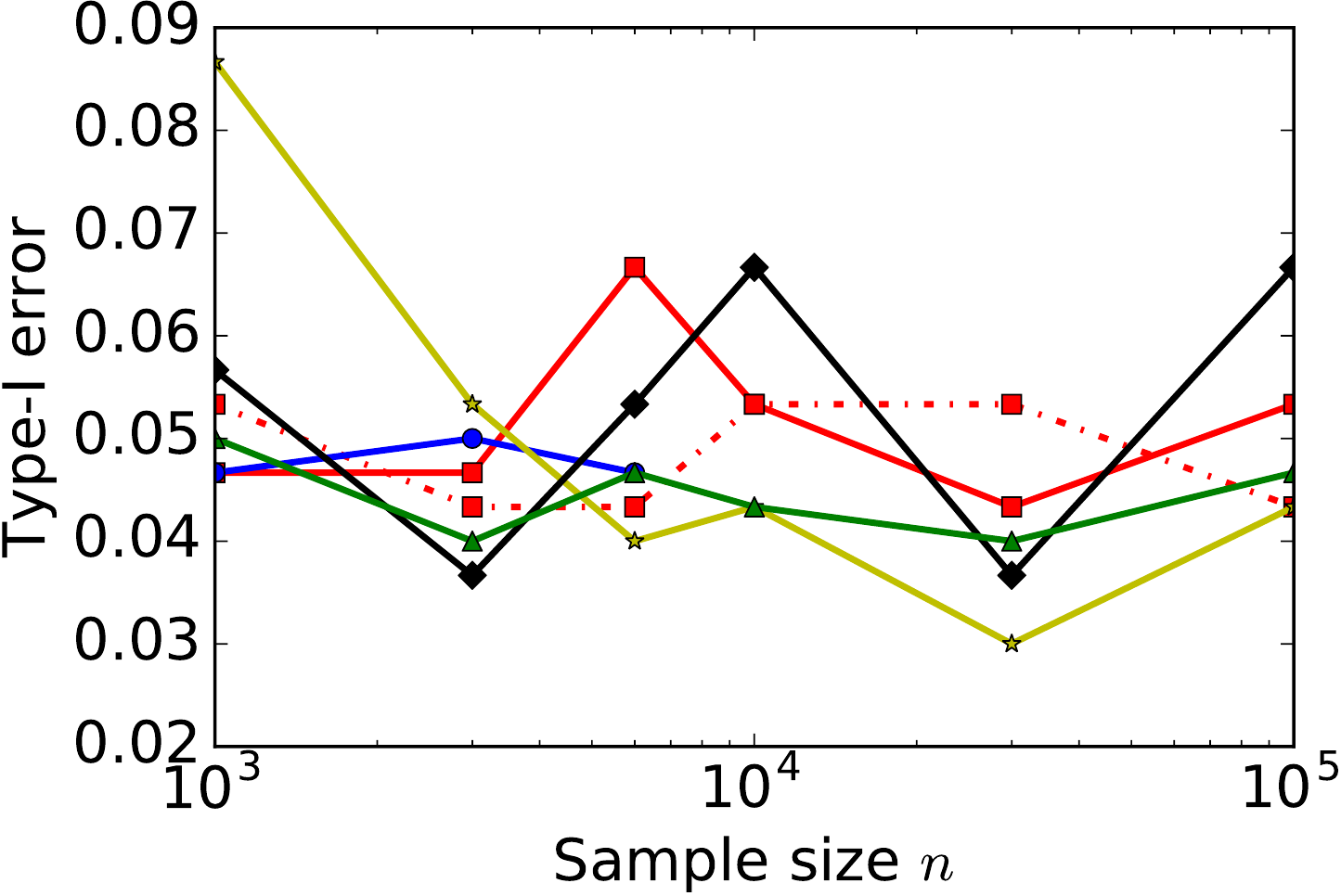}
}
\subfloat[Sin. $\omega=4$. \label{fig:sin_n_pow}]{
\includegraphics[width=0.21\textwidth]{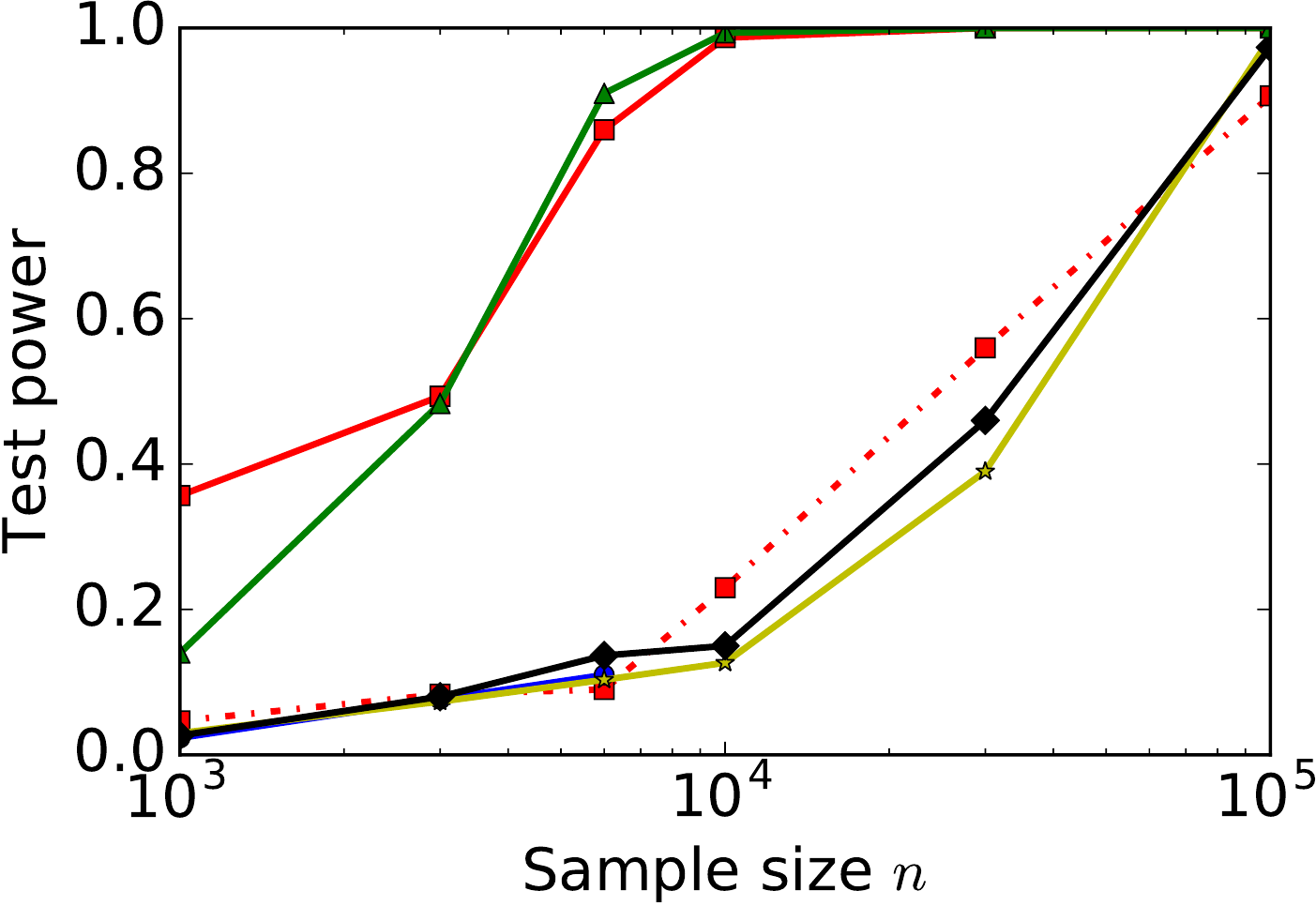}
}
\subfloat[GSign. $d_x=4$.\label{fig:gsign_n_pow}]{
\includegraphics[width=0.31\textwidth]{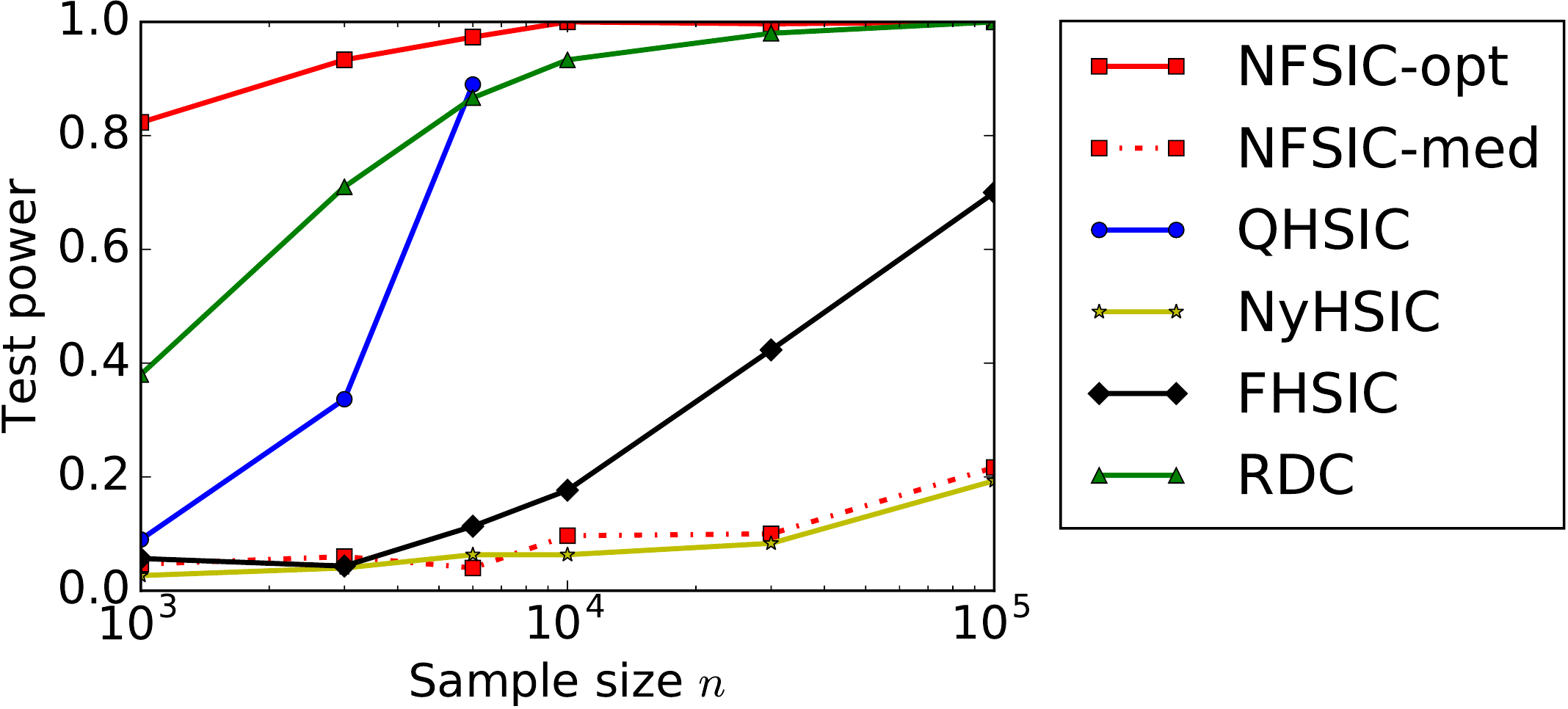}
}

\caption{(a) Runtime. (b): Probability of rejecting $H_{0}$ as $n$ increases
in the toy problems.\label{fig:toy_pow_vs_n}}
\end{figure*}

We consider three toy problems: Same Gaussian (SG), Sinusoid (Sin),
and Gaussian Sign (GSign). 

\textbf{1. Same Gaussian (SG).} The two variables are independently
drawn from the standard multivariate normal distribution i.e., $X\sim\mathcal{N}(\mathbf{0},\mathbf{I}_{d_{x}})$
and $Y\sim\mathcal{N}(\mathbf{0},\mathbf{I}_{d_{y}})$ where $\mathbf{I}_{d}$
is the $d\times d$ identity matrix. This problem represents a case
in which $H_{0}$ holds.

\textbf{2. Sinusoid (Sin).} Let $p_{xy}$ be the probability density
of $P_{xy}$. In the Sinusoid problem, the dependency of $X$ and
$Y$ is characterized by $(X,Y)\sim p_{xy}(x,y)\propto1+\sin(\omega x)\sin(\omega y),$
where the domains of $\mathcal{X},\mathcal{Y}=(-\pi,\pi)$ and $\omega$
is the frequency of the sinusoid. As the frequency $\omega$ increases,
the drawn sample becomes more similar to a sample drawn from $\mathrm{Uniform}((-\pi,\pi)^{2})$.
That is, the higher $\omega$, the harder to detect the dependency
between $X$ and $Y$. This problem was studied in \citet{Sejdinovic2013}.
Plots of the density for a few values of $\omega$ are shown in Figures
\ref{fig:redundant_locs-1} and \ref{fig:pow_vs_J} in the appendix.
The main characteristic of interest in this problem is the local change
in the density function. 

\textbf{3. Gaussian Sign (GSign).} In this problem, $Y=|Z|\prod_{i=1}^{d_{x}}\mathrm{sgn}(X_{i}),$
where $X\sim\mathcal{N}(\mathbf{0},\mathbf{I}_{d_{x}})$, $\mathrm{sgn}(\cdot)$
is the sign function, and $Z\sim\mathcal{N}(0,1)$ serves as a source
of noise. The full interaction of $X=(X_{1},\ldots,X_{d_{x}})$ is
what makes the problem challenging. That is, $Y$ is dependent on
$X$, yet it is independent of any proper subset of $\{X_{1},\ldots,X_{d}\}$.
Thus, simultaneous consideration of all the coordinates of $X$ is
required to successfully detect the dependency.

We fix $n=4000$ and vary the problem parameters. Each problem is
repeated for 300 trials, and the sample is redrawn each time. The
significance level $\alpha$ is set to 0.05. The results are shown
in Figure \ref{fig:toy_pow_vs_params}. It can be seen that in the
SG problem (Figure \ref{fig:sg_param_pow}) where $H_{0}$ holds,
all the tests achieve roughly correct type-I errors at $\alpha=0.05$.
In the Sin problem, NFSIC-opt achieves the highest test power for
all considered $\omega=1,\ldots,6$, highlighting its strength in
detecting local changes in the joint density. The performance of NFSIC-med
is significantly lower than that of NFSIC-opt. This phenomenon clearly
emphasizes the importance of the optimization to place the locations
at the relevant regions in $\mathcal{X}\times\mathcal{Y}$. RDC has
a remarkably high performance in both Sin and GSign (Figure \ref{fig:sin_param_pow}, \ref{fig:gsign_param_pow})
despite no parameter tuning. Interestingly, both NFSIC-opt and RDC
outperform the quadratic-time QHSIC in these two problems. The ability
to simultaneously consider interacting features of NFSIC-opt is indicated
by its superior test power in GSign, especially at the challenging
settings of $d_{x}=5,6$. An average trial runtime for each test in
the SG problem is shown in Figure \ref{fig:sg_param_runtime}. We
observe that the runtime does not increase with dimension, as the
complexity of all the tests is linear in the dimension of the input.
All the tests are implemented in Python using a common \texttt{SciPy}
Stack.

To investigate the sample efficiency of all the tests, we fix $d_{x}=d_{y}=250$
in SG, $\omega=4$ in Sin, $d_{x}=4$ in GSign, and increase $n$.
Figure \ref{fig:toy_pow_vs_n} shows the results. The quadratic dependency
on $n$ in QHSIC makes it infeasible both in terms of memory and runtime
to consider $n$ larger than 6000 (Figure \ref{fig:sg_n_runtime}).
In constrast, although not the most time-efficient, NFSIC-opt has
the highest sample-efficiency for GSign, and for Sin in the low-sample
regime, significantly outperforming QHSIC. Despite the small additional
overhead from the optimization, we are yet able to conduct an accurate
test with $n=10^{5},d_{x}=d_{y}=250$ in less than $100$ seconds.
We observe in Figure \ref{fig:sg_n_pow} that the two NFSIC variants
have correct type-I errors across all sample sizes, indicating that
the asymptotic null distribution approximately holds by the time $n$
reaches 1000. We recall from Theorem \ref{thm:nfsic_good_test} that
the NFSIC test with random test locations will asymptotically reject
$H_{0}$ if it is false. A demonstration of this property is given
in Figure \ref{fig:sin_n_pow}, where the test power of NFSIC-med
eventually reaches 1 with $n$ higher than $10^{5}$.

\subsection{Real Problems\label{sec:real_problems}}

We now examine the performance of our proposed test on real problems. 

\textbf{Million Song Data (MSD)} We consider a subset of the Million
Song Data\footnote{Million Song Data subset: \url{https://archive.ics.uci.edu/ml/datasets/YearPredictionMSD}.}
\citep{Bertin-Mahieux2011}, in which each song $(X)$ out of 515,345
is represented by 90 features, of which 12 features are timbre average
(over all segments) of the song, and 78 features are timbre covariance.
Most of the songs are western commercial tracks from 1922 to 2011.
The goal is to detect the dependency between each song and its year
of release $(Y)$. We set $\alpha=0.01$, and repeat for 300 trials
where the full sample is randomly subsampled to $n$ points in each
trial. Other settings are the same as in the toy problems. To make
sure that the type-I error is correct, we use the permutation approach
in the NFSIC tests to compute the threshold. Figure \ref{fig:msd_n_pow}
shows the test powers as $n$ increases from 500 to 2000. To simulate
the case where $H_{0}$ holds in the problem, we permute the sample
to break the dependency of $X$ and $Y$. The results are shown in
Figure \ref{fig:real_h0_vs_n} in the appendix.

\begin{figure}
\captionsetup[subfigure]{labelformat=empty}  
\centering
\subfloat{
\includegraphics[width=0.5\textwidth]{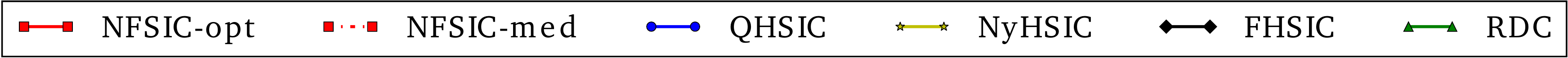}
}
\vspace{-2mm}
\subfloat[(a) MSD problem. \label{fig:msd_n_pow}]{
\includegraphics[width=0.24\textwidth]{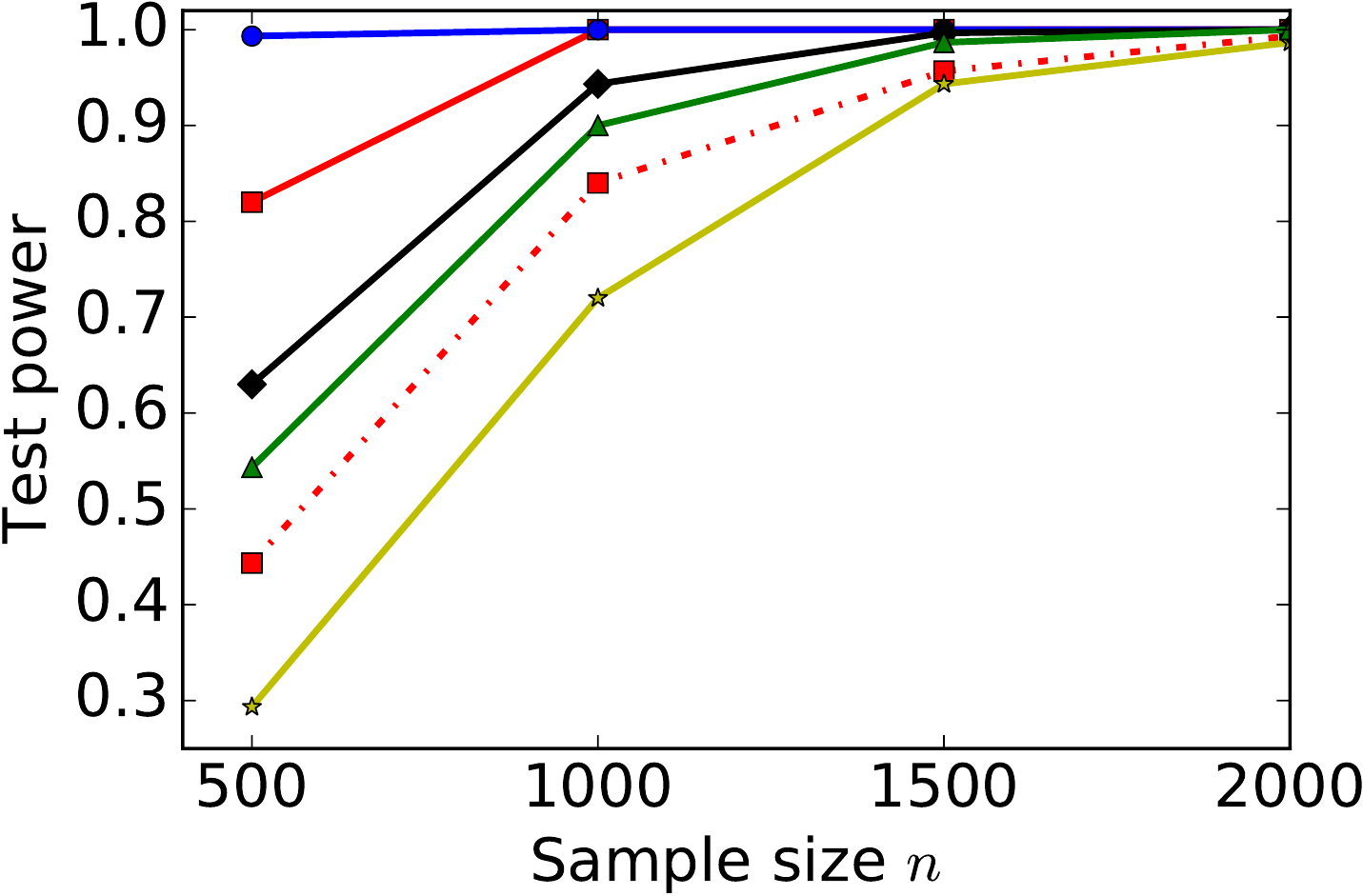}
}
\subfloat[(b) Videos \& Captions problem. \label{fig:vdo_n_pow}]{
\includegraphics[width=0.24\textwidth]{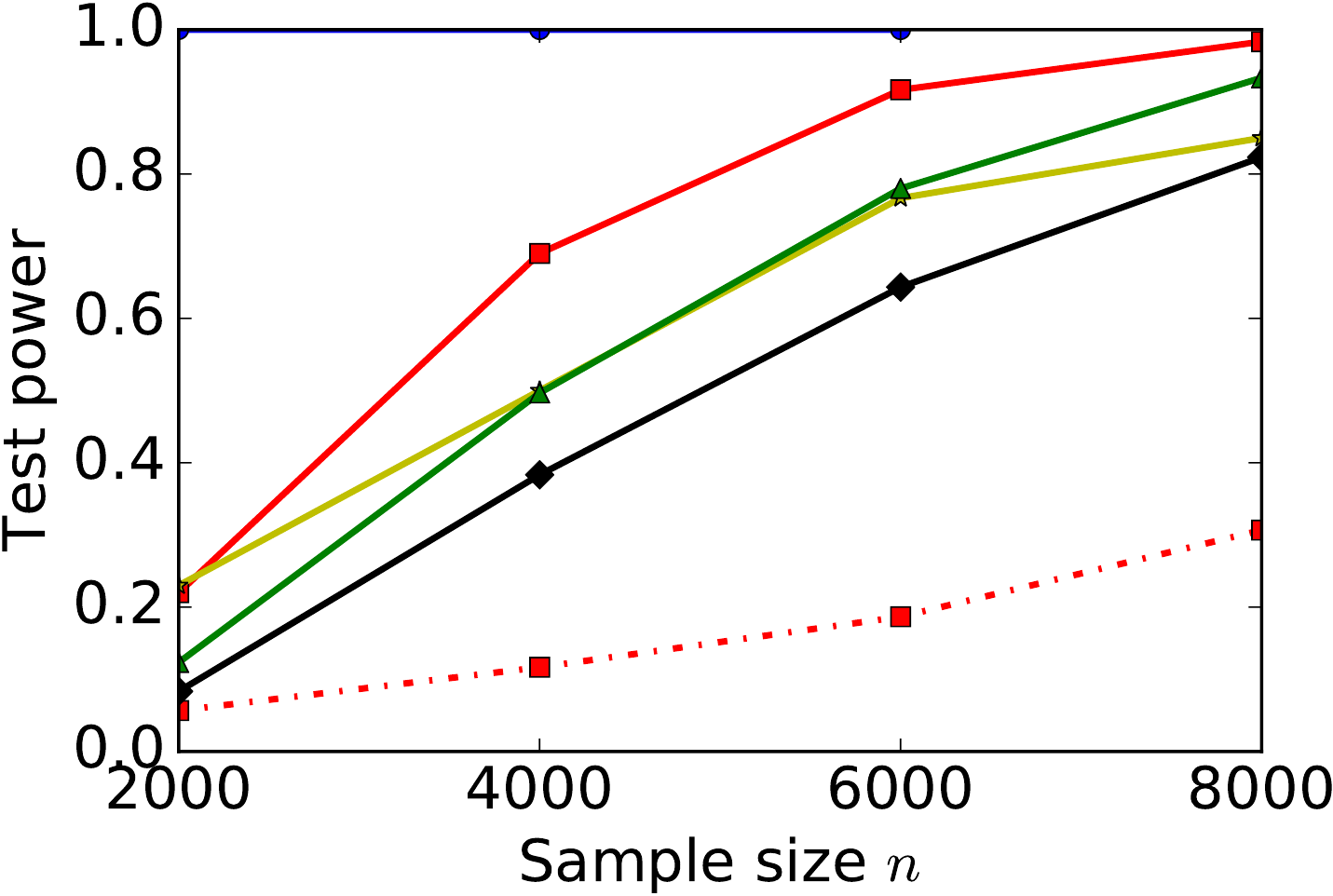}
}
\vspace{-2mm}

\caption{Probability of rejecting $H_{0}$ as $n$ increases in the two real
problems. $\alpha=0.01$.\label{fig:real_vs_n}}
\end{figure}

Evidently, NFSIC-opt has the highest test power among all the linear-time
tests for all the sample sizes. Its test power is second to only QHSIC.
We recall that NFSIC-opt uses half of the sample for parameter tuning.
Thus, at $n=500$, the actual sample for testing is 250, which is
relatively small. The fact that there is a vast power gain from 0.4
(NFSIC-med) to 0.8 (NFSIC-opt) at $n=500$ suggests that the optimization
procedure can perform well even at a lower sample sizes.

\textbf{Videos and Captions} Our last problem is based on the VideoStory46K\texttt{}\footnote{VideoStory46K dataset: \texttt{\url{https://ivi.fnwi.uva.nl/isis/mediamill/datasets/videostory.php}.}}\texttt{
}dataset \citep{Habibian2014}. The dataset contains 45,826 Youtube
videos $(X)$ of an average length of roughly one minute, and their
corresponding text captions $(Y)$ uploaded by the users. Each video
is represented as a $d_{x}=2000$ dimensional Fisher vector encoding
of motion boundary histograms (MBH) descriptors of \citet{Wang2013}.
Each caption is represented as a bag of words with each feature being
the frequency of one word. After filtering only words which occur
in at least six video captions, we obtain $d_{y}=1878$ words. We
examine the test powers as $n$ increases from $2000$ to $8000$.
The results are given in Figure \ref{fig:real_vs_n}. The problem
is sufficiently challenging that all linear-time tests achieve a low
power at $n=2000$. QHSIC performs exceptionally well on this problem,
achieving a maximum power throughout. NFSIC-opt has the highest sample
efficiency among the linear-time tests, showing that the optimization
procedure is also practical in a high dimensional setting.

\subsubsection*{Acknowledgement}

We thank the Gatsby Charitable Foundation for the financial support.
The major part of this work was carried out while Zolt{\'a}n Szab{\'o}
was a research associate at the Gatsby Computational Neuroscience
Unit, University College London.

\bibliographystyle{abbrvnat}
\bibliography{fsic}

\newpage
\onecolumn
\appendix

\begin{center}
{\large \textbf{\ourtitle}}

{\large Supplementary Material}
\end{center}

\section{Type-I Errors\label{sec:type1_err}}

In this section, we show that all the tests have correct type-I errors
(i.e., the probability of reject $H_{0}$ when it is true) in real
problems. We permute the joint sample so that the dependency is broken
to simulate cases in which $H_{0}$ holds. The results are shown in
Figure \ref{fig:real_h0_vs_n}.

\begin{figure}[th]
\captionsetup[subfigure]{labelformat=empty}  
\centering
\subfloat{
\includegraphics[width=0.6\textwidth]{img/legend6-crop.pdf}
}
\vspace{-2mm}
\subfloat[(a) MSD problem (permuted). \label{fig:msd_h0_n_pow}]{
\includegraphics[width=0.30\textwidth]{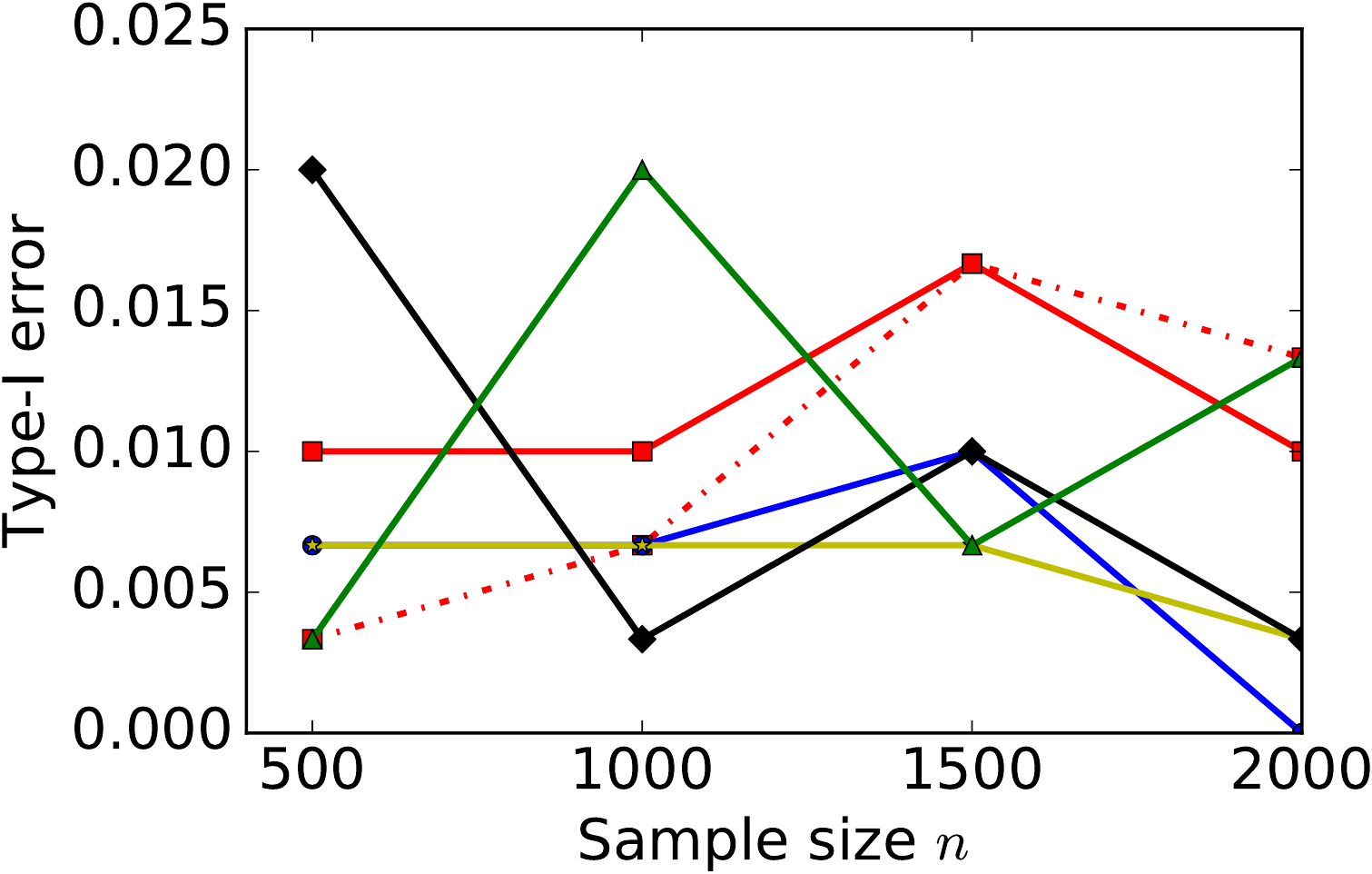}
}
\subfloat[(a) Videos \& Captions problem with shuffled sample. \label{fig:vdo_h0_n_pow}]{
\includegraphics[width=0.30\textwidth]{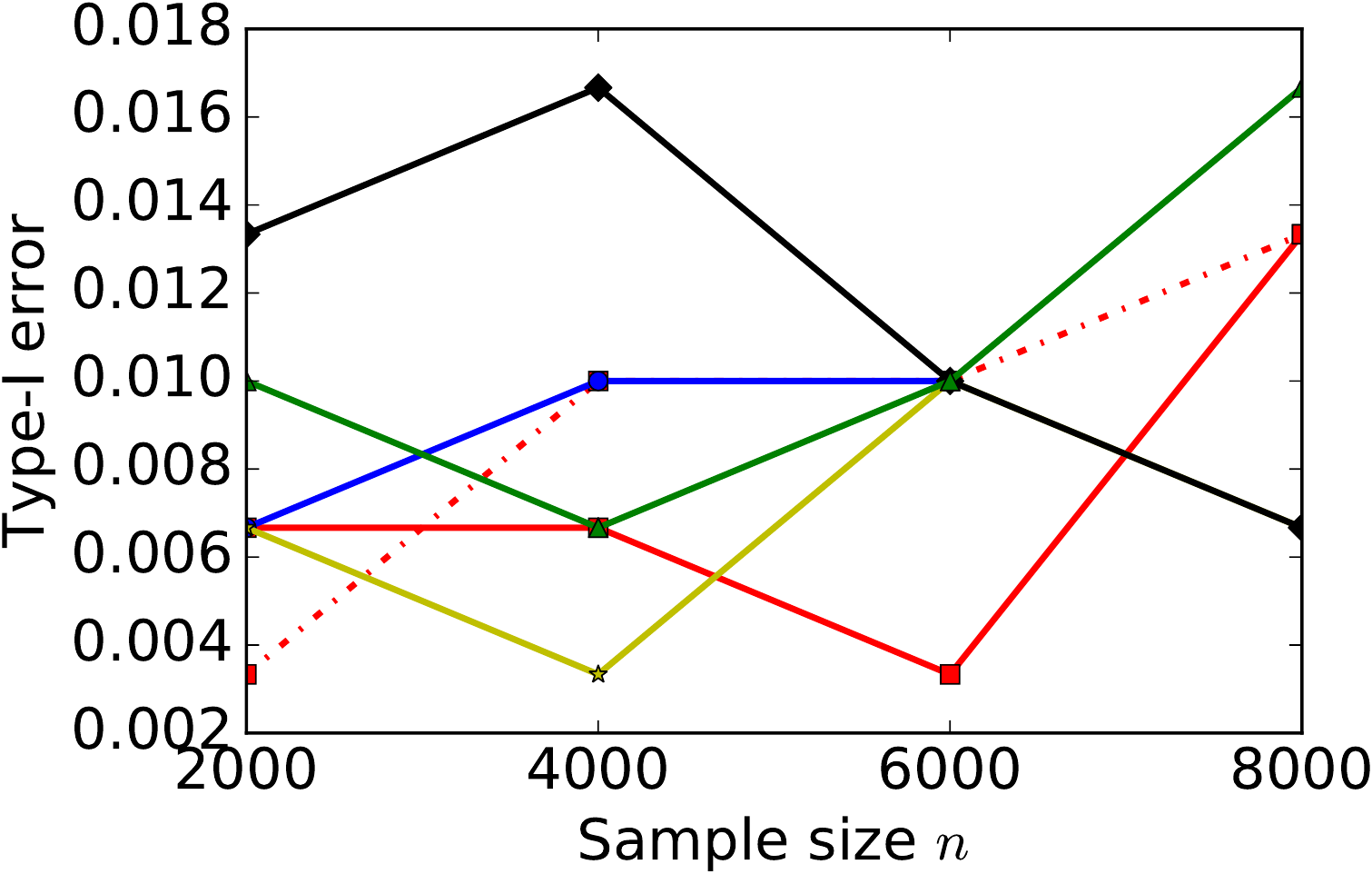}
}

\caption{Probability of rejecting $H_{0}$ as $n$ increases in the Million
Song problem. $\alpha=0.01$.\label{fig:real_h0_vs_n}}
\end{figure}

\section{Redundant Test Locations}

Here, we provide a simple illustration to show that two locations
$\mathbf{t}_{1}=(\mathbf{v}_{1},\mathbf{w}_{1})$ and $\mathbf{t}_{2}=(\mathbf{v}_{2},\mathbf{w}_{2})$
which are too close to each other will reduce the optimization objective.
We consider the Sinusoid problem described in Section \ref{sec:toy_problems}
with $\omega=1$, and use $J=2$ test locations. In Figure \ref{fig:redundant_locs-1},
$\mathbf{t}_{1}$ is fixed at the red star, while $\mathbf{t}_{2}$
is varied along the horizontal line. The objective value $\hat{\lambda}_{n}$
as a function of $(\mathbf{t}_{1},\mathbf{t}_{2})$ is shown in the
bottom figure. It can be seen that $\hat{\lambda}_{n}$ decreases
sharply when $\mathbf{t}_{2}$ is in the neighborhood of $\mathbf{t}_{1}$.
This property implies that two locations which are too close will
not maximize the objective function (i.e., the second feature contains
no additional information when it matches the first). For $J>2$,
the objective sharply decreases if any two locations are in the same
neighborhood.

\begin{figure}[th]
\begin{centering}
\includegraphics[width=0.3\textwidth]{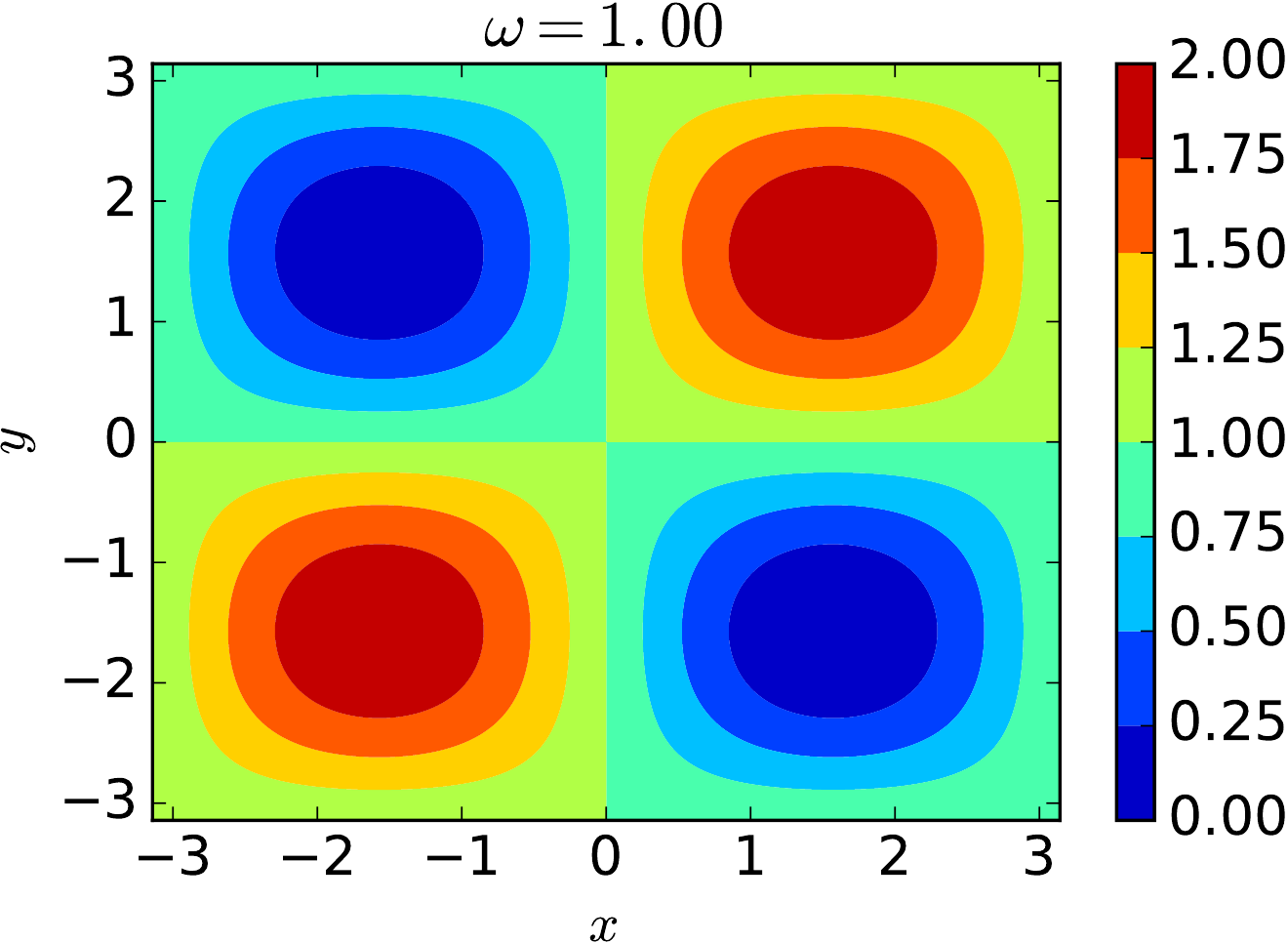}\hspace{1cm}
\includegraphics[width=0.3\textwidth]{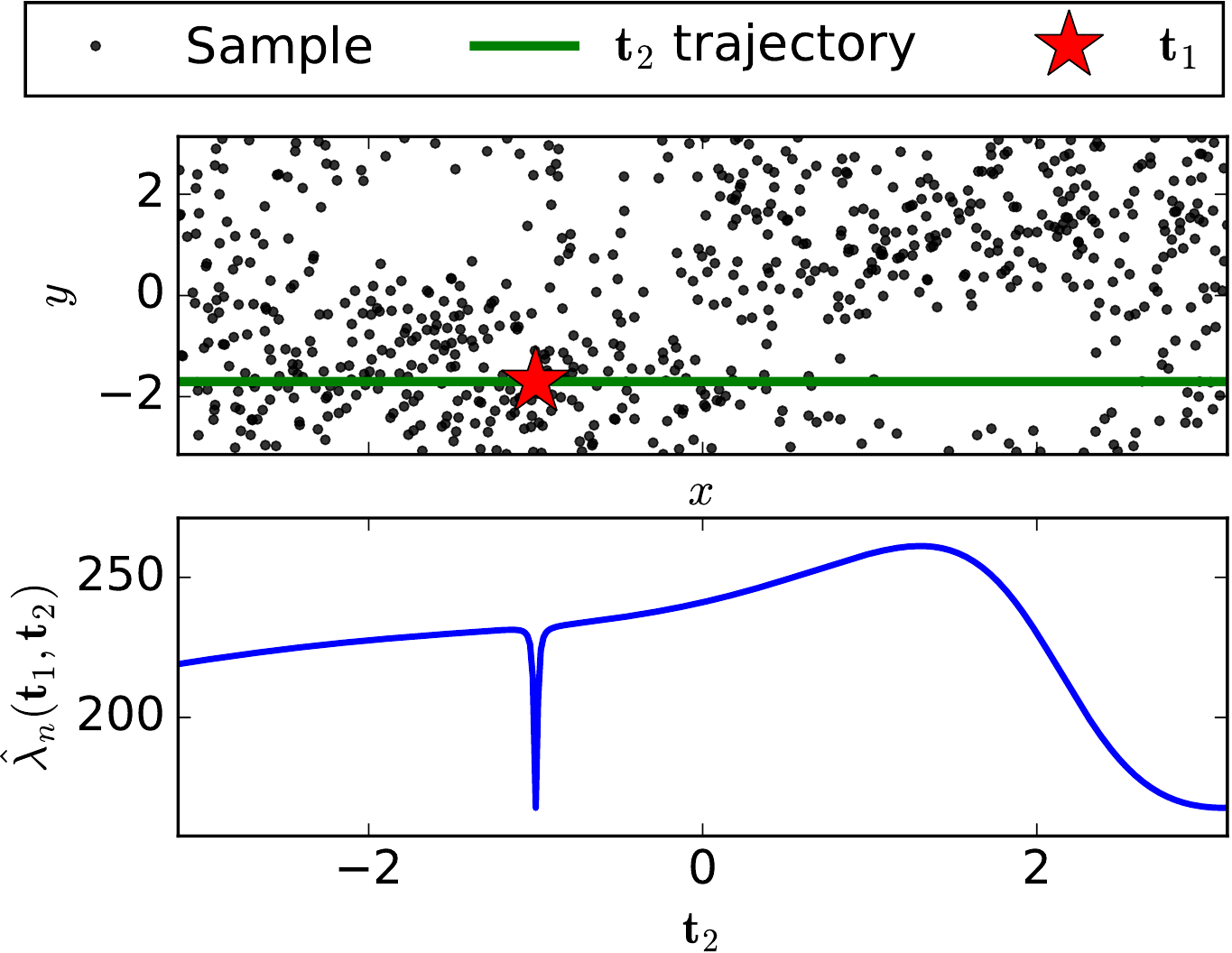} 
\par\end{centering}
\caption{Plot of optimization objective values as location $\mathbf{t}_{2}$
moves along the green line. The objective sharply drops when the two
locations are in the same neighborhood. \label{fig:redundant_locs-1}}
\end{figure}

\section{Test Power vs. $J$\label{sec:pow_vs_J}}

It might seem intuitive that as the number of locations $J$ increases,
the test power should also increase. Here, we empirically show that
this statement is \emph{not }always true. Consider the Sinusoid toy
example described in Section \ref{sec:toy_problems} with $\omega=2$
(also see the left figure of Figure \ref{fig:pow_vs_J}). By construction,
$X$ and $Y$ are dependent in this problem. We run NFSIC test with
a sample size of $n=800$, varying $J$ from $1$ to $600$. For each
value of $J$, the test is repeated for 500 times. In each trial,
the sample is redrawn and the $J$ test locations are drawn from $\mathrm{Uniform}((-\pi,\pi)^{2})$.
There is no optimization of the test locations. We use Gaussian kernels
for both $X$ and $Y$, and use the median heuristic to set the Gaussian
widths to 1.8. Figure \ref{fig:pow_vs_J} shows the test power as
$J$ increases.

\begin{figure}[th]
\begin{centering}
\includegraphics[height=4cm]{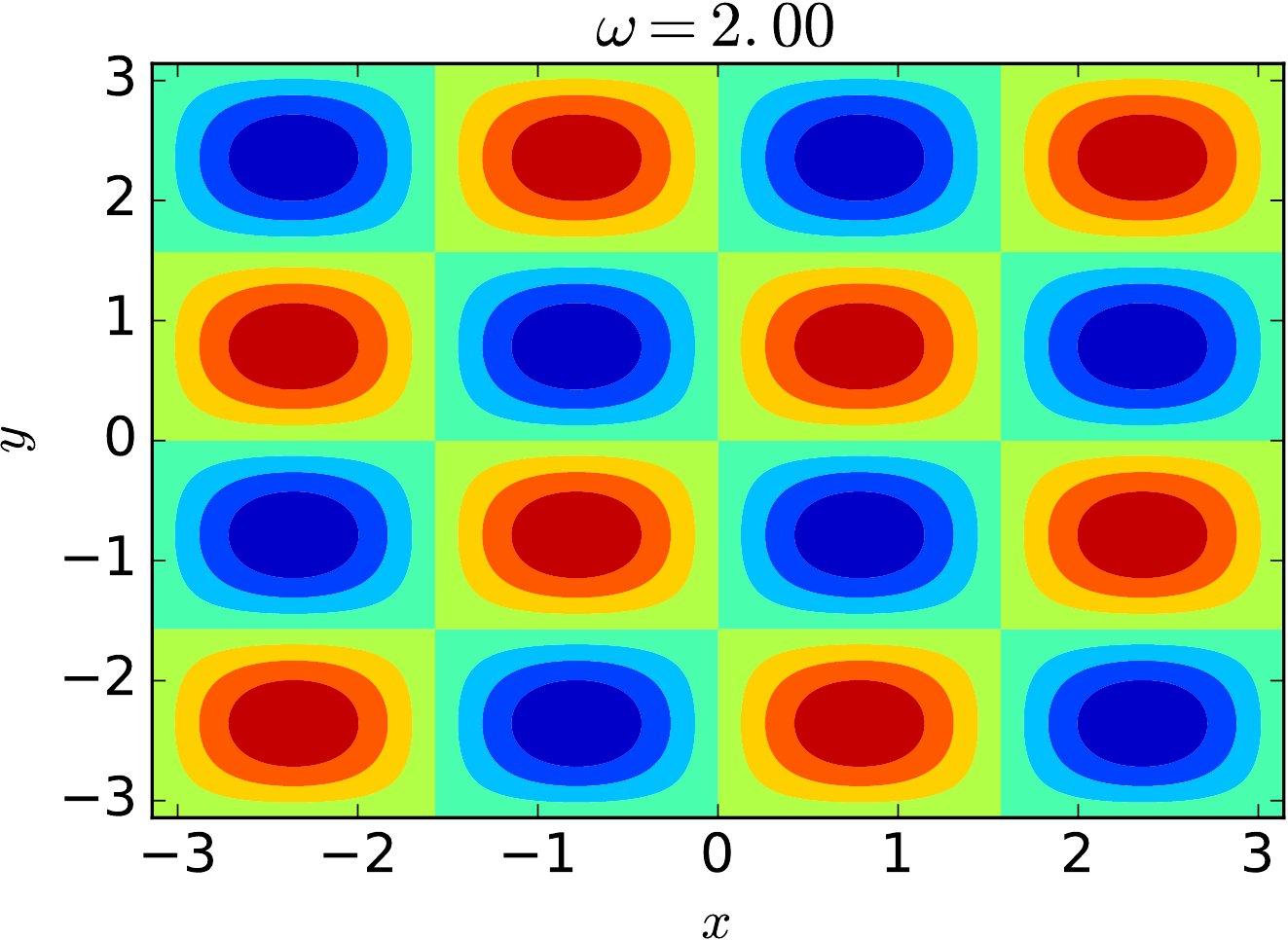} \hspace{1cm}\includegraphics[height=4cm]{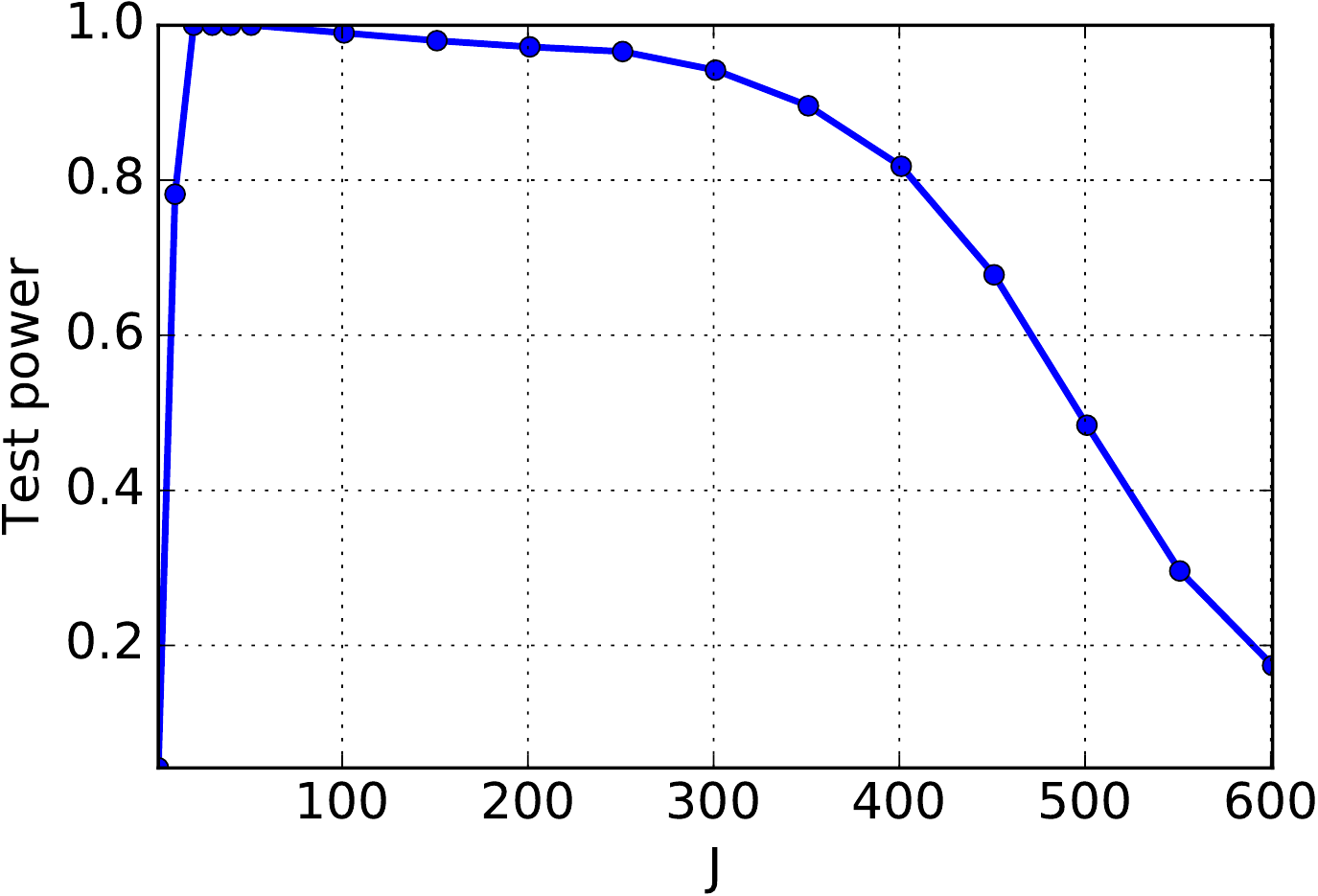}
\par\end{centering}
\caption{The Sinusoid problem and the plot of test power vs. the number of
test locations.\label{fig:pow_vs_J}}
\end{figure}
We observe that the test power does not monotonically increase as
$J$ increases. When $J=1$, the difference of $p_{xy}$ and $p_{x}p_{y}$
cannot be adequately captured, resulting in a low power. The power
increases rapidly to roughly 0.8 at $J=10$, and stays at the maximum
until about $J=100$. Then, the power starts to drop sharply when
$J$ is higher than $400$ in this problem. 

Unlike random Fourier features, the number of test locations in NFSIC
is not the number of Monte Carlo particles used to approximate an
expectation. There is a tradeoff: if the test locations are in key
regions (i.e., regions in which there is a big difference between
$p_{xy}$ and $p_{x}p_{y}$), then they increase power; yet the statistic
gains in variance (thus reducing test power) as $J$ increases. As
can be seen in Figure \ref{fig:pow_vs_J}, there are eight key regions
(in blue) that can reveal the difference of $p_{xy}$ and $p_{x}p_{y}$.
Using an unnecessarily high $J$ not only makes the covariance matrix
$\hat{\mathbf{\Sigma}}$ harder to estimate accurately, it also increases
the computation as the complexity on $J$ is $\mathcal{O}(J^{3})$. 

We note that NFSIC is not intended to be used with a large $J$. In
practice, it should be set to be large enough so as to capture the
key regions as stated. As a practical guide, with optimization of
the test locations, a good starting point is $J=5$ or $10$.

\section{Proof of Proposition \ref{prop:product_gaussian_kers} \label{sec:proof_prod_kgauss_ana}}

Recall Proposition \ref{prop:product_gaussian_kers},
\begin{prop*}[{\small{}A product of Gaussian kernels is characteristic and analytic}]
 \prodgkers{}
\end{prop*}
\begin{proof}
Let $\mathbf{z}:=(\mathbf{x}^{\top},\mathbf{y}^{\top})^{\top}$ and
$\mathbf{z}':=(\mathbf{x}'^{\top},\mathbf{y}'^{\top})^{\top}$ be
vectors in $\mathbb{R}^{d_{x}+d_{y}}$. We prove by reducing the product
kernel to one Gaussian kernel with $g(\mathbf{z},\mathbf{z}')=\exp\left(-(\mathbf{z}-\mathbf{z}')^{\top}\mathbf{C}(\mathbf{z}-\mathbf{z}')\right)$
where $\mathbf{C}:=\left(\begin{array}{cc}
\mathbf{A} & \mathbf{0}\\
\mathbf{0} & \mathbf{B}
\end{array}\right)$. Write $g(\mathbf{z},\mathbf{z}')=\Psi(\mathbf{z}-\mathbf{z}')$
where $\Psi(\mathbf{t}):=\exp\left(-\mathbf{t}^{\top}\mathbf{C}\mathbf{t}\right)$.
Since $\mathbf{C}$ is positive definite, we see that the finite measure
$\zeta$ corresponding to $\Psi$ as defined in Lemma \ref{lem:bochner}
has support everywhere in $\mathbb{R}^{d_{x}+d_{y}}$. Thus, \citet[Theorem 9]{Sriperumbudur2010}
implies that $g$ is characteristic. 

To see that $g$ is analytic, we observe that for each $\mathbf{z}'\in\mathbb{R}^{d_{x}+d_{y}}$,
$\mathbf{z}\mapsto-(\mathbf{z}-\mathbf{z}')^{\top}\mathbf{C}(\mathbf{z}-\mathbf{z}')$
is a multivariate polynomial in $\mathbf{z}$, which is known to be
analytic. Using the fact that $t\mapsto\exp(t)$ is analytic on $\mathbb{R}$,
and that a composition of analytic functions is analytic, we see that
$\mathbf{z}\mapsto\exp\left(-(\mathbf{z}-\mathbf{z}')^{\top}\mathbf{C}(\mathbf{z}-\mathbf{z}')\right)$
is analytic on $\mathbb{R}^{d_{x}+d_{y}}$ for each $\mathbf{z}'$. 
\end{proof}

\section{Proof of Theorem \ref{thm:nfsic_good_test}\label{sec:proof_nfsic_consistent}}

Recall Theorem \ref{thm:nfsic_good_test}, 

\nfsicgoodtest*
\begin{proof}
Assume that $H_{0}$ holds. The consistency of $\hat{\mathbf{\Sigma}}$
and the continuous mapping theorem imply that $\left(\hat{\boldsymbol{\Sigma}}+\gamma_{n}\mathbf{I}\right)^{-1}\stackrel{p}{\to}\boldsymbol{\Sigma}^{-1}$
which is a constant. Let $\mathbf{a}$ be a random vector in $\mathbb{R}^{J}$
following $\mathcal{N}(\mathbf{0},\boldsymbol{\Sigma})$. By \citet[Theorem 2.7 (v)]{Vaart2000},
it follows that $\left[\sqrt{n}\hat{\mathbf{u}},\left(\hat{\boldsymbol{\Sigma}}+\gamma_{n}\mathbf{I}\right)^{-1}\right]\stackrel{d}{\to}\left[\mathbf{a},\boldsymbol{\Sigma}^{-1}\right]$
where $\mathbf{u}=0$ almost surely by Proposition \ref{prop:fsic_dependence_measure},
and $\sqrt{n}\hat{\mathbf{u}}\stackrel{d}{\to}\mathcal{N}(\mathbf{0},\boldsymbol{\Sigma})$
by Proposition \ref{prop:asymp_u}. Since $f(\mathbf{x},\mathbf{S}):=\mathbf{x}^{\top}\mathbf{S}\mathbf{x}$
is continuous, $f\left(\sqrt{n}\hat{\mathbf{u}},\left(\hat{\boldsymbol{\Sigma}}+\gamma_{n}\mathbf{I}\right)^{-1}\right)\stackrel{d}{\to}f(\mathbf{a},\boldsymbol{\Sigma}^{-1})$.
Equivalently, $n\hat{\mathbf{u}}^{\top}\left(\hat{\boldsymbol{\Sigma}}+\gamma_{n}\mathbf{I}\right)^{-1}\hat{\mathbf{u}}\stackrel{d}{\to}\mathbf{a}^{\top}\boldsymbol{\Sigma}^{-1}\mathbf{a}\sim\chi^{2}(J)$
by \citet[Theorem 3.3.3]{Anderson2003}. This proves the first claim.

The proof of the second claim has a very similar structure to the
proof of Proposition 2 of \citet{Chwialkowski2015}. Assume that $H_{1}$
holds. Then, $\mathbf{u}\neq\mathbf{0}$ almost surely by Proposition
\ref{prop:fsic_dependence_measure}. Since $k$ and $l$ are bounded,
it follows that $|h_{\mathbf{t}}(\mathbf{z},\mathbf{z}')|\le2B_{k}B_{l}$
for any $\mathbf{z},\mathbf{z}'$ (see (\ref{eq:ustat_core_bound})),
and we have that $\hat{\mathbf{u}}\stackrel{a.s.}{\to}\mathbf{u}$
by \citet[Section 5.4, Theorem A]{Serfling2009}. Thus, $\hat{\mathbf{u}}^{\top}\left(\hat{\boldsymbol{\Sigma}}+\gamma_{n}\mathbf{I}\right)^{-1}\hat{\mathbf{u}}-\frac{r}{n}\stackrel{d}{\to}\mathbf{u}^{\top}\mathbf{\Sigma}^{-1}\mathbf{u}$
by the continuous mapping theorem, and the consistency of $\hat{\mathbf{\Sigma}}$.
Consequently, 
\begin{align*}
 & \lim_{n\to\infty}\mathbb{P}\left(\hat{\lambda}_{n}\ge r\right)\\
 & =1-\lim_{n\to\infty}\mathbb{P}\left(\hat{\mathbf{u}}^{\top}\left(\hat{\boldsymbol{\Sigma}}+\gamma_{n}\mathbf{I}\right)^{-1}\hat{\mathbf{u}}-\frac{r}{n}<0\right)\\
 & \stackrel{(a)}{=}1-\mathbb{P}\left(\mathbf{u}^{\top}\mathbf{\Sigma}^{-1}\mathbf{u}<0\right)\stackrel{(b)}{=}1,
\end{align*}
where at $(a)$ we use the Portmanteau theorem \citep[Lemma 2.2 (i)]{Vaart2000}
guaranteeing that $x_{n}\stackrel{d}{\to x}$ if and only if $\mathbb{P}(x_{n}<t)\to\mathbb{P}(x<t)$
for all continuity points of $t\mapsto\mathbb{P}(x<t)$.\textcolor{red}{{}
}Step $(b)$ is justified by noting that the covariance matrix $\boldsymbol{\Sigma}$
is positive definite so that $\mathbf{u}^{\top}\mathbf{\Sigma}^{-1}\mathbf{u}>0$,
and $t\mapsto\mathbb{P}(\mathbf{u}^{\top}\mathbf{\Sigma}^{-1}\mathbf{u}<t)$
(a step function) is continuous at $0$. 
\end{proof}

\section{Proof of Theorem \ref{thm:lower_bound_pow}\label{sec:proof_lb_pow}}

Recall Theorem \ref{thm:lower_bound_pow},

\lbpow*

\paragraph{Overview of the proof}

We first derive a probabilistic bound for $|\hat{\lambda}_{n}-\lambda_{n}|/n$.
The bound is in turn upper bounded by an expression involving $\|\hat{\mathbf{u}}-\mathbf{u}\|_{2}$
and $\|\hat{\mathbf{\Sigma}}-\mathbf{\Sigma}\|_{F}$. The difference
$\|\hat{\mathbf{u}}-\mathbf{u}\|_{2}$ can be bounded by applying
the bound for U-statistics given in \citet[Theorem A, p. 201]{Serfling2009}.
For $\|\hat{\mathbf{\Sigma}}-\mathbf{\Sigma}\|_{F}$, we decompose
it into a sum of smaller components, and bound each term with a product
variant of the Hoeffding's inequality (Lemma \ref{lem:bound_product_diff}).
$L(\lambda_{n})$ is obtained by combining all the bounds with the
union bound. 

\subsection{Notations}

Let $\left\langle \mathbf{A},\mathbf{B}\right\rangle _{F}:=\mathrm{tr}(\mathbf{A}^{\top}\mathbf{B})$
denote the Frobenius inner product, and $\|\mathbf{A}\|_{F}:=\sqrt{\mathrm{tr}(\mathbf{A}^{\top}\mathbf{A})}$
be the Frobenius norm. Write $\mathbf{z}:=(\mathbf{x},\mathbf{y})$
to denote a pair of points from $\mathcal{X}\times\mathcal{Y}$. We
write $\mathbf{t}:=(\mathbf{v},\mathbf{w})$ to denote a pair of test
locations from $\mathcal{X}\times\mathcal{Y}$. For brevity, an expectation
over $(\mathbf{x},\mathbf{y})$ (i.e., $\mathbb{E}_{(\mathbf{x},\mathbf{y})\sim P_{xy}}$)
will be written as $\mathbb{E}_{\mathbf{z}}$ or $\mathbb{E}_{\mathbf{x}\mathbf{y}}$.
Define $\tilde{k}(\mathbf{x},\mathbf{v}):=k(\mathbf{x},\mathbf{v})-\mathbb{E}_{\mathbf{x}'}k(\mathbf{x}',\mathbf{v})$,
and $\tilde{l}(\mathbf{y},\mathbf{w}):=l(\mathbf{y},\mathbf{w})-\mathbb{E}_{\mathbf{y}'}l(\mathbf{y}',\mathbf{w})$.
Let $B_{2}(r):=\{\mathbf{x}\mid\|\mathbf{x}\|_{2}\le r\}$ be a closed
ball with radius $r$ centered at the origin. Similarly, define $B_{F}(r):=\{\mathbf{A}\mid\|\mathbf{A}\|_{F}\le r\}$
to be a closed ball with radius\textbf{ }$r$ of $J\times J$ matrices
under the Frobenius norm. Denote the max operation by $(x_{1},\ldots,x_{m})_{+}=\max(x_{1},\ldots,x_{m})$.

For a product of marginal mean embeddings $\mu_{x}(\mathbf{v})\mu_{y}(\mathbf{w})$,
we write $\widehat{\mu_{x}\mu_{y}}(\mathbf{v},\mathbf{w}):=\frac{1}{n(n-1)}\sum_{i=1}^{n}\sum_{j\neq i}k(\mathbf{x}_{i},\mathbf{v})l(\mathbf{y}_{j},\mathbf{w})$
to denote the unbiased plug-in estimator, and write $\hat{\mu}_{x}(\mathbf{v})\hat{\mu}_{y}(\mathbf{w}):=\frac{1}{n}\sum_{i=1}^{n}k(\mathbf{x}_{i},\mathbf{v})\frac{1}{n}\sum_{j=1}^{n}l(\mathbf{y}_{j},\mathbf{w})$
which is a biased estimator. Define $\hat{u}^{b}(\mathbf{v},\mathbf{w}):=\hat{\mu}_{xy}(\mathbf{v},\mathbf{w})-\hat{\mu}_{x}(\mathbf{v})\hat{\mu}_{y}(\mathbf{w})$
so that $\hat{\mathbf{u}}^{b}:=\left(\hat{u}^{b}(\mathbf{t}_{1}),\ldots,\hat{u}^{b}(\mathbf{t}_{J})\right)^{\top}$
where the superscript $b$ stands for ``biased''. To avoid confusing
with a positive definite kernel, we will refer to a U-statistic kernel
as a \emph{core}. 

\subsection{Proof}

We will first derive a bound for $\mathbb{P}(|\hat{\lambda}_{n}-\lambda_{n}|\ge t)$,
which will then be reparametrized to get a bound for the target quantity
$\mathbb{P}(\hat{\lambda}_{n}\ge r)$. We closely follow the proof
in \citetsup[Section C.1]{Jitkrittum2016} up to (\ref{eq:bound_lamb_diff1}),
then we diverge. We start by considering $|\hat{\lambda}_{n}-\lambda_{n}|/n$.
\begin{align*}
|\hat{\lambda}_{n}-\lambda_{n}|/n & =\left|\hat{\mathbf{u}}^{\top}(\hat{\boldsymbol{\Sigma}}+\gamma_{n}\mathbf{I})^{-1}\hat{\mathbf{u}}-\mathbf{u}^{\top}\boldsymbol{\Sigma}^{-1}\mathbf{u}\right|\\
 & =\left|\hat{\mathbf{u}}^{\top}\left(\hat{\boldsymbol{\Sigma}}+\gamma_{n}\mathbf{I}\right)^{-1}\hat{\mathbf{u}}-\mathbf{u}^{\top}\left(\boldsymbol{\Sigma}+\gamma_{n}\mathbf{I}\right)^{-1}\mathbf{u}+\mathbf{u}^{\top}\left(\boldsymbol{\Sigma}+\gamma_{n}\mathbf{I}\right)^{-1}\mathbf{u}-\mathbf{u}^{\top}\boldsymbol{\Sigma}^{-1}\mathbf{u}\right|\\
 & \le\left|\hat{\mathbf{u}}^{\top}\left(\hat{\boldsymbol{\Sigma}}+\gamma_{n}\mathbf{I}\right)^{-1}\hat{\mathbf{u}}-\mathbf{u}^{\top}\left(\boldsymbol{\Sigma}+\gamma_{n}\mathbf{I}\right)^{-1}\mathbf{u}\right|+\left|\mathbf{u}^{\top}\left(\boldsymbol{\Sigma}+\gamma_{n}\mathbf{I}\right)^{-1}\mathbf{u}-\mathbf{u}^{\top}\boldsymbol{\Sigma}^{-1}\mathbf{u}\right|\\
 & :=\left(\bigstar\right)_{1}+\left(\bigstar\right)_{2}.
\end{align*}
We next bound $(\bigstar_{1})$ and $(\bigstar_{2})$ separately.

\begin{align}
(\bigstar)_{1} & =\left|\hat{\mathbf{u}}^{\top}\left(\hat{\boldsymbol{\Sigma}}+\gamma_{n}\mathbf{I}\right)^{-1}\hat{\mathbf{u}}-\mathbf{u}^{\top}\left(\boldsymbol{\Sigma}+\gamma_{n}\mathbf{I}\right)^{-1}\mathbf{u}\right|\nonumber \\
 & =\left|\hat{\mathbf{u}}^{\top}\left(\hat{\boldsymbol{\Sigma}}+\gamma_{n}\mathbf{I}\right)^{-1}\hat{\mathbf{u}}-\hat{\mathbf{u}}^{\top}\left(\boldsymbol{\Sigma}+\gamma_{n}\mathbf{I}\right)^{-1}\hat{\mathbf{u}}+\hat{\mathbf{u}}^{\top}\left(\boldsymbol{\Sigma}+\gamma_{n}\mathbf{I}\right)^{-1}\hat{\mathbf{u}}-\mathbf{u}^{\top}\left(\boldsymbol{\Sigma}+\gamma_{n}\mathbf{I}\right)^{-1}\mathbf{u}\right|\nonumber \\
 & \le\left|\hat{\mathbf{u}}^{\top}\left(\hat{\boldsymbol{\Sigma}}+\gamma_{n}\mathbf{I}\right)^{-1}\hat{\mathbf{u}}-\hat{\mathbf{u}}^{\top}\left(\boldsymbol{\Sigma}+\gamma_{n}\mathbf{I}\right)^{-1}\hat{\mathbf{u}}\right|+\left|\hat{\mathbf{u}}^{\top}\left(\boldsymbol{\Sigma}+\gamma_{n}\mathbf{I}\right)^{-1}\hat{\mathbf{u}}-\mathbf{u}^{\top}\left(\boldsymbol{\Sigma}+\gamma_{n}\mathbf{I}\right)^{-1}\mathbf{u}\right|\nonumber \\
 & =\left|\left\langle \hat{\mathbf{u}}\hat{\mathbf{u}}^{\top},\left(\hat{\boldsymbol{\Sigma}}+\gamma_{n}\mathbf{I}\right)^{-1}-\left(\boldsymbol{\Sigma}+\gamma_{n}\mathbf{I}\right)^{-1}\right\rangle _{F}\right|+\left|\left\langle \hat{\mathbf{u}}\hat{\mathbf{u}}^{\top}-\mathbf{u}\mathbf{u}^{\top},\left(\boldsymbol{\Sigma}+\gamma_{n}\mathbf{I}\right)^{-1}\right\rangle _{F}\right|\nonumber \\
 & \le\|\hat{\mathbf{u}}\hat{\mathbf{u}}^{\top}\|_{F}\|(\hat{\boldsymbol{\Sigma}}+\gamma_{n}\mathbf{I})^{-1}-(\boldsymbol{\Sigma}+\gamma_{n}\mathbf{I})^{-1}\|_{F}+\|\hat{\mathbf{u}}\hat{\mathbf{u}}^{\top}-\mathbf{u}\mathbf{u}^{\top}\|_{F}\|(\boldsymbol{\Sigma}+\gamma_{n}\mathbf{I})^{-1}\|_{F}\nonumber \\
 & =\|\hat{\mathbf{u}}\hat{\mathbf{u}}^{\top}\|_{F}\|(\hat{\boldsymbol{\Sigma}}+\gamma_{n}\mathbf{I})^{-1}[(\boldsymbol{\Sigma}+\gamma_{n}\mathbf{I})-(\hat{\boldsymbol{\Sigma}}+\gamma_{n}\mathbf{I})](\boldsymbol{\Sigma}+\gamma_{n}\mathbf{I})^{-1}\|_{F}+\|\hat{\mathbf{u}}\hat{\mathbf{u}}^{\top}-\hat{\mathbf{u}}\mathbf{u}^{\top}+\hat{\mathbf{u}}\mathbf{u}^{\top}-\mathbf{u}\mathbf{u}^{\top}\|_{F}\|(\boldsymbol{\Sigma}+\gamma_{n}\mathbf{I})^{-1}\|_{F}\nonumber \\
 & \stackrel{(a)}{\le}\|\hat{\mathbf{u}}\hat{\mathbf{u}}^{\top}\|_{F}\|(\hat{\boldsymbol{\Sigma}}+\gamma_{n}\mathbf{I})^{-1}\|_{F}\|\boldsymbol{\Sigma}-\hat{\boldsymbol{\Sigma}}\|_{F}\|\boldsymbol{\Sigma}^{-1}\|_{F}+\|\hat{\mathbf{u}}\hat{\mathbf{u}}^{\top}-\hat{\mathbf{u}}\mathbf{u}^{\top}+\hat{\mathbf{u}}\mathbf{u}^{\top}-\mathbf{u}\mathbf{u}^{\top}\|_{F}\|\boldsymbol{\Sigma}^{-1}\|_{F}\nonumber \\
 & \stackrel{(b)}{\le}\frac{\sqrt{J}}{\gamma_{n}}\|\hat{\mathbf{u}}\|_{2}^{2}\|\boldsymbol{\Sigma}-\hat{\boldsymbol{\Sigma}}\|_{F}\|\boldsymbol{\Sigma}^{-1}\|_{F}+\left(\|\hat{\mathbf{u}}(\hat{\mathbf{u}}-\mathbf{u})^{\top}\|_{F}+\|(\hat{\mathbf{u}}-\mathbf{u})\mathbf{u}^{\top}\|_{F}\right)\|\boldsymbol{\Sigma}^{-1}\|_{F}\nonumber \\
 & \le\frac{\sqrt{J}}{\gamma_{n}}\|\hat{\mathbf{u}}\|_{2}^{2}\|\boldsymbol{\Sigma}-\hat{\boldsymbol{\Sigma}}\|_{F}\|\boldsymbol{\Sigma}^{-1}\|_{F}+\left(\|\hat{\mathbf{u}}\|_{2}+\|\mathbf{u}\|_{2}\right)\|\hat{\mathbf{u}}-\mathbf{u}\|_{2}\|\boldsymbol{\Sigma}^{-1}\|_{F},\label{eq:bound_star1}
\end{align}
where at $(a)$ we used $\|(\boldsymbol{\Sigma}+\gamma_{n}\mathbf{I})^{-1}\|_{F}\le\|\boldsymbol{\Sigma}^{-1}\|_{F}$,
at $(b)$ we used $\|(\hat{\boldsymbol{\Sigma}}+\gamma_{n}\mathbf{I})^{-1}\|_{F}\le\sqrt{J}\|(\hat{\boldsymbol{\Sigma}}+\gamma_{n}\mathbf{I})^{-1}\|_{2}\le\sqrt{J}/\gamma_{n}$.

For $(\bigstar)_{2}$, we have
\begin{align}
(\bigstar)_{2} & =\left|\mathbf{u}^{\top}\left(\boldsymbol{\Sigma}+\gamma_{n}\mathbf{I}\right)^{-1}\mathbf{u}-\mathbf{u}^{\top}\boldsymbol{\Sigma}^{-1}\mathbf{u}\right|\nonumber \\
 & =\left|\left\langle \mathbf{u}\mathbf{u}^{\top},(\boldsymbol{\Sigma}+\gamma_{n}\mathbf{I})^{-1}-\boldsymbol{\Sigma}^{-1}\right\rangle _{F}\right|\nonumber \\
 & \le\|\mathbf{u}\mathbf{u}^{\top}\|_{F}\|(\boldsymbol{\Sigma}+\gamma_{n}\mathbf{I})^{-1}-\boldsymbol{\Sigma}^{-1}\|_{F}\nonumber \\
 & =\|\mathbf{u}\|_{2}^{2}\|(\boldsymbol{\Sigma}+\gamma_{n}\mathbf{I})^{-1}\left[\boldsymbol{\Sigma}-(\boldsymbol{\Sigma}+\gamma_{n}\mathbf{I})\right]\boldsymbol{\Sigma}^{-1}\|_{F}\nonumber \\
 & \le\gamma_{n}\|\mathbf{u}\|_{2}^{2}\|(\boldsymbol{\Sigma}+\gamma_{n}\mathbf{I})^{-1}\|_{F}\|\boldsymbol{\Sigma}^{-1}\|_{F}\nonumber \\
 & \stackrel{(a)}{\le}\gamma_{n}\|\mathbf{u}\|_{2}^{2}\|\boldsymbol{\Sigma}^{-1}\|_{F}^{2},\label{eq:bound_star2}
\end{align}
where at $(a)$ we used $\|(\boldsymbol{\Sigma}+\gamma_{n}\mathbf{I})^{-1}\|_{F}\le\|\boldsymbol{\Sigma}^{-1}\|_{F}$.

Combining (\ref{eq:bound_star1}) and (\ref{eq:bound_star2}), we
have 
\begin{align}
 & \left|\hat{\mathbf{u}}^{\top}(\hat{\boldsymbol{\Sigma}}+\gamma_{n}\mathbf{I})^{-1}\hat{\mathbf{u}}-\mathbf{u}^{\top}\boldsymbol{\Sigma}^{-1}\mathbf{u}\right|\nonumber \\
 & \le\frac{\sqrt{J}}{\gamma_{n}}\|\hat{\mathbf{u}}\|^{2}\|\boldsymbol{\Sigma}-\hat{\boldsymbol{\Sigma}}\|_{F}\|\boldsymbol{\Sigma}^{-1}\|_{F}+\left(\|\hat{\mathbf{u}}\|_{2}+\|\mathbf{u}\|_{2}\right)\|\hat{\mathbf{u}}-\mathbf{u}\|_{2}\|\boldsymbol{\Sigma}^{-1}\|_{F}+\gamma_{n}\|\mathbf{u}\|_{2}^{2}\|\boldsymbol{\Sigma}^{-1}\|_{F}^{2}.\label{eq:bound_uu1}
\end{align}

\paragraph{Bounding $\|\hat{\mathbf{u}}\|_{2}^{2}$ and $\|\mathbf{u}\|_{2}^{2}$}

Here, we show that by the boundedness of the kernels $k$ and $l$,
it follows that $\|\hat{\mathbf{u}}\|_{2}^{2}$ is bounded. Recall
that $\sup_{\mathbf{x},\mathbf{x}'\in\mathcal{X}}|k(\mathbf{x},\mathbf{x}')|\le B_{k}$,
$\sup_{\mathbf{y},\mathbf{y}'}|l(\mathbf{y},\mathbf{y}')|\le B_{l}$,
our notation $\mathbf{t}=(\mathbf{v},\mathbf{w})$ for the test locations,
and $\mathbf{z}_{i}:=(\mathbf{x}_{i},\mathbf{y}_{i})$. We first show
that the U-statistic core $h$ is bounded. 
\begin{align}
\left|h_{\mathbf{t}}((\mathbf{x},\mathbf{y}),(\mathbf{x}',\mathbf{y}'))\right| & =\left|\frac{1}{2}(k(\mathbf{x},\mathbf{v})-k(\mathbf{x}',\mathbf{v}))(l(\mathbf{y},\mathbf{w})-l(\mathbf{y}',\mathbf{w}))\right|\nonumber \\
 & \le\frac{1}{2}\left(|k(\mathbf{x},\mathbf{v})|+|k(\mathbf{x}',\mathbf{v})|\right)\left(|l(\mathbf{y},\mathbf{w})|+|l(\mathbf{y}',\mathbf{w})|\right)\nonumber \\
 & \le2B_{k}B_{l}:=2B,\label{eq:ustat_core_bound}
\end{align}
where we define $B:=B_{k}B_{l}$. It follows that 
\begin{align}
\|\hat{\mathbf{u}}\|_{2}^{2} & =\sum_{m=1}^{J}\left[\frac{2}{n(n-1)}\sum_{i<j}h_{\mathbf{t}_{m}}(\mathbf{z}_{i},\mathbf{z}_{j})\right]^{2}\le\sum_{m=1}^{J}\left[2B_{k}B_{l}\right]^{2}=4B^{2}J,\label{eq:bound_uhat}\\
\|\mathbf{u}\|_{2}^{2} & =\sum_{m=1}^{J}\left[\mathbb{E}_{\mathbf{z}}\mathbb{E}_{\mathbf{z}'}h_{\mathbf{t}_{m}}(\mathbf{z},\mathbf{z}')\right]^{2}\le4B^{2}J.\label{eq:bound_u}
\end{align}

Using the upper bounds on $\|\hat{\mathbf{u}}\|_{2}^{2}$, $\|\mathbf{u}\|_{2}^{2}$
,(\ref{eq:bound_uu1}) and the definition of $\tilde{c}$, we have
\begin{align}
 & \left|\hat{\mathbf{u}}^{\top}(\hat{\boldsymbol{\Sigma}}+\gamma_{n}\mathbf{I})^{-1}\hat{\mathbf{u}}-\mathbf{u}^{\top}\boldsymbol{\Sigma}^{-1}\mathbf{u}\right|\nonumber \\
 & \le\frac{\sqrt{J}}{\gamma_{n}}4B^{2}J\tilde{c}\|\boldsymbol{\Sigma}-\hat{\boldsymbol{\Sigma}}\|_{F}+4B\sqrt{J}\tilde{c}\|\hat{\mathbf{u}}-\mathbf{u}\|_{2}+4B^{2}J\tilde{c}^{2}\gamma_{n}\nonumber \\
 & =:\frac{c_{1}}{\gamma_{n}}\|\boldsymbol{\Sigma}-\hat{\boldsymbol{\Sigma}}\|_{F}+c_{2}\|\hat{\mathbf{u}}-\mathbf{u}\|_{2}+c_{3}\gamma_{n},\label{eq:bound_uu2}
\end{align}
where we define $c_{1}:=4B^{2}J\sqrt{J}\tilde{c}$, $c_{2}:=4B\sqrt{J}\tilde{c}$,
and $c_{3}:=4B^{2}J\tilde{c}^{2}$. This upper bound implies that
\begin{align}
|\hat{\lambda}_{n}-\lambda_{n}| & \le\frac{c_{1}}{\gamma_{n}}n\|\boldsymbol{\Sigma}-\hat{\boldsymbol{\Sigma}}\|_{F}+c_{2}n\|\hat{\mathbf{u}}-\mathbf{u}\|_{2}+c_{3}n\gamma_{n}.\label{eq:bound_lamb_diff1}
\end{align}
We will separately upper bound $\|\boldsymbol{\Sigma}-\hat{\boldsymbol{\Sigma}}\|_{F}$
and $\|\hat{\mathbf{u}}-\mathbf{u}\|_{2}$, and combine them with
a union bound.

\subsubsection{Bounding $\|\hat{\mathbf{u}}-\mathbf{u}\|_{2}$}

Let $\mathbf{t}^{*}=\arg\max_{\mathbf{t}\in\{\mathbf{t}_{1},\ldots,\mathbf{t}_{J}\}}|\hat{u}(\mathbf{t})-u(\mathbf{t})|$.
Recall that $\mathbf{u}=(u(\mathbf{t}_{1}),\ldots,u(\mathbf{t}_{J}))^{\top}=(u_{1},\ldots,u_{J})^{\top}$.
\begin{align}
\|\hat{\mathbf{u}}-\mathbf{u}\|_{2} & =\sup_{\mathbf{b}\in B_{2}(1)}\left\langle \mathbf{b},\hat{\mathbf{u}}-\mathbf{u}\right\rangle _{2}\le\sup_{\mathbf{b}\in B_{2}(1)}\sum_{j=1}^{J}|b_{j}||\hat{u}(\mathbf{t}_{j})-u(\mathbf{t}_{j})|\nonumber \\
 & \le|\hat{u}(\mathbf{t}^{*})-u(\mathbf{t}^{*})|\sup_{\mathbf{b}\in B_{2}(1)}\sum_{j=1}^{J}|b_{j}|\nonumber \\
 & \stackrel{(a)}{\le}\sqrt{J}|\hat{u}(\mathbf{t}^{*})-u(\mathbf{t}^{*})|\sup_{\mathbf{b}\in B_{2}(1)}\|\mathbf{b}\|_{2}\nonumber \\
 & =\sqrt{J}|\hat{u}(\mathbf{t}^{*})-u(\mathbf{t}^{*})|,\label{eq:bound_udiff1}
\end{align}
where at $(a)$ we used $\|\mathbf{a}\|_{1}\le\sqrt{J}\|\mathbf{a}\|_{2}$
for any $\mathbf{a}\in\mathbb{R}^{J}$. From (\ref{eq:bound_udiff1}),
it can be seen that bounding $\|\hat{\mathbf{u}}-\mathbf{u}\|_{2}$
amounts to bounding the difference of a U-statistic $\hat{u}(\mathbf{t}^{*})$
(see (\ref{eq:ustat})) to its expectation $u(\mathbf{t}^{*})$. Combining
(\ref{eq:bound_udiff1}) and (\ref{eq:bound_lamb_diff1}), we have
\begin{align}
|\hat{\lambda}_{n}-\lambda_{n}| & \le\frac{c_{1}}{\gamma_{n}}n\|\boldsymbol{\Sigma}-\hat{\boldsymbol{\Sigma}}\|_{F}+c_{2}n\sqrt{J}|\hat{u}(\mathbf{t}^{*})-u(\mathbf{t}^{*})|+c_{3}n\gamma_{n}.\label{eq:bound_lamb_diff2}
\end{align}

\subsubsection{Bounding $\|\hat{\mathbf{\Sigma}}-\mathbf{\Sigma}\|_{F}$\label{sec:bound_sigma}}

The plan is to write $\hat{\mathbf{\Sigma}}=\hat{\mathbf{S}}-\hat{\mathbf{u}}^{b}\hat{\mathbf{u}}^{b\top},\mathbf{\Sigma}=\mathbf{S}-\mathbf{u}\mathbf{u}^{\top},$
so that $\|\hat{\boldsymbol{\Sigma}}-\boldsymbol{\Sigma}\|_{F}\le\|\hat{\mathbf{S}}-\mathbf{S}\|_{F}+\|\hat{\mathbf{u}}^{b}\hat{\mathbf{u}}^{b\top}-\mathbf{u}\mathbf{u}^{\top}\|_{F}$
and bound separately $\|\hat{\mathbf{S}}-\mathbf{S}\|_{F}$ and $\|\hat{\mathbf{u}}^{b}\hat{\mathbf{u}}^{b\top}-\mathbf{u}\mathbf{u}^{\top}\|_{F}$. 

Recall that $\Sigma_{ij}=\eta(\mathbf{t}_{i},\mathbf{t}_{j}),\eta(\mathbf{t},\mathbf{t}')=\mathbb{E}_{\mathbf{x}\mathbf{y}}[\big(\tilde{k}(\mathbf{x},\mathbf{v})\tilde{l}(\mathbf{y},\mathbf{w})-u(\mathbf{v},\mathbf{w})\big)\big(\tilde{k}(\mathbf{x},\mathbf{v}')\tilde{l}(\mathbf{y},\mathbf{w}')-u(\mathbf{v}',\mathbf{w}')\big)]$
where $\tilde{k}(\mathbf{x},\mathbf{v})=k(\mathbf{x},\mathbf{v})-\mathbb{E}_{\mathbf{x}'}k(\mathbf{x}',\mathbf{v})$,
and $\tilde{l}(\mathbf{y},\mathbf{w})=l(\mathbf{y},\mathbf{w})-\mathbb{E}_{\mathbf{y}'}l(\mathbf{y}',\mathbf{w})$.
Its empirical estimator (see Proposition \ref{prop:estimators}) is
$\hat{\Sigma}_{ij}=\hat{\eta}(\mathbf{t}_{i},\mathbf{t}_{j})$ where
\begin{align*}
\hat{\eta}(\mathbf{t},\mathbf{t}') & =\frac{1}{n}\sum_{i=1}^{n}[\big(\overline{k}(\mathbf{x}_{i},\mathbf{v})\overline{l}(\mathbf{y}_{i},\mathbf{w})-\hat{u}^{b}(\mathbf{v},\mathbf{w})\big)\big(\overline{k}(\mathbf{x}_{i},\mathbf{v}')\overline{l}(\mathbf{y}_{i},\mathbf{w}')-\hat{u}^{b}(\mathbf{v}',\mathbf{w}')\big)]\\
 & =\frac{1}{n}\sum_{i=1}^{n}\overline{k}(\mathbf{x}_{i},\mathbf{v})\overline{l}(\mathbf{y}_{i},\mathbf{w})\overline{k}(\mathbf{x}_{i},\mathbf{v}')\overline{l}(\mathbf{y}_{i},\mathbf{w}')-\hat{u}^{b}(\mathbf{v},\mathbf{w})\hat{u}^{b}(\mathbf{v}',\mathbf{w}'),
\end{align*}
$\overline{k}(\mathbf{x},\mathbf{v}):=k(\mathbf{x},\mathbf{v})-\frac{1}{n}\sum_{i=1}^{n}k(\mathbf{x}_{i},\mathbf{v})$,
and $\overline{l}(\mathbf{y},\mathbf{w}):=l(\mathbf{y},\mathbf{w})-\frac{1}{n}\sum_{i=1}^{n}l(\mathbf{y}_{i},\mathbf{w})$.
We note that $\frac{1}{n}\sum_{i=1}^{n}\overline{k}(\mathbf{x}_{i},\mathbf{v})\overline{l}(\mathbf{y}_{i},\mathbf{w})=\hat{u}^{b}(\mathbf{v},\mathbf{w})$.
We define\textbf{ $\hat{\mathbf{S}}\in\mathbb{R}^{J\times J}$} such
that $\hat{S}_{ij}:=\frac{1}{n}\sum_{m=1}^{n}\overline{k}(\mathbf{x}_{m},\mathbf{v}_{i})\overline{l}(\mathbf{y}_{m},\mathbf{w}_{i})\overline{k}(\mathbf{x}_{m},\mathbf{v}_{j})\overline{l}(\mathbf{y}_{i},\mathbf{w}_{j})$,
and define similarly its population counterpart $\mathbf{S}$ such
that $S_{ij}:=\mathbb{E}_{\mathbf{x}\mathbf{y}}[\tilde{k}(\mathbf{x},\mathbf{v})\tilde{l}(\mathbf{y},\mathbf{w})\tilde{k}(\mathbf{x},\mathbf{v}')\tilde{l}(\mathbf{y},\mathbf{w}')]$.
We have 
\begin{align}
\hat{\mathbf{\Sigma}} & =\hat{\mathbf{S}}-\hat{\mathbf{u}}^{b}\hat{\mathbf{u}}^{b\top},\nonumber \\
\mathbf{\Sigma} & =\mathbf{S}-\mathbf{u}\mathbf{u}^{\top},\nonumber \\
\|\hat{\boldsymbol{\Sigma}}-\boldsymbol{\Sigma}\|_{F} & =\|\hat{\mathbf{S}}-\mathbf{S}-(\hat{\mathbf{u}}^{b}\hat{\mathbf{u}}^{b\top}-\mathbf{u}\mathbf{u}^{\top})\|_{F}\\
 & \le\|\hat{\mathbf{S}}-\mathbf{S}\|_{F}+\|\hat{\mathbf{u}}^{b}\hat{\mathbf{u}}^{b\top}-\mathbf{u}\mathbf{u}^{\top}\|_{F}.\label{eq:bound_sigma_diff1}
\end{align}
With (\ref{eq:bound_sigma_diff1}), (\ref{eq:bound_lamb_diff2}) becomes
\begin{align}
|\hat{\lambda}_{n}-\lambda_{n}| & \le\frac{c_{1}n}{\gamma_{n}}\|\hat{\mathbf{S}}-\mathbf{S}\|_{F}+\frac{c_{1}n}{\gamma_{n}}\|\hat{\mathbf{u}}^{b}\hat{\mathbf{u}}^{b\top}-\mathbf{u}\mathbf{u}^{\top}\|_{F}+c_{2}n\sqrt{J}|\hat{u}(\mathbf{t}^{*})-u(\mathbf{t}^{*})|+c_{3}n\gamma_{n}.\label{eq:bound_lamb_diff3}
\end{align}
We will further separately bound $\|\hat{\mathbf{S}}-\mathbf{S}\|_{F}$
and $\|\hat{\mathbf{u}}^{b}\hat{\mathbf{u}}^{b\top}-\mathbf{u}\mathbf{u}^{\top}\|_{F}$. 

\subsubsection{Bounding $\|\hat{\mathbf{u}}^{b}\hat{\mathbf{u}}^{b\top}-\mathbf{u}\mathbf{u}^{\top}\|_{F}$ }

\begin{align*}
\|\hat{\mathbf{u}}^{b}\hat{\mathbf{u}}^{b\top}-\mathbf{u}\mathbf{u}^{\top}\|_{F} & =\|\hat{\mathbf{u}}^{b}\hat{\mathbf{u}}^{b\top}-\hat{\mathbf{u}}^{b}\mathbf{u}^{\top}+\hat{\mathbf{u}}^{b}\mathbf{u}^{\top}-\mathbf{u}\mathbf{u}^{\top}\|_{F}\\
 & \le\|\hat{\mathbf{u}}^{b}(\hat{\mathbf{u}}^{b}-\mathbf{u})^{\top}\|_{F}+\|(\hat{\mathbf{u}}^{b}-\mathbf{u})\mathbf{u}^{\top}\|_{F}\\
 & =\|\hat{\mathbf{u}}^{b}\|_{2}\|\hat{\mathbf{u}}^{b}-\mathbf{u}\|_{2}+\|\hat{\mathbf{u}}^{b}-\mathbf{u}\|_{2}\|\mathbf{u}\|_{2}\\
 & \le4B\sqrt{J}\|\hat{\mathbf{u}}^{b}-\mathbf{u}\|_{2},
\end{align*}
where we used (\ref{eq:bound_u}) and the fact that $\|\hat{\mathbf{u}}^{b}\|_{2}\le2B\sqrt{J}$
which can be shown similarly to (\ref{eq:bound_uhat}) as
\begin{align*}
\|\hat{\mathbf{u}}^{b}\|_{2}^{2} & =\sum_{m=1}^{J}\left[\hat{\mu}_{xy}(\mathbf{v}_{m},\mathbf{w}_{m})-\hat{\mu}_{x}(\mathbf{v}_{m})\hat{\mu}_{y}(\mathbf{w}_{m})\right]^{2}=\sum_{m=1}^{J}\left[\frac{1}{n^{2}}\sum_{i=1}^{n}\sum_{j=1}^{n}h_{\mathbf{t}_{m}}(\mathbf{z}_{i},\mathbf{z}_{j})\right]^{2}\le\sum_{m=1}^{J}\left[2B_{k}B_{l}\right]^{2}=4B^{2}J.
\end{align*}
 Let $(\tilde{\mathbf{v}},\tilde{\mathbf{w}}):=\tilde{\mathbf{t}}=\arg\max_{\mathbf{t}\in\{\mathbf{t}_{1},\ldots,\mathbf{t}_{J}\}}|\hat{u}^{b}(\mathbf{t})-u(\mathbf{t})|$.
We bound $\|\hat{\mathbf{u}}^{b}-\mathbf{u}\|_{2}$ by 
\begin{align}
\|\hat{\mathbf{u}}^{b}-\mathbf{u}\|_{2} & \stackrel{(a)}{\le}\sqrt{J}|\hat{u}^{b}(\tilde{\mathbf{t}})-u(\tilde{\mathbf{t}})|\nonumber \\
 & =\sqrt{J}\left|\hat{\mu}_{xy}(\tilde{\mathbf{t}})-\hat{\mu}_{x}(\tilde{\mathbf{v}})\hat{\mu}_{y}(\tilde{\mathbf{w}})-u(\tilde{\mathbf{t}})\right|\nonumber \\
 & =\sqrt{J}\left|\hat{\mu}_{xy}(\tilde{\mathbf{t}})-\widehat{\mu_{x}\mu_{y}}(\tilde{\mathbf{t}})+\widehat{\mu_{x}\mu_{y}}(\tilde{\mathbf{t}})-\hat{\mu}_{x}(\tilde{\mathbf{v}})\hat{\mu}_{y}(\tilde{\mathbf{w}})-u(\tilde{\mathbf{t}})\right|\nonumber \\
 & \le\sqrt{J}\left|\hat{\mu}_{xy}(\tilde{\mathbf{t}})-\widehat{\mu_{x}\mu_{y}}(\tilde{\mathbf{t}})-u(\tilde{\mathbf{t}})\right|+\sqrt{J}\left|\widehat{\mu_{x}\mu_{y}}(\tilde{\mathbf{t}})-\hat{\mu}_{x}(\tilde{\mathbf{v}})\hat{\mu}_{y}(\tilde{\mathbf{w}})\right|\nonumber \\
 & =\sqrt{J}\left|\hat{u}(\tilde{\mathbf{t}})-u(\tilde{\mathbf{t}})\right|+\sqrt{J}\left|\widehat{\mu_{x}\mu_{y}}(\tilde{\mathbf{t}})-\hat{\mu}_{x}(\tilde{\mathbf{v}})\hat{\mu}_{y}(\tilde{\mathbf{w}})\right|,\label{eq:bound_ubdiff1}
\end{align}
where at $(a)$ we used the same reasoning as in (\ref{eq:bound_udiff1}).
The bias $\left|\widehat{\mu_{x}\mu_{y}}(\tilde{\mathbf{t}})-\hat{\mu}_{x}(\tilde{\mathbf{v}})\hat{\mu}_{y}(\tilde{\mathbf{w}})\right|$
in the second term can be bounded as 
\begin{align*}
 & \left|\widehat{\mu_{x}\mu_{y}}(\tilde{\mathbf{t}})-\hat{\mu}_{x}(\tilde{\mathbf{v}})\hat{\mu}_{y}(\tilde{\mathbf{w}})\right|\\
 & =\left|\frac{1}{n(n-1)}\sum_{i=1}^{n}\sum_{j\neq i}k(\mathbf{x}_{i},\tilde{\mathbf{v}})l(\mathbf{y}_{j},\tilde{\mathbf{w}})-\frac{1}{n^{2}}\sum_{i=1}^{n}\sum_{j=1}^{n}k(\mathbf{x}_{i},\tilde{\mathbf{v}})l(\mathbf{y}_{j},\tilde{\mathbf{w}})\right|\\
 & =\left|\frac{1}{n(n-1)}\sum_{i=1}^{n}\sum_{j=1}^{n}k(\mathbf{x}_{i},\tilde{\mathbf{v}})l(\mathbf{y}_{j},\tilde{\mathbf{w}})-\frac{1}{n(n-1)}\sum_{i=1}^{n}k(\mathbf{x}_{i},\tilde{\mathbf{v}})l(\mathbf{y}_{i},\tilde{\mathbf{w}})-\frac{1}{n^{2}}\sum_{i=1}^{n}\sum_{j=1}^{n}k(\mathbf{x}_{i},\tilde{\mathbf{v}})l(\mathbf{y}_{j},\tilde{\mathbf{w}})\right|\\
 & =\left|\left(1-\frac{n}{n-1}\right)\frac{1}{n^{2}}\sum_{i=1}^{n}\sum_{j=1}^{n}k(\mathbf{x}_{i},\tilde{\mathbf{v}})l(\mathbf{y}_{j},\tilde{\mathbf{w}})+\frac{1}{n(n-1)}\sum_{i=1}^{n}k(\mathbf{x}_{i},\tilde{\mathbf{v}})l(\mathbf{y}_{i},\tilde{\mathbf{w}})\right|\\
 & \le\left|\left(1-\frac{n}{n-1}\right)\frac{1}{n^{2}}\sum_{i=1}^{n}\sum_{j=1}^{n}k(\mathbf{x}_{i},\tilde{\mathbf{v}})l(\mathbf{y}_{j},\tilde{\mathbf{w}})\right|+\left|\frac{1}{n(n-1)}\sum_{i=1}^{n}k(\mathbf{x}_{i},\tilde{\mathbf{v}})l(\mathbf{y}_{i},\tilde{\mathbf{w}})\right|\\
 & \le\frac{B}{n-1}+\frac{B}{n-1}=\frac{2B}{n-1}.
\end{align*}
Combining this upper bound with (\ref{eq:bound_ubdiff1}), we have
\begin{align}
\|\hat{\mathbf{u}}^{b}\hat{\mathbf{u}}^{b\top}-\mathbf{u}\mathbf{u}^{\top}\|_{F} & \le4BJ\left|\hat{u}(\tilde{\mathbf{t}})-u(\tilde{\mathbf{t}})\right|+\frac{8B^{2}J}{n-1}.\label{eq:bound_uu_diff1}
\end{align}
With (\ref{eq:bound_uu_diff1}), (\ref{eq:bound_lamb_diff3}) becomes
\begin{align}
|\hat{\lambda}_{n}-\lambda_{n}| & \le\frac{c_{1}n}{\gamma_{n}}\|\hat{\mathbf{S}}-\mathbf{S}\|_{F}+\frac{4BJc_{1}n}{\gamma_{n}}\left|\hat{u}(\tilde{\mathbf{t}})-u(\tilde{\mathbf{t}})\right|+\frac{c_{1}n}{\gamma_{n}}\frac{8B^{2}J}{n-1}+c_{2}n\sqrt{J}|\hat{u}(\mathbf{t}^{*})-u(\mathbf{t}^{*})|+c_{3}n\gamma_{n}.\label{eq:bound_lamb_diff4}
\end{align}

\subsubsection{Bounding $\|\hat{\mathbf{S}}-\mathbf{S}\|_{F}$ }

Recall that $V_{J}=\{\mathbf{t}_{1},\ldots,\mathbf{t}_{J}\}$, $\hat{S}_{ij}=\hat{S}(\mathbf{t}_{i},\mathbf{t}_{j})=\frac{1}{n}\sum_{m=1}^{n}\overline{k}(\mathbf{x}_{m},\mathbf{v}_{i})\overline{l}(\mathbf{y}_{m},\mathbf{w}_{i})\overline{k}(\mathbf{x}_{m},\mathbf{v}_{j})\overline{l}(\mathbf{y}_{m},\mathbf{w}_{j})$,
and $S_{ij}=S(\mathbf{t}_{i},\mathbf{t}_{j})=\mathbb{E}_{\mathbf{x}\mathbf{y}}[\tilde{k}(\mathbf{x},\mathbf{v}_{i})\tilde{l}(\mathbf{y},\mathbf{w}_{i})\tilde{k}(\mathbf{x},\mathbf{v}_{j})\tilde{l}(\mathbf{y},\mathbf{w}_{j})]$.
Let $(\mathbf{t}^{(1)},\mathbf{t}^{(2)})=\arg\max_{(\mathbf{s},\mathbf{t})\in V_{J}\times V_{J}}|\hat{S}(\mathbf{s},\mathbf{t})-S(\mathbf{s},\mathbf{t})|$.

\begin{align}
\|\hat{\mathbf{S}}-\mathbf{S}\|_{F} & =\sup_{\mathbf{B}\in B_{F}(1)}\left\langle \mathbf{B},\hat{\mathbf{S}}-\mathbf{S}\right\rangle _{F}\nonumber \\
 & \le\sup_{\mathbf{B}\in B_{F}(1)}\sum_{i=1}^{J}\sum_{j=1}^{J}|B_{ij}||\hat{S}_{ij}-S_{ij}|\nonumber \\
 & \le\left|\hat{S}(\mathbf{t}^{(1)},\mathbf{t}^{(2)})-S(\mathbf{t}^{(1)},\mathbf{t}^{(2)})\right|\sup_{\mathbf{B}\in B_{F}(1)}\sum_{i=1}^{J}\sum_{j=1}^{J}|B_{ij}|\nonumber \\
 & \stackrel{(a)}{\le}J\left|\hat{S}(\mathbf{t}^{(1)},\mathbf{t}^{(2)})-S(\mathbf{t}^{(1)},\mathbf{t}^{(2)})\right|\sup_{\mathbf{B}\in B_{F}(1)}\|\mathbf{B}\|_{F}\nonumber \\
 & =J\left|\hat{S}(\mathbf{t}^{(1)},\mathbf{t}^{(2)})-S(\mathbf{t}^{(1)},\mathbf{t}^{(2)})\right|,\label{eq:bound_sdiff1}
\end{align}
where at $(a)$ we used $\sum_{i=1}^{J}\sum_{j=1}^{J}|A_{ij}|\le J\|\mathbf{A}\|_{F}$
for any matrix $\mathbf{A}\in\mathbb{R}^{J\times J}$. We arrive at
\begin{align}
|\hat{\lambda}_{n}-\lambda_{n}| & \le\frac{c_{1}Jn}{\gamma_{n}}\left|\hat{S}(\mathbf{t}^{(1)},\mathbf{t}^{(2)})-S(\mathbf{t}^{(1)},\mathbf{t}^{(2)})\right|+\frac{4BJc_{1}n}{\gamma_{n}}\left|\hat{u}(\tilde{\mathbf{t}})-u(\tilde{\mathbf{t}})\right|\nonumber \\
 & +\frac{c_{1}n}{\gamma_{n}}\frac{8B^{2}J}{n-1}+c_{2}n\sqrt{J}|\hat{u}(\mathbf{t}^{*})-u(\mathbf{t}^{*})|+c_{3}n\gamma_{n}.\label{eq:bound_lamb_diff5}
\end{align}

\subsubsection{Bounding $\left|\hat{S}(\mathbf{t},\mathbf{t}')-S(\mathbf{t},\mathbf{t}')\right|$}

Having an upper bound for $\left|\hat{S}(\mathbf{t},\mathbf{t}')-S(\mathbf{t},\mathbf{t}')\right|$
will allow us to bound (\ref{eq:bound_lamb_diff5}). To keep the notations
uncluttered, we will define the following shorthands. 
\begin{center}
\begin{tabular}{cc}
\toprule 
Expression & Shorthand\tabularnewline
\midrule
\midrule 
$k(\mathbf{x},\mathbf{v})$ & $a$\tabularnewline
\midrule 
$k(\mathbf{x},\mathbf{v}')$ & $a'$\tabularnewline
\midrule 
$k(\mathbf{x}_{i},\mathbf{v})$ & $a_{i}$\tabularnewline
\midrule 
$k(\mathbf{x}_{i},\mathbf{v}')$ & $a_{i}'$\tabularnewline
\midrule 
$\mathbb{E}_{\mathbf{x}\sim P_{x}}k(\mathbf{x},\mathbf{v})$ & $\tilde{a}$\tabularnewline
\midrule 
$\mathbb{E}_{\mathbf{x}\sim P_{x}}k(\mathbf{x},\mathbf{v}')$ & $\tilde{a}'$\tabularnewline
\midrule 
$\frac{1}{n}\sum_{i=1}^{n}k(\mathbf{x}_{i},\mathbf{v})$ & $\overline{a}$\tabularnewline
\midrule 
$\frac{1}{n}\sum_{i=1}^{n}k(\mathbf{x}_{i},\mathbf{v}')$ & $\overline{a}'$\tabularnewline
\bottomrule
\end{tabular}\hspace{1cm} %
\begin{tabular}{cc}
\toprule 
Expression & Shorthand\tabularnewline
\midrule
\midrule 
$l(\mathbf{y},\mathbf{w})$ & $b$\tabularnewline
\midrule 
$l(\mathbf{y},\mathbf{w}')$ & $b'$\tabularnewline
\midrule 
$l(\mathbf{y}_{i},\mathbf{w})$ & $b_{i}$\tabularnewline
\midrule 
$l(\mathbf{y}_{i},\mathbf{w}')$ & \textbf{$b_{i}'$}\tabularnewline
\midrule 
$\mathbb{E}_{\mathbf{y}\sim P_{y}}l(\mathbf{y},\mathbf{w})$ & $\tilde{b}$\tabularnewline
\midrule 
$\mathbb{E}_{\mathbf{y}\sim P_{y}}l(\mathbf{y},\mathbf{w}')$ & $\tilde{b}'$\tabularnewline
\midrule 
$\frac{1}{n}\sum_{i=1}^{n}l(\mathbf{y}_{i},\mathbf{w})$ & $\overline{b}$\tabularnewline
\midrule 
$\frac{1}{n}\sum_{i=1}^{n}l(\mathbf{y}_{i},\mathbf{w}')$ & $\overline{b}'$\tabularnewline
\bottomrule
\end{tabular}
\par\end{center}

We will also use $\overline{\thinspace\cdot\thinspace}$ to denote
a empirical expectation over $\mathbf{x}$, or \textbf{$\mathbf{y},$
}or $(\mathbf{x},\mathbf{y})$. The argument under $\overline{\thinspace\cdot\thinspace}$
will determine the variable over which we take the expectation. For
instance, $\overline{aa'}=\frac{1}{n}\sum_{i=1}^{n}k(\mathbf{x}_{i},\mathbf{v})k(\mathbf{x}_{i},\mathbf{v}')$
and $\overline{aba'}=\frac{1}{n}\sum_{i=1}^{n}k(\mathbf{x}_{i},\mathbf{v})l(\mathbf{y}_{i},\mathbf{w})k(\mathbf{x}_{i},\mathbf{v}')$,
and so on. We define in the same way for the population expectation
using $\widetilde{\cdot}$ i.e., $\widetilde{aa'}=\mathbb{E}_{\mathbf{x}}\left[k(\mathbf{x},\mathbf{v})k(\mathbf{x},\mathbf{v}')\right]$
and $\widetilde{aba'}=\mathbb{E}_{\mathbf{x}\mathbf{y}}\left[k(\mathbf{x},\mathbf{v})l(\mathbf{y},\mathbf{w})k(\mathbf{x},\mathbf{v}')\right]$.

With these shorthands, we can rewrite $\hat{S}(\mathbf{t},\mathbf{t}')$
and $S(\mathbf{t},\mathbf{t}')$ as 
\begin{align*}
\hat{S}(\mathbf{t},\mathbf{t}') & =\frac{1}{n}\sum_{i=1}^{n}(a_{i}-\overline{a})(b_{i}-\overline{b})(a_{i}'-\overline{a}')(b_{i}'-\overline{b}'),\\
S(\mathbf{t},\mathbf{t}') & =\mathbb{E}_{\mathbf{x}\mathbf{y}}\left[(a-\tilde{a})(b-\tilde{b})(a'-\tilde{a}')(b'-\tilde{b}')\right].
\end{align*}
By expanding $S(\mathbf{t},\mathbf{t}')$, we have 
\begin{align*}
S(\mathbf{t},\mathbf{t}') & =\mathbb{E}_{\mathbf{x}\mathbf{y}}\big[+aba'b'-aba'\tilde{b}'-ab\tilde{a}'b'+ab\tilde{a}'\tilde{b}'\\
 & \phantom{\thinspace=\mathbb{E}_{\mathbf{x}\mathbf{y}}\big[}-a\tilde{b}a'b'+a\tilde{b}a'\tilde{b}'+a\tilde{b}\tilde{a}'b'-a\tilde{b}\tilde{a}'\tilde{b}'\\
 & \phantom{\thinspace=\mathbb{E}_{\mathbf{x}\mathbf{y}}\big[}-\tilde{a}ba'b'+\tilde{a}ba'\tilde{b}'+\tilde{a}b\tilde{a}'b'-\tilde{a}b\tilde{a}'\tilde{b}'\\
 & \phantom{\thinspace=\mathbb{E}_{\mathbf{x}\mathbf{y}}\big[}+\tilde{a}\tilde{b}a'b'-\tilde{a}\tilde{b}a'\tilde{b}'-\tilde{a}\tilde{b}\tilde{a}'\tilde{b}'+\tilde{a}\tilde{b}\tilde{a}'\tilde{b}'\big]\\
 & =+\widetilde{aba'b'}-\widetilde{aba'}\tilde{b}'-\widetilde{abb'}\tilde{a}'+\widetilde{ab}\tilde{a}'\tilde{b}'\\
 & \phantom{=}-\widetilde{aa'b'}\tilde{b}+\widetilde{aa'}\tilde{b}\tilde{b}'+\widetilde{ab'}\tilde{a}'\tilde{b}-{\color{blue}\tilde{a}\tilde{b}\tilde{a}'\tilde{b}'}\\
 & \phantom{=}-\widetilde{a'bb'}\tilde{a}+\widetilde{a'b}\tilde{a}\tilde{b}'+\tilde{a}\tilde{a}'\widetilde{bb'}-{\color{purple}\tilde{a}\tilde{b}\tilde{a}'\tilde{b}'}\\
 & \phantom{=}+\widetilde{a'b'}\tilde{a}\tilde{b}-{\color{purple}\tilde{a}\tilde{b}\tilde{a}'\tilde{b}'}-{\color{purple}\tilde{a}\tilde{b}\tilde{a}'\tilde{b}'}+{\color{blue}\tilde{a}\tilde{b}\tilde{a}'\tilde{b}'}\\
 & =+\widetilde{aba'b'}-\widetilde{aba'}\tilde{b}'-\widetilde{abb'}\tilde{a}'+\widetilde{ab}\tilde{a}'\tilde{b}'\\
 & \phantom{=}-\widetilde{aa'b'}\tilde{b}+\widetilde{aa'}\tilde{b}\tilde{b}'+\widetilde{ab'}\tilde{a}'\tilde{b}+\widetilde{a'b'}\tilde{a}\tilde{b}\\
 & \phantom{=}-\widetilde{a'bb'}\tilde{a}+\widetilde{a'b}\tilde{a}\tilde{b}'+\tilde{a}\tilde{a}'\widetilde{bb'}-3\tilde{a}\tilde{b}\tilde{a}'\tilde{b}'.
\end{align*}
The expansion of $\hat{S}(\mathbf{t},\mathbf{t}')$ can be done in
the same way. By the triangle inequality, we have
\begin{align*}
\left|\hat{S}(\mathbf{t},\mathbf{t}')-S(\mathbf{t},\mathbf{t}')\right| & \le\left|\overline{aba'b'}-\widetilde{aba'b'}\right|+\left|\overline{aba'}\thinspace\overline{b}'-\widetilde{aba'}\tilde{b}'\right|+\left|\overline{abb'}\overline{a}'-\widetilde{abb'}\tilde{a}'\right|+\left|\overline{ab}\overline{a}'\overline{b}'-\widetilde{ab}\tilde{a}'\tilde{b}'\right|\\
 & \phantom{\thinspace\thinspace\le}\left|\overline{aa'b'}\thinspace\overline{b}-\widetilde{aa'b'}\tilde{b}\right|+\left|\overline{aa'}\thinspace\overline{b}\thinspace\overline{b}'-\widetilde{aa'}\tilde{b}\tilde{b}'\right|+\left|\overline{ab'}\overline{a}'\overline{b}-\widetilde{ab'}\tilde{a}'\tilde{b}\right|+\left|\overline{a'b'}\overline{a}\overline{b}-\widetilde{a'b'}\tilde{a}\tilde{b}\right|\\
 & \phantom{\thinspace\thinspace\le}\left|\overline{a'bb'}\overline{a}-\widetilde{a'bb'}\tilde{a}\right|+\left|\overline{a'b}\overline{a}\overline{b}'-\widetilde{a'b}\tilde{a}\tilde{b}'\right|+\left|\overline{a}\thinspace\overline{a}'\overline{bb'}-\tilde{a}\tilde{a}'\widetilde{bb'}\right|+3\left|\overline{a}\overline{b}\overline{a}'\overline{b}'-\tilde{a}\tilde{b}\tilde{a}'\tilde{b}'\right|.
\end{align*}
The first term $\left|\overline{aba'b'}-\widetilde{aba'b'}\right|$
can be bounded by applying the Hoeffding's inequality. Other terms
can be bounded by applying Lemma \ref{lem:bound_product_diff}. Recall
that we write $(x_{1},\ldots,x_{m})_{+}$ for $\max(x_{1},\ldots,x_{m})$.

\paragraph{Bounding $\left|\overline{aba'b'}-\widetilde{aba'b'}\right|$ ($1^{st}$
term).}

Since $-B^{2}\le aba'b'\le B^{2}$, by the Hoeffding's inequality
(Lemma \ref{lem:hoeffding}), we have
\begin{align*}
\mathbb{P}\left(\left|\overline{aba'b'}-\widetilde{aba'b'}\right|\le t\right) & \ge1-2\exp\left(-\frac{nt^{2}}{2B^{4}}\right).
\end{align*}

\paragraph{Bounding $\left|\overline{aba'}\thinspace\overline{b}'-\widetilde{aba'}\tilde{b}'\right|$
($2^{nd}$ term).}

Let $f_{1}(\mathbf{x},\mathbf{y})=aba'=k(\mathbf{x},\mathbf{v})l(\mathbf{y},\mathbf{w})k(\mathbf{x},\mathbf{v}')$
and $f_{2}(\mathbf{y})=b'=l(\mathbf{y},\mathbf{w}')$. We note that
$|f_{1}(\mathbf{x},\mathbf{y})|\le(BB_{k},B_{l})_{+}$ and $|f_{2}(\mathbf{y})|\le(BB_{k},B_{l})_{+}$.
Thus, by Lemma \ref{lem:bound_product_diff} with $E=2$, we have
\begin{align*}
\mathbb{P}\left(\left|\overline{aba'}\thinspace\overline{b}'-\widetilde{aba'}\tilde{b}'\right|\le t\right) & \ge1-4\exp\left(-\frac{nt^{2}}{8(BB_{k},B_{l})_{+}^{4}}\right).
\end{align*}

\paragraph{Bounding $\left|\overline{ab}\overline{a}'\overline{b}'-\widetilde{ab}\tilde{a}'\tilde{b}'\right|$
($4^{th}$ term).}

Let $f_{1}(\mathbf{x},\mathbf{y})=ab=k(\mathbf{x},\mathbf{v})l(\mathbf{y},\mathbf{w})$,
$f_{2}(\mathbf{x})=a'=k(\mathbf{x},\mathbf{v}')$ and $f_{3}(\mathbf{y})=b'=l(\mathbf{y},\mathbf{w}')$.
We can see that $|f_{1}(\mathbf{x},\mathbf{y})|,|f_{2}(\mathbf{x})|,|f_{3}(\mathbf{y})|\le(B,B_{k},B_{l})_{+}$.
Thus, by Lemma \ref{lem:bound_product_diff} with $E=3$, we have
\begin{align*}
\mathbb{P}\left(\left|\overline{ab}\overline{a}'\overline{b}'-\widetilde{ab}\tilde{a}'\tilde{b}'\right|\le t\right) & \ge1-6\exp\left(-\frac{nt^{2}}{18(B,B_{k},B_{l})_{+}^{6}}\right).
\end{align*}

\paragraph{Bounding $\left|\overline{a}\overline{b}\overline{a}'\overline{b}'-\tilde{a}\tilde{b}\tilde{a}'\tilde{b}'\right|$
(last term).}

Let $f_{1}(\mathbf{x})=a=k(\mathbf{x},\mathbf{v}),f_{2}(\mathbf{y})=b=l(\mathbf{y},\mathbf{w}),f_{3}(\mathbf{x})=a'=k(\mathbf{x},\mathbf{v}')$
and $f_{4}(\mathbf{y})=b'=l(\mathbf{y},\mathbf{w}')$. It can be seen
that $|f_{1}(\mathbf{x})|,|f_{2}(\mathbf{y})|,|f_{3}(\mathbf{x})|,|f_{4}(\mathbf{y})|\le(B_{k},B_{l})_{+}$.
Thus, by Lemma \ref{lem:bound_product_diff} with $E=4$, we have
\begin{align*}
\mathbb{P}\left(3\left|\overline{a}\overline{b}\overline{a}'\overline{b}'-\tilde{a}\tilde{b}\tilde{a}'\tilde{b}'\right|\le t\right) & \ge1-8\exp\left(-\frac{nt^{2}}{32\cdot3^{2}(B_{k},B_{l})_{+}^{8}}\right).
\end{align*}

Bounds for other terms can be derived in a similar way to yield
\begin{align*}
\text{(}3^{rd}\text{ term)}\quad\mathbb{P}\left(\left|\overline{abb'}\overline{a}'-\widetilde{abb'}\tilde{a}'\right|\le t\right) & \ge1-4\exp\left(-\frac{nt^{2}}{8(BB_{l},B_{k})_{+}^{4}}\right),\\
\text{(}5^{th}\text{ term)}\quad\mathbb{P}\left(\left|\overline{aa'b'}\thinspace\overline{b}-\widetilde{aa'b'}\tilde{b}\right|\le t\right) & \ge1-4\exp\left(-\frac{nt^{2}}{8(BB_{k},B_{l})_{+}^{4}}\right),\\
\text{(}6^{th}\text{ term)}\quad\mathbb{P}\left(\left|\overline{aa'}\thinspace\overline{b}\thinspace\overline{b}'-\widetilde{aa'}\tilde{b}\tilde{b}'\right|\le t\right) & \ge1-6\exp\left(-\frac{nt^{2}}{18(B_{k}^{2},B_{l})_{+}^{6}}\right),\\
\text{(}7^{th}\text{ term)}\quad\mathbb{P}\left(\left|\overline{ab'}\overline{a}'\overline{b}-\widetilde{ab'}\tilde{a}'\tilde{b}\right|\le t\right) & \ge1-6\exp\left(-\frac{nt^{2}}{18(B,B_{k},B_{l})_{+}^{6}}\right),\\
\text{(}8^{th}\text{ term)}\quad\mathbb{P}\left(\left|\overline{a'b'}\overline{a}\overline{b}-\widetilde{a'b'}\tilde{a}\tilde{b}\right|\le t\right) & \ge1-6\exp\left(-\frac{nt^{2}}{18(B,B_{k},B_{l})_{+}^{6}}\right),\\
\text{(}9^{th}\text{ term)}\quad\mathbb{P}\left(\left|\overline{a'bb'}\overline{a}-\widetilde{a'bb'}\tilde{a}\right|\le t\right) & \ge1-4\exp\left(-\frac{nt^{2}}{8(BB_{l},B_{k})_{+}^{4}}\right),\\
\text{(}10^{th}\text{ term)}\quad\mathbb{P}\left(\left|\overline{a'b}\overline{a}\overline{b}'-\widetilde{a'b}\tilde{a}\tilde{b}'\right|\le t\right) & \ge1-6\exp\left(-\frac{nt^{2}}{18(B,B_{k},B_{l})_{+}^{6}}\right),\\
\text{(}11^{th}\text{ term)}\quad\mathbb{P}\left(\left|\overline{a}\thinspace\overline{a}'\overline{bb'}-\tilde{a}\tilde{a}'\widetilde{bb'}\right|\le t\right) & \ge1-6\exp\left(-\frac{nt^{2}}{18(B_{k},B_{l}^{2})_{+}^{6}}\right).
\end{align*}
By the union bound, we have {\small{}
\begin{align*}
 & \mathbb{P}\left(\left|\hat{S}(\mathbf{t},\mathbf{t}')-S(\mathbf{t},\mathbf{t}')\right|\le12t\right)\\
 & \ge1-\bigg[2\exp\left(-\frac{nt^{2}}{2B^{4}}\right)+{\color{blue}4\exp\left(-\frac{nt^{2}}{8(BB_{k},B_{l})_{+}^{4}}\right)}+{\color{magenta}4\exp\left(-\frac{nt^{2}}{8(BB_{l},B_{k})_{+}^{4}}\right)}+{\color{teal}6\exp\left(-\frac{nt^{2}}{18(B,B_{k},B_{l})_{+}^{6}}\right)}\\
 & \phantom{\ge1-\bigg[\thinspace}{\color{blue}4\exp\left(-\frac{nt^{2}}{8(BB_{k},B_{l})_{+}^{4}}\right)}+6\exp\left(-\frac{nt^{2}}{18(B_{k}^{2},B_{l})_{+}^{6}}\right)+{\color{teal}6\exp\left(-\frac{nt^{2}}{18(B,B_{k},B_{l})_{+}^{6}}\right)}+{\color{teal}6\exp\left(-\frac{nt^{2}}{18(B,B_{k},B_{l})_{+}^{6}}\right)}\\
 & \phantom{\ge1-\bigg[\thinspace}{\color{magenta}4\exp\left(-\frac{nt^{2}}{8(BB_{l},B_{k})_{+}^{4}}\right)}+{\color{teal}6\exp\left(-\frac{nt^{2}}{18(B,B_{k},B_{l})_{+}^{6}}\right)}+6\exp\left(-\frac{nt^{2}}{18(B_{k},B_{l}^{2})_{+}^{6}}\right)+8\exp\left(-\frac{nt^{2}}{32\cdot3^{2}(B_{k},B_{l})_{+}^{8}}\right)\bigg]\\
 & =1-\bigg[2\exp\left(-\frac{nt^{2}}{2B^{4}}\right)+8\exp\left(-\frac{nt^{2}}{8(BB_{k},B_{l})_{+}^{4}}\right)+8\exp\left(-\frac{nt^{2}}{8(BB_{l},B_{k})_{+}^{4}}\right)+24\exp\left(-\frac{nt^{2}}{18(B,B_{k},B_{l})_{+}^{6}}\right)\\
 & \phantom{\ge1-\bigg[\thinspace}+6\exp\left(-\frac{nt^{2}}{18(B_{k}^{2},B_{l})_{+}^{6}}\right)+6\exp\left(-\frac{nt^{2}}{18(B_{k},B_{l}^{2})_{+}^{6}}\right)+8\exp\left(-\frac{nt^{2}}{32\cdot3^{2}(B_{k},B_{l})_{+}^{8}}\right)\bigg]\\
 & \ge1-\bigg[2\exp\left(-\frac{12^{2}nt^{2}}{B^{*}}\right)+8\exp\left(-\frac{12^{2}nt^{2}}{B^{*}}\right)+8\exp\left(-\frac{12^{2}nt^{2}}{B^{*}}\right)+24\exp\left(-\frac{12^{2}nt^{2}}{B^{*}}\right)\\
 & \phantom{\ge1-\bigg[\thinspace}+6\exp\left(-\frac{12^{2}nt^{2}}{B^{*}}\right)+6\exp\left(-\frac{12^{2}nt^{2}}{B^{*}}\right)+8\exp\left(-\frac{12^{2}nt^{2}}{B^{*}}\right)\bigg]\\
 & =1-62\exp\left(-\frac{12^{2}nt^{2}}{B^{*}}\right),
\end{align*}
}where 
\[
B^{*}:=\frac{1}{12^{2}}\max(2B^{4},8(BB_{k},B_{l})_{+}^{4},8(BB_{l},B_{k})_{+}^{4},18(B,B_{k},B_{l})_{+}^{6},18(B_{k}^{2},B_{l})_{+}^{6},18(B_{k},B_{l}^{2})_{+}^{6},32\cdot3^{2}(B_{k},B_{l})_{+}^{8}).
\]
By reparameterization, it follows that 
\begin{align}
\mathbb{P}\left(\frac{c_{1}Jn}{\gamma_{n}}\left|\hat{S}(\mathbf{t},\mathbf{t}')-S(\mathbf{t},\mathbf{t}')\right|\le t\right) & \ge1-62\exp\left(-\frac{\gamma_{n}^{2}t^{2}}{c_{1}^{2}J^{2}nB^{*}}\right).\label{eq:bound_sdiff2}
\end{align}

\subsubsection{Union Bound for $\left|\hat{\lambda}_{n}-\lambda_{n}\right|$ and
Final Lower Bound}

Recall from (\ref{eq:bound_lamb_diff5}) that 
\begin{align*}
|\hat{\lambda}_{n}-\lambda_{n}| & \le\frac{c_{1}Jn}{\gamma_{n}}\left|\hat{S}(\mathbf{t}^{(1)},\mathbf{t}^{(2)})-S(\mathbf{t}^{(1)},\mathbf{t}^{(2)})\right|+\frac{4BJc_{1}n}{\gamma_{n}}\left|\hat{u}(\tilde{\mathbf{t}})-u(\tilde{\mathbf{t}})\right|\\
 & +\frac{c_{1}n}{\gamma_{n}}\frac{8B^{2}J}{n-1}+c_{2}n\sqrt{J}|\hat{u}(\mathbf{t}^{*})-u(\mathbf{t}^{*})|+c_{3}n\gamma_{n}.
\end{align*}
We will bound terms in (\ref{eq:bound_lamb_diff5}) separately and
combine all the bounds with the union bound. As shown in (\ref{eq:ustat_core_bound}),
the U-statistic core $h$ is bounded between $-2B$ and $2B$. Thus,
by Lemma \ref{lem:u_stat_bound} (with $m=2$), we have

\begin{align}
\mathbb{P}\left(c_{2}n\sqrt{J}|\hat{u}(\mathbf{t}^{*})-u(\mathbf{t}^{*})|\le t\right) & \ge1-2\exp\left(-\frac{\lfloor0.5n\rfloor t^{2}}{8c_{2}^{2}n^{2}JB^{2}}\right).\label{eq:bound_utstar_diff1}
\end{align}

\paragraph{Bounding $\frac{c_{1}n}{\gamma_{n}}\frac{8B^{2}J}{n-1}+c_{3}n\gamma_{n}+\frac{4BJc_{1}n}{\gamma_{n}}\left|\hat{u}(\tilde{\mathbf{t}})-u(\tilde{\mathbf{t}})\right|$.}

By Lemma \ref{lem:u_stat_bound} (with $m=2$), it follows that
\begin{align}
 & \mathbb{P}\left(\frac{c_{1}n}{\gamma_{n}}\frac{8B^{2}J}{n-1}+c_{3}n\gamma_{n}+\frac{4BJc_{1}n}{\gamma_{n}}\left|\hat{u}(\tilde{\mathbf{t}})-u(\tilde{\mathbf{t}})\right|\le t\right)\nonumber \\
 & \ge1-2\exp\left(-\frac{\lfloor0.5n\rfloor\gamma_{n}^{2}\left[t-\frac{c_{1}n}{\gamma_{n}}\frac{8B^{2}J}{n-1}-c_{3}n\gamma_{n}\right]^{2}}{2^{7}B^{4}J^{2}c_{1}^{2}n^{2}}\right)\nonumber \\
 & =1-2\exp\left(-\frac{\lfloor0.5n\rfloor\left[t\gamma_{n}(n-1)-8c_{1}B^{2}nJ-c_{3}n(n-1)\gamma_{n}^{2}\right]^{2}}{2^{7}B^{4}J^{2}c_{1}^{2}n^{2}(n-1)^{2}}\right)\nonumber \\
 & \stackrel{(a)}{\ge}1-2\exp\left(-\frac{\left[t\gamma_{n}(n-1)-8c_{1}B^{2}nJ-c_{3}n(n-1)\gamma_{n}^{2}\right]^{2}}{2^{8}B^{4}J^{2}c_{1}^{2}n^{2}(n-1)}\right),\label{eq:bound_const_term_lamb}
\end{align}
where at $(a)$ we used $\lfloor0.5n\rfloor\ge(n-1)/2$. Combining
(\ref{eq:bound_sdiff2}), (\ref{eq:bound_utstar_diff1}), and (\ref{eq:bound_const_term_lamb})
with the union bound (set $T=3t$), we can bound (\ref{eq:bound_lamb_diff5})
with 
\begin{align*}
\mathbb{P}\left(\left|\hat{\lambda}_{n}-\lambda_{n}\right|\le T\right) & \ge1-62\exp\left(-\frac{\gamma_{n}^{2}T^{2}}{3^{2}c_{1}^{2}J^{2}nB^{*}}\right)-2\exp\left(-\frac{\lfloor0.5n\rfloor T^{2}}{72c_{2}^{2}n^{2}JB^{2}}\right)\\
 & -2\exp\left(-\frac{\left[T\gamma_{n}(n-1)/3-8c_{1}B^{2}nJ-c_{3}\gamma_{n}^{2}n(n-1)\right]^{2}}{2^{8}B^{4}J^{2}c_{1}^{2}n^{2}(n-1)}\right).
\end{align*}
Since $\left|\hat{\lambda}_{n}-\lambda_{n}\right|\le T$ implies $\hat{\lambda}_{n}\ge\lambda_{n}-T$,
a reparametrization with $r=\lambda_{n}-T$ gives
\begin{align*}
\mathbb{P}\left(\hat{\lambda}_{n}\ge r\right) & \ge1-62\exp\left(-\frac{\gamma_{n}^{2}(\lambda_{n}-r)^{2}}{3^{2}c_{1}^{2}J^{2}nB^{*}}\right)-2\exp\left(-\frac{\lfloor0.5n\rfloor(\lambda_{n}-r)^{2}}{72c_{2}^{2}n^{2}JB^{2}}\right)\\
 & -2\exp\left(-\frac{\left[(\lambda_{n}-r)\gamma_{n}(n-1)/3-8c_{1}B^{2}nJ-c_{3}\gamma_{n}^{2}n(n-1)\right]^{2}}{2^{8}B^{4}J^{2}c_{1}^{2}n^{2}(n-1)}\right)\\
 & :=L(\lambda_{n}).
\end{align*}
Grouping constants into $\xi_{1},\ldots\xi_{5}$ gives the result. 

The lower bound $L(\lambda_{n})$ takes the form 
\[
1-62\exp\left(-C_{1}(\lambda_{n}-T_{\alpha})^{2}\right)-2\exp\left(-C_{2}(\lambda_{n}-T_{\alpha})^{2}\right)-2\exp\left(-\frac{[(\lambda_{n}-T_{\alpha})C_{3}-C_{4}]^{2}}{C_{5}}\right),
\]
where $C_{1},\ldots,C_{5}$ are positive constants. For fixed large
enough $n$ such that $\lambda_{n}>T_{\alpha}$, and fixed significance
level $\alpha$, increasing $\lambda_{n}$ will increase $L(\lambda_{n})$.
Specifically, since $n$ is fixed, increasing $\mathbf{u}^{\top}\mathbf{\Sigma}^{-1}\mathbf{u}$
in $\lambda_{n}=n\mathbf{u}^{\top}\mathbf{\Sigma}^{-1}\mathbf{u}$
will increase $L(\lambda_{n})$.

\section{Helper Lemmas}

This section contains lemmas used to prove the main results in this
work.

\begin{lem}[Product to sum] \label{lemma:prod-to-sum}
 Assume that  $|a_i| \le B$, $|b_i| \le B$ for $i=1,\ldots,E$. Then $\left|\prod_{i=1}^E a_i - \prod_{i=1}^E b_i \right|\le B^{E-1} \sum_{j=1}^E |a_j-b_j|$.
\end{lem}

\begin{proof}
  \begin{align*}
    \left|\prod_{i=1}^E a_i - \prod_{j=1}^E b_j \right| &\le \left| \prod_{i=1}^E a_i - \prod_{i=1}^{E-1} a_i b_E\right| + \left|\prod_{i=1}^{E-1} a_i b_E - \prod_{i=1}^{E-2} a_i b_{E-1}b_E\right| + \ldots + \left|a_1 \prod_{j=2}^E b_j - \prod_{j=1}^E b_j\right|\\
	  &\le |a_E - b_E| \left| \prod_{i=1}^{E-1} a_i\right| + \left|a_{E-1}-b_{E-1}\right| \left|\left(\prod_{i=1}^{E-2}a_i\right)b_E\right| + \ldots + \left|a_1-b_1\right| \left|\prod_{j=2}^Eb_j\right|\\
	  &\le |a_E - b_E| B^{E-1} + \left|a_{E-1}-b_{E-1}\right| B^{E-1} + \ldots + \left|a_1-b_1\right| B^{E-1} \\ 
	& = B^{E-1} \sum_{j=1}^E |a_j-b_j|
  \end{align*}
applying triangle inequality, and the boundedness of $a_i$ and $b_i$-s. 
\end{proof}

\begin{lem}[Product variant of the Hoeffding's inequality]
\label{lem:bound_product_diff}For $i=1,\ldots,E$, let $\{\mathbf{x}_{j}^{(i)}\}_{j=1}^{n_{i}}\subset\mathcal{X}_{i}$
be an i.i.d. sample from a distribution $P_{i}$, and $f_{i}:\mathcal{X}_{i}\mapsto\mathbb{R}$
be a measurable function. Note that it is possible that $P_{1}=P_{2}=\cdots=P_{E}$
and $\{\mathbf{x}_{j}^{(1)}\}_{j=1}^{n_{1}}=\cdots=\{\mathbf{x}_{j}^{(E)}\}_{j=1}^{n_{E}}$.
Assume that $|f_{i}(\mathbf{x})|\le B<\infty$ for all $\mathbf{x}\in\mathcal{X}_{i}$
and $i=1,\ldots,E$. Write $\hat{P}_{i}$ to denote an empirical distribution
based on the sample $\{\mathbf{x}_{j}^{(i)}\}_{j=1}^{n_{i}}$. Then,
\begin{align*}
\mathbb{P}\left(\left|\left[\prod_{i=1}^{E}\mathbb{E}_{\mathbf{x}^{(i)}\sim\hat{P}_{i}}f_{i}(\mathbf{x}^{(i)})\right]-\left[\prod_{i=1}^{E}\mathbb{E}_{\mathbf{x}^{(i)}\sim P_{i}}f_{i}(\mathbf{x}^{(i)})\right]\right|\le T\right) & \ge1-2\sum_{i=1}^{E}\exp\left(-\frac{n_{i}T^{2}}{2E^{2}B^{2E}}\right).
\end{align*}
\end{lem}
\begin{proof}
By Lemma\,\ref{lemma:prod-to-sum}, we have 
\begin{align*}
\left|\left[\prod_{i=1}^{E}\mathbb{E}_{\mathbf{x}^{(i)}\sim\hat{P}_{i}}f_{i}(\mathbf{x}^{(i)})\right]-\left[\prod_{i=1}^{E}\mathbb{E}_{\mathbf{x}^{(i)}\sim P_{i}}f_{i}(\mathbf{x}^{(i)})\right]\right| & \le B^{E-1}\sum_{i=1}^{E}\left|\mathbb{E}_{\mathbf{x}^{(i)}\sim\hat{P}_{i}}f_{i}(\mathbf{x}^{(i)})-\mathbb{E}_{\mathbf{x}^{(i)}\sim P_{i}}f_{i}(\mathbf{x}^{(i)})\right|.
\end{align*}
By applying the Hoeffding's inequality to each term in the sum, we
have $\mathbb{P}\left(\left|\mathbb{E}_{\mathbf{x}^{(i)}\sim\hat{P}_{i}}f_{i}(\mathbf{x}^{(i)})-\mathbb{E}_{\mathbf{x}^{(i)}\sim P_{i}}f_{i}(\mathbf{x}^{(i)})\right|\le t\right)\ge1-2\exp\left(-\frac{2n_{i}t^{2}}{4B^{2}}\right).$
The result is obtained with a union bound. 
\end{proof}

\section{External Lemmas}

In this section, we provide known results referred to in this work.
\begin{lem}[{\citetsup[Lemma 1]{Chwialkowski2015}}]
\label{lem:analytic_rkhs}If $k$ is a bounded, analytic kernel (in
the sense given in Definition\,\ref{def:analytic_kernel}) on $\mathbb{R}^{d}\times\mathbb{R}^{d}$,
then all functions in the RKHS defined by $k$ are analytic.
\end{lem}

\begin{lem}[{\citetsup[Lemma 3]{Chwialkowski2015}}]
\label{lem:metric_prob} Let $\Lambda$ be an injective mapping from
the space of probability measures into a space of analytic functions
on $\mathbb{R}^{d}$. Define 
\begin{align*}
d_{V_{J}}^{2}(P,Q) & =\sum_{j=1}^{J}\left|[\Lambda P](\mathbf{v}_{j})-[\Lambda Q](\mathbf{v}_{j})\right|^{2},
\end{align*}
where $V_{J}=\{\mathbf{v}_{i}\}_{i=1}^{J}$ are vector-valued i.i.d.
random variables from a distribution which is absolutely continuous
with respect to the Lebesgue measure. Then, $d_{V_{J}}(P,Q)$ is almost
surely (w.r.t. $V_{J}$) a metric.
\end{lem}

\begin{lem}[Bochner's theorem \citepsup{Rudin2011}]
\label{lem:bochner} A continuous function $\Psi:\mathbb{R}^{d}\to\mathbb{R}$
is positive definite if and only if it is the Fourier transform of
a finite nonnegative Borel measure $\zeta$ on $\mathbb{R}^{d}$,
that is, $\Psi(\mathbf{x})=\int_{\mathbb{R}^{d}}e^{-i\mathbf{x}^{\top}\boldsymbol{\omega}}\thinspace\mathrm{d}\zeta(\boldsymbol{\omega}),\thinspace\mathbf{x}\in\mathbb{R}^{d}.$ 
\end{lem}

\begin{lem}[{A bound for U-statistics \citepsup[Theorem A, p.\ 201]{Serfling2009}}]
\label{lem:u_stat_bound} Let $h(\mathbf{x}_{1},\ldots,\mathbf{x}_{m})$
be a U-statistic kernel for an $m$-order U-statistic such that $h(\mathbf{x}_{1},\ldots,\mathbf{x}_{m})\in[a,b]$
where $a\le b<\infty$. Let $U_{n}=\binom{n}{m}^{-1}\sum_{i_{1}<\cdots<i_{m}}h(\mathbf{x}_{i_{1}},\ldots,\mathbf{x}_{i_{m}})$
be a U-statistic computed with a sample of size $n$, where the summation
is over the $\binom{n}{m}$ combinations of $m$ distinct elements
$\{i_{1},\ldots,i_{m}\}$ from $\{1,\ldots,n\}$. Then, for $t>0$
and $n\ge m$,
\begin{align*}
\mathbb{P}(U_{n}-\mathbb{E}h(\mathbf{x}_{1},\ldots,\mathbf{x}_{m})\ge t) & \le\exp\left(-2\lfloor n/m\rfloor t^{2}/(b-a)^{2}\right),\\
\mathbb{P}(|U_{n}-\mathbb{E}h(\mathbf{x}_{1},\ldots,\mathbf{x}_{m})|\ge t) & \le2\exp\left(-2\lfloor n/m\rfloor t^{2}/(b-a)^{2}\right),
\end{align*}
where $\lfloor x\rfloor$ denotes the greatest integer which is smaller
than or equal to $x$. Hoeffind's inequality is a special case when
$m=1$.
\end{lem}

\begin{lem}[Hoeffding's inequality]
\label{lem:hoeffding} Let $X_{1},\ldots,X_{n}$ be i.i.d. random
variables such that $a\le X_{i}\le b$ almost surely. Define $\overline{X}:=\frac{1}{n}\sum_{i=1}^{n}X_{i}$.
Then,
\[
\mathbb{P}\left(\left|\overline{X}-\mathbb{E}[\overline{X}]\right|\le\alpha\right)\ge1-2\exp\left(-\frac{2n\alpha^{2}}{(b-a)^{2}}\right).
\]
\end{lem}

\bibliographystylesup{abbrvnat}
\bibliographysup{fsic_appendix}
\end{document}